\documentclass[a4paper, 11pt]{article}
\usepackage[margin=1in]{geometry}
\usepackage{graphicx} 
\usepackage{amsthm}
\usepackage[toc,page]{appendix}
\usepackage{soul}
\usepackage{bbm}
\usepackage{tikz}
\usetikzlibrary{automata, positioning, arrows}
\usepackage{amsmath,amssymb,bm}
\usepackage{hyperref}
\usepackage[capitalise]{cleveref}
\usepackage[skip=6pt, indent=20pt]{parskip}
\usepackage{lineno}
\usepackage{comment}
\usepackage{esvect}
\usepackage{algorithm}
\usepackage{algpseudocode}
\allowdisplaybreaks[1]

\makeatletter
\newenvironment{breakablealgorithm}
  {
   \begin{center}
     \refstepcounter{algorithm}
     \hrule height.8pt depth0pt \kern2pt
     \renewcommand{\caption}[2][\relax]{
       {\raggedright\textbf{\fname@algorithm~\thealgorithm} ##2\par}%
       \ifx\relax##1\relax 
         \addcontentsline{loa}{algorithm}{\protect\numberline{\thealgorithm}##2}%
       \else 
         \addcontentsline{loa}{algorithm}{\protect\numberline{\thealgorithm}##1}%
       \fi
       \kern2pt\hrule\kern2pt
     }
  }{
     \kern2pt\hrule\relax
   \end{center}
  }
\makeatother

\renewcommand{\P}{\mathbb{P}}

\newcommand{\E}{\mathbb{E}}
\newcommand{\dkl}{D_{\textnormal{KL}}}
\newcommand{\dtv}{d_{\textnormal{TV}}}
\newcommand{\Esrw}{\E_{\textnormal{srw}}}
\newcommand{\Psrw}{\P_{\textnormal{srw}}}
\newcommand{\Eerrw}{\E_{\textnormal{errw}}}
\newcommand{\Perrw}{\P_{\textnormal{errw}}}

\newcommand{\Pvrrw}{\P_{\textnormal{vrrw}}}
\newcommand{\var}{\textnormal{Var}}

\newcommand{\allone}{\mathbf{1}}

\newcommand{\ind}{\mathbf{1}}

\newcommand{\diam}{\textnormal{diam}}
\newcommand{\cov}{\textnormal{cov}}
\newcommand{\covar}{\textnormal{Cov}}

\newcommand{\hit}{\textnormal{hit}}
\newcommand{\comm}{\textnormal{comm}}
\newcommand{\gam}{\textnormal{Gamma}}

\newcommand{\mc}{\mathcal{C}}
\newcommand{\mr}{\mathcal{R}}
\newcommand{\poly}{\textnormal{poly}}

\newtheorem{thm}{Theorem}[section]
\newtheorem{lem}[thm]{Lemma}
\newtheorem{cor}[thm]{Corollary}
\newtheorem{rem}{Remark}[thm]

\newtheorem{remark}{Remark}

\newif\ifva
\vatrue

\title{On Statistical Estimation of Edge-Reinforced Random Walks}
\author{Qinghua (Devon) Ding, \quad Venkat Anantharam\\
  Department of Electrical Engineering and Computer Sciences\\
                    University of California at Berkeley\\
                    Berkeley, CA, United States\\
                    Email: \{devon\_ding, ananth\}@berkeley.edu}

\begin{document}

\maketitle

\begin{abstract}
Reinforced random walks (RRWs), including vertex-reinforced random walks (VRRWs) and edge-reinforced random walks (ERRWs), model random walks where the transition probabilities evolve based on prior visitation history~\cite{mgr, fmk, tarres, volkov}. These models have found applications in various areas, such as network representation  learning~\cite{xzzs}, reinforced PageRank~\cite{gly}, and modeling animal behaviors~\cite{smouse}, among others. However, statistical estimation of the parameters governing RRWs remains underexplored. This work focuses on estimating the initial edge weights of ERRWs using observed trajectory data. Leveraging the connections between an ERRW and a random walk in a random environment (RWRE)~\cite{mr, mr2}, as given by the so-called ``magic formula", we propose an estimator based on the generalized method of moments. To analyze the sample complexity of our estimator, we exploit the hyperbolic Gaussian structure embedded in the random environment to bound the fluctuations of the underlying random edge conductances.
\end{abstract}

\section{Introduction}


Given a simple random walk on an unweighted undirected graph $G=(V, E)$, it is often of interest to study its stationary distribution $\pi$, which, for example, may serve as an indicator of the importance of each state, as in the PageRank algorithm~\cite{page}. 
This has motivated research aimed at designing
faster algorithms to compute the stationary distribution of such random walks, see e.g.~\cite{dmpu,lbgs}.
There is also considerable
interest in estimating the Markov chain itself, i.e., the transition matrix $P$, based on curated data. In the simplest scenario, one is given a trajectory $X:=\{X_t, 0 \le t \le N \}$ for some $N \ge 0$ and tasked with constructing an estimator $\tilde{P}(X)$ of the Markov chain from this trajectory~\cite{ddg, wk, cdl}.

However, human and animal behaviors in the real world often deviate from being Markovian~\cite{mgr} or even from having a higher-order Markovian structure~\cite{bgl}. As an illustrative example, individuals tend to revisit websites they have accessed before, leading to a “rich-get-richer” phenomenon~\cite{mgr}. Similarly, animal movements frequently exhibit spatial memory, which can be incorporated into models as a reinforced random walk~\cite{fmk}. Interested readers can explore numerous applications in network science and biological sciences in these papers and their references. For instance, reinforced random walks and their variants have been applied to network representation learning~\cite{xzzs}, reinforced PageRank~\cite{gly}, modeling animal behavior~\cite{smouse}, and even tumor studies~\cite{ps}.

In our current study of reinforced random walks, we aim to develop theoretically guaranteed algorithms for related statistical estimation problems similar in flavor to the ones typically considered for random walks. Specifically, we focus on the following question:

\vspace{.3em}
\begin{center}
\noindent\fbox{%
    \parbox{0.9\textwidth}{%
        \emph{How can we estimate the parameters of reinforced random walks from our observations, which are usually sampled trajectory data? What is the sample complexity of some such estimators?}
    }%
}
\end{center}
\vspace{.3em}

We 
restrict attention to
reinforced random walks (RRWs) on a finite set $V=[n]$,
where movement is possible between some pairs of adjacent vertices, i.e. undirected edges, defining a connected undirected graph $G=(V, E)$.
Here $[n]$ denotes $\{1, \ldots, n\}$. 
There are two major families of reinforced random walks studied in the literature: vertex-reinforced random walks (VRRWs) and edge-reinforced random walks (ERRWs). Although both of them have a self-reinforcing property, their qualitative behaviors can be drastically different. 
Let us introduce VRRWs first.

Suppose we are given a graph $G=(V, E)$ together with a vector $\Lambda_0 \in \mathbb{R}_{++}^V$
(i.e. every entry of $\Lambda_0$ is strictly positive).
We will assume, for simplicity, that the graph is simple, i.e.\, it has no self-loops or multiple edges between a pair of vertices. 
Let us fix some $v_0\in V$ and consider the process that starts at $X_0=v_0$. We can define the vertex local time of the trajectory $\{X_t:t\in\mathbb{Z}_+\}$ at state $i\in V$ and time $t$ to be 
\begin{equation}    \label{eqn:vertexlt}
\Lambda_t^i:=\Lambda_0^i+\sum_{s=1}^t \ind(X_s=i). 
\end{equation}
Here $\mathbb{Z}_+$ denotes the set of nonnegative integers. The process 
$\{\Lambda_t, t \in \mathbb{Z}_+\}$, where 
$\Lambda_t:=(\Lambda_t^i)_{i\in V}$ is known as the vertex local time process and $\Lambda_0$ is called the initial vertex local time. The law of VRRW starting at $v_0$ with initial vertex local time $\Lambda_0$ can be stated as
\[
\Pvrrw^{v_0, \Lambda_0}(X_{t+1}=i \mid \{X_0, X_1, \cdots, X_t\}) = \frac{\Lambda_t^{i}}{\sum_{i':X_t \sim i'} \Lambda_t^{i'}},\quad \forall i\in V,
\quad \forall t \in \mathbb{Z}_+.
\]

Here for any $i,j\in V$, $i\sim j$ denotes that there is an edge in $E$ between $i$ and $j$. For a VRRW, it is well-known that, for certain types of graphs, the random walk is almost surely confined to some proper subgraph. A classical example is VRRW on $\mathbb{Z}$, where $i\sim j$ if and only if $|i-j|=1$. Given initial vertex local time $\Lambda_0=\allone$, 
where $\allone$ denotes the all-ones 
sequence,
\ifva
\color{red}
\color{black}
\fi
and the initial state $v_0=0$, for any $k\in\mathbb{Z}$ the VRRW on this graph has a positive probability of getting trapped in the subgraph induced by $V'=\{k-2,k-1,k,k+1,k+2\}$, i.e., $\Pvrrw^{v_0, \Lambda_0}(\{v\in V: X_t=v\text{ infinitely often }\}=V')>0$ (Theorem 1.3, \cite{tarres}). This localization phenomenon of VRRW was also generalized to other graphs in \cite{volkov}.

Another class of reinforced random walks arises when reinforcement occurs on edges instead of vertices. To discuss ERRW, let us define the edge local time processes first. Assume that we are given the initial edge weights as $L_0 = A \in \mathbb{R}_{++}^E$. For trajectory $\{X_t, t \in \mathbb{Z}_+\}$ define the edge weight of $e=\{i,j\}\in E$ 
\footnote{Note that the edges are assumed to be undirected, and so, since the graph is assumed to be simple, an edge can be identified with the set of its end points. In particular, $\{i,j\} = \{j,i\}$.}
at time $t$ as $L_t^e = L_0^e+\sum_{s=0}^{t-1}\ind(\{X_s,X_{s+1}\}=\{i,j\})$.
The process 
$\{L_t, t \in \mathbb{Z}_+\}$, where
$L_t:=(L_t^e)_{e\in E}$, can be called either the edge weight process or the edge local time process. The law of ERRW starting at $v_0$ with initial edge local times given by $L_0=A$ can be stated as 
\begin{equation}        \label{eqn:errwlaw}
\Perrw^{v_0,A}(X_{t+1}=i \mid \{X_0,X_1,\cdots, X_t\}) = \frac{L_t^{\{X_t, i\}}}{\sum_{i':X_t \sim i'} L_t^{\{X_t, i'\}}}, \quad \forall i\in V,
\quad \forall t \in \mathbb{Z}_+.
\end{equation}

Unlike for VRRW, it is known that, almost surely, ERRW will visit every state in $V$ infinitely often if the graph is finite and connected and all the initial edge weights are positive (as we assume). To be precise, for any $v_0\in V$ and $A \in \mathbb{R}_{++}^E$, we have $\Perrw^{v_0,A}(\{v\in V: X_t=v\text{ infinitely often }\}=V)=1$ if $G$ is finite and connected. This can be seen from the interpretation of ERRW as a random walk in a random environment(RWRE) where the random environment is supported on $\mathbb{R}_{++}^E$. See \cref{eqn:mixing} below for the precise density formula for the random environment.

\subsection{ERRW as random walk in a random environment(RWRE)}

For 
a random walk
on a connected graph $G$, the normalized vertex local time process 
converges
almost surely to the stationary measure $\pi$ of the irreducible transition probability $P$. This classical result is known as the strong law of large numbers for Markov chains,
and is an immediate consequence of stronger
convergence results such as 
Theorem 1 in Section 1.6 of~\cite{Chung60}.

In the case of VRRW, the normalized vertex local time
$\frac{1}{t}\Lambda_t$
will converge in probability to a distribution, as $t \to \infty$. Moreover, the distribution that it converges to will be one of the solutions of a certain fixed-point equation\footnote{The fixed-point equation involved here is essentially the replicator equation.}(Theorem 1.1, \cite{benaim}). However, it is known that, depending on the graph $G$ and the initial conditions, there could be multiple solutions to the corresponding fixed-point equation, and each of them has a positive probability of being the limiting distribution if they satisfy some regularity conditions (Corollary 6.5, \cite{benaim}). For example, for VRRW on $\mathbb{Z}$, it is known that the fixed-point equations have infinitely many solutions and each of them has a non-zero probability of being the limit of 
$\frac{1}{t}\Lambda_t$~\cite{tarres}. We remark that the methods people use to study vertex-reinforced random walks usually come from dynamical systems~\cite{benaim} 
and techniques for analyzing complex yet structured recursions~\cite{tarres}.

One important difference between ERRW and VRRW is that ERRW does not suffer from the issue of localization. Hence one can guarantee that, almost surely, an ERRW on $\mathbb{Z}$ will eventually visit every state. Besides this, we would like to highlight two beautiful results about ERRW.

\begin{enumerate}
    \item ERRW can be interpreted as a random walk in a random environment (RWRE) \cite{kr00}, and one can explicitly compute the probability measure governing the random edge conductances, which is also known as the mixing measure (\cite{mr2}, Theorem 1), when the underlying graph is finite. More precisely, an ERRW with initial edge weights $A:=(a_e)_{e\in E}$ and initial state $v_0\in V$ can be thought of as a simple random walk with random conductances given by $W=(W_{ij})_{i,j\in V}$ coming from the mixing measure $d\mu_{v_0,A}(w)$:
    \[\Perrw^{v_0, A}[\cdot]=\int\Psrw^{v_0, w}[\cdot]d\mu_{v_0,A}(w).\]
    
    Here $\Perrw^{v_0, A}$ represents the probability measure associated with ERRW with initial edge weight $A$ and initial state $v_0$, while $\Psrw^{v_0, w}$ represents the probability measure associated with a simple random walk with edge weight matrix $w$ and initial state $v_0$. We will give the explicit formula for the mixing measure later. This formula is known as the ``magic formula'' in the study of ERRW~\cite{dr,mr2}.
    \ifva
\color{red}
\color{black}
\fi

    \item The problem of deciding whether the local time at 0 is infinite almost surely (positive recurrent) or finite almost surely (transient) for ERRW on the graph $\mathbb{Z}^d$ is resolved: it is positive recurrent when $d=2$; while it is transient when $d\geq 3$~\cite{st, bhs}. These results make use of the beautiful connections between ERRW and the vertex-reinforced jump process (VRJP)~\cite{st}, as well as the so-called Anderson model~\cite{bhs, st}. These connections will not be explored in this paper but interested readers are referred to~\cite{bhs, st}.
\end{enumerate}

Therefore, via the connection between ERRW and RWRE, one can see that the normalized local time process of an ERRW will converge to the stationary distribution of the random environment $W$ that comes from the mixing measure $\mu_{v_0, A}$. To define the mixing measure $\mu_{v_0,A}$, let us fix a special edge $e_0$ that is incident on $v_0$.\footnote{We remark that the law $\Psrw^{v_0,W}$ under random conductance $d\mu_{v_0,A}(W)$ is invariant under different choices of $e_0$. In other words, if we denote the transition matrix corresponding to weight $W$ as $P$, then under different choices of $e_0$, the induced distribution of $P$ will not change even though the law of $W$ changes.} Let 
$a_v:=\sum_{u\sim v}a_{\{u,v\}}$ be the sum of weights of all the edges that are incident with node $v$. Then the random conductance $W$ can be taken with $W_{e_0}=1$, and $W_{-e_0}:=(W_e)_{e\neq e_0}$ distributed according to
\begin{equation}
\label{eqn:mixing}
d\mu_{v_0,A}(w_{-e_0}) =Z_{v_0, A}^{-1}\cdot  
\frac{w_{v_0}^{\frac{1}{2}}\prod_{e\in E}w_e^{a_e-1}} {\prod_{v\in V}w_v^{\frac{1}{2}(a_v+1)}}
\sqrt{\sum_{T\in\mathcal{T}}\prod_{e\in T}w_e}\,dw_{-e_0}.
\end{equation}
Here, for any $v\in V$, $w_v$ is defined as $w_v:=\sum_{u\sim v}w_{uv}$.
Also, $\mathcal{T}$ denotes the set of all spanning trees of the underlying graph, and $dw_{-e_0}$ is defined as $dw_{-e_0}:=\prod_{e\in E\backslash\{e_0\}}dw_e$. 
To be more precise, if we let $\delta(x)$ denote the Dirac delta function, i.e. $\int_\mathbb{R}\delta(x)dx=1$ and $\delta(x) = 0$ for $x \neq 0$, then the mixing distribution on $w$ can be written as
$\delta(w_{e_0}-1)dw_{e_0}d\mu_{v_0,A}(w_{-e_0})$. 
This is captured by writing the formula as in 
\cref{eqn:mixing}, with the understanding that 
$w_{e_0} = 1$.

The normalizing factor $Z_{v_0, A}$ is defined as
\begin{equation}
\label{eqn:normalizer1}
Z_{v_0, A}:=\int_{\mathbb{R}_{++}^{|E|-1}} \frac{w_{v_0}^{\frac{1}{2}}\prod_{e\in E}w_e^{a_e-1}} {\prod_{v\in V}w_v^{\frac{1}{2}(a_v+1)}}
\sqrt{\sum_{T\in\mathcal{T}}\prod_{e\in T}w_e}\,dw_{-e_0},
\end{equation}
where again $w_{e_0}$ is taken to equal $1$.

Moreover, one can evaluate the normalizing factor (Theorem 4, \cite{stz}) as
\begin{equation}
\label{eqn:normalizer2}
Z_{v_0, A}=\frac{\pi^{\frac{1}{2}(|V|-1)}}{2^{1-|V|+\sum_{e\in E}a_e}}\cdot \frac{\prod_{e\in E}\Gamma(a_e)}{\prod_{v\in V}\Gamma(\frac{1}{2}(a_v+1-\ind_{v=v_0}))}.
\end{equation}
It may be surprising that this expression does not depend on the choice of the edge $e_0$ that is incident on $v_0$. Why this is true will become more clear after the discussion at the beginning of \cref{sec:nonasymptotic}.

The reason why such an RWRE representation is available is that ERRW satisfies partial exchangeability: for any $t\in\mathbb{Z}_+$, $v_0\in V$, and $A, B\in\mathbb{R}_{++}^E$, conditioned on $X_0=v_0, L_0=A$ and that the edge local time at time $t$ is $L_t=B$, the distribution of the trajectory up to time $t$ is uniformly random among those trajectories that align with these conditions. 
In other words, given any trajectory $X=\{X_s, 0\leq s\leq t\}$, we can permute the vertices in an arbitrary way to obtain another trajectory $\tilde{X}$, but $\Perrw^{v_0,A}(X)=\Perrw^{v_0,A}(\tilde{X})$ so long as the two trajectories yield the same edge local time at time $t$. 
Then de Finetti's theorem for reversible Markov chains~\cite{dr} guarantees that ERRW has a representation as a mixture of reversible Markov chains. Interested readers may also refer to \cite{rolles}, where it is argued that ERRW is the only family of random processes that satisfies partial exchangeability under certain conditions.

From now on, we will think of ERRW as being naturally coupled with the underlying RWRE. From the RWRE perspective, 
it is almost surely true that each of the edge weights $W_e$ is strictly positive.
\ifva
\color{red}
\color{black}
\fi
Since the graph is connected and finite, we would expect the simple random walk with the underlying weights $W$ to cover every state as $t\rightarrow\infty$. This observation guarantees that, almost surely, the ERRW has a finite random cover time, which is defined as $\tau_\cov:=\inf\{s\geq 0:\,\Lambda_s^i-\Lambda_0^i>0,\forall\,i\in[n]\}$. 
\footnote{Here we continue to use the notation
$\Lambda_t^i$, as given in \cref{eqn:vertexlt},
for the vertex local times, even though we are dealing from now on with ERRW, with the dynamics given in \cref{eqn:errwlaw}.}
Actually, we will see later that the cover time $t_\cov$, which is the maximum of $\Eerrw^{v_0, A}(\tau_\cov)$ over all starting states $v_0\in V$, is also finite. In conclusion, ERRW is not as localized as VRRW, and conditional mixing (the process mixes given $W$) happens in finite time almost surely. 


\subsection{The estimation problem of interest}
\label{sec:formulation}

We are interested in the problem of parameter estimation of ERRW using sample trajectories. Given the graph structure and the sampled trajectories, we want to recover the list of initial edge weights $L_0=A$. 

Formally, given a connected simple graph $G=(V=[n], E)$, $v_0\in V$ and $A:=(a_e)_{e\in E}\in\mathbb{R}_{++}^E$\footnote{To be precise, $G$ and $v_0$ are observed while $A$ is not.}, we obtain 
$K \ge 1$
i.i.d. ERRW sample trajectories of length 
$T \ge 1$
on graph $G$ with starting point $X_0=v_0$ and initial local time $L_0=A$. We denote these sample trajectories as $\big\{X_t(k)\,\big|\,0 \le t\leq T, 1 \le k\leq K\big\}$, and we want to find an estimator $\hat{A}$ such that the error criterion $d(A,\hat{A})$ is small enough with high probability. We assume that all the initial edge weights are bounded above and below by positive constants. We assume that $\underline{a}\leq \min_{e\in E} a_e\leq \max_{e\in E} a_e \leq \overline{a}$ for some positive constants $\underline{a}, \overline{a}$ that do not scale with $n$.

To be precise, 
we want an estimator $\hat{A}$ that is an entry-wise $\epsilon$-multiplicative approximation to $A$, with high probability. This motivates us to consider the following ratio-based divergence between two entry-wise positive 
sequences $A, B\in\mathbb{R}_{++}^E$:
\[d(A,B):=\max_{\{i,j\} \in E}\bigg(\max\bigg\{\frac{A_{ij}}{B_{ij}}, \frac{B_{ij}}{A_{ij}}\bigg\}-1\bigg).\]
It is straight-forward to check that we have: (i) $d(A, B)\geq 0$; (ii) $d(A, B)=0$ if and only if $A=B$; (iii) if $d(A,B)\leq \epsilon$ 
for some $\epsilon\in(0,1)$ 
then $A$ is an entry-wise $\epsilon$-multiplicative approximation of $B$, i.e., $(1-\epsilon)B_{ij}\leq A_{ij}\leq (1+\epsilon)B_{ij}$ for every $\{i,j\} \in E$.
\ifva
\color{red}
\color{black}
\fi

Informally speaking, when $K$ and $T$ are large enough, we should be able to recover $A$ with arbitrarily small error. But fixing the error level $\epsilon>0$ and a confidence level $\delta\in(0,1)$, how many samples are sufficient for the recovery of $A$ 
by an estimator $\hat{A}$
\ifva
\color{red}
\color{black}
\fi
with error $d(A,\hat{A})\leq \epsilon$ with probability $\geq 1-\delta$ (upper bound)? How many samples are needed to do so (lower bound)?

\section{Learning the parameters of Edge-Reinforced Random Walks}

As stated in section \ref{sec:formulation}, fixing $G=(V, E)$, $v_0\in V$ and $A\in\mathbb{R}_{++}^E$, we observe 
$K \ge 1$
i.i.d. sample trajectories of length 
$T \ge 1$
from $\Perrw^{v_0, A}$, and want to estimate $A$. 
This statistical question
arises in modeling user website browsing behavior via an ERRW. 
The vertices of $G$ are the websites, and the network links between websites form the edges in $G$, while the vector of initial local times $A$ encodes the prior preferences of the population: larger weights correspond to more ``preferred'' links. Given a segment of users, one assumes each user's browsing log is generated by $\Perrw^{v_0,A}$, with $v_0$ representing a common entry point (e.g., a search engine). Here ERRW serves as a way to capture the users' self-reinforcement behaviors. Given users' weblogs, which are treated as i.i.d.\,sample trajectories, we then aim to recover $A$. 
Alternatively, in a personalized context, where we have a single user, multiple continuous browsing sessions can each be viewed as an independent ERRW trajectory, capturing short-term memory effects within that session, and assuming independence between different sessions.

Several works in the literature illustrate some similar modeling approaches:
\begin{itemize}
    \item \cite{mgr, xzzs} modeled user browsing with reinforced random walks, then ranked the websites using certain averaged visitation statistics over multiple observed trajectories. 
    \item \cite{bgl} employed a variant of ERRW to model taxis' trajectories, assuming each follows the same underlying distribution. 
\ifva
\color{red}
\color{black}
\fi
    \item \cite{ps} examined potential applications of reinforced random walk variants in modeling tumor angiogenesis. To infer the model parameters, one can likely obtain near-i.i.d.\,sample trajectories under controlled lab conditions.
\end{itemize}

Thus, for a wide range of potential applications, one can frame the task as estimating the model parameter $A$ given multiple sampled i.i.d. ERRW trajectories.

We now introduce some standard statistical notation.
The total variation distance between two probability 
distributions on a measurable space $(\mathcal{X}, \mathcal{B})$
is defined, as usual, as 
$\dtv(P_X, Q_X) := \sup_{B \in \mathcal{B}} \left( P_X(B) - Q_X(B) \right)$.
Recall that if $P_X$ and $Q_X$ are probability distribution
on the finite set $\mathcal{X}$ we have
\[\dtv(P_X, Q_X) := \frac{1}{2}\sum_{x \in \mathcal{X}} |p_X(x)-q_X(x)|.\]
The Kullback-Leibler (KL) divergence of $P_X$ with respect
to $Q_X$ is defined to be $\infty$ if $P_X$ is not
absolutely continuous with respect to $Q_X$, and is defined
to be $\dkl(P_X \| Q_X) := \mathbb{E}_{P_X} \left[ \log \frac{dP_X}{dQ_X} (X) \right]$ otherwise.
Pinsker's inequality, see e.g. Theorem 4.19 of~\cite{BLM2013}, relates KL divergence and total variation distance by 
\[
\dtv(P_X,Q_X)\leq \sqrt{\frac{1}{2}\dkl(P_X\|Q_X)}.
\]

We use the notation $X_i^j:=\{X_i,\cdots, X_j\}$ for $0\leq i\leq j\leq T$.

\subsection{Insufficiency of a single trajectory}

When estimating the transition matrix of an irreducible finite Markov chain, given a single trajectory that is long enough, one can expect to recover the model parameters
in the limit of large trajectory length. 
However, we will argue in this section that a single trajectory, even if infinitely long, is not good enough 
to recover the parameters (i.e. the initial edge weights) for ERRW. 

Consider a scenario where we have a sample trajectory coming from either $\Perrw^{v_0, A}$ or $\Perrw^{v_0, \tilde{A}}$, with the same graph structure $G=(V, E)$ and the same starting point $v_0\in V$. We want to test the null hypothesis $H_0$:``the initial weight vector is $L_0=A$'' against the alternative hypothesis $H_1$:``the initial weight vector is $L_0=\tilde{A}$''. To do so, one designs a testing algorithm $\mathcal{A}$ that takes some trajectory $\underline{X}_0^T$ of the ERRW (assumed to be generated under hypothesis $\Theta\in\{H_0, H_1\}$), and outputs a decision whether $\underline{X}_0^T$ is generated from $\Perrw^{v_0, A}$ 
or it is generated from $\Perrw^{v_0,\tilde{A}}$. 
\ifva
\color{red}
\color{black}
\fi
The testing error of this algorithm is defined as
\[\mathcal{E}:=\frac{1}{2}\big(\P(\mathcal{A}(\underline{X}_0^T)=H_1\,|\,\Theta=H_0)+\P(\mathcal{A}(\underline{X}_0^T)=H_0\,|\,\Theta=H_1)\big).\]

In hypothesis testing language, this is a combination of Type-I error which is defined as $\mathcal{E}_\text{I}:=\P(\mathcal{A}(\underline{X}_0^T)=H_1\,|\,\Theta=H_0)$, and Type-II error which is defined as $\mathcal{E}_{\text{II}}:=\P(\mathcal{A}(\underline{X}_0^T)=H_0\,|\,\Theta=H_1)$. This testing error $\mathcal{E}$ is also referred to as the Bayesian error of the testing procedure under the uniform prior. $\mathcal{E}$ being small is equivalent to both Type-I error and Type-II error being small. A good testing procedure should make the testing error $\mathcal{E}$ as small as possible. 

Let us denote a trajectory generated from $H_0$ as $X_0^T$ and another independent trajectory generated from $H_1$ as $\tilde{X}_0^T$. Recall that for any ERRW, one can naturally couple it with an RWRE. Let us assume that the random environment for each trajectory is $W$ and $\tilde{W}$, respectively. Under such a coupling, we have Markov chains $A\rightarrow W\rightarrow X_0^T$ and $\tilde{A}\rightarrow \tilde{W}\rightarrow \tilde{X}_0^T$. 
We then have
\begin{align*}
\dkl(X_0^T, W\|\tilde{X}_0^T, \tilde{W})
&=\int \sum_{x_0^T} p_{X_0^T|W}(x_0^T|w)p_W(w)\log\bigg(\frac{p_{X_0^T|W}(x_0^T|w)}{p_{\tilde{X}_0^T|\tilde{W}}(x_0^T|w)}\cdot \frac{p_W(w)}{p_{\tilde{W}}(w)}\bigg)\,dw.
\end{align*}

Since $p_{X_0^T|W}(x_0^T|w)=p_{\tilde{X}_0^T|\tilde{W}}(x_0^T|w)$, as they are both the probability of generating trajectory $x_0^T$ from a simple random walk $\Psrw^{v_0,w}$, we conclude that
\begin{align}
\label{eqn:kl1}
\dkl(X_0^T, W\|\tilde{X}_0^T, \tilde{W})
&= \int \sum_{x_0^T} p_{X_0^T|W}(x_0^T|w)p_W(w)\log\bigg(\frac{p_W(w)}{p_{\tilde{W}}(w)}\bigg)\,dw \nonumber\\
&= \int p_W(w)\log\bigg(\frac{p_W(w)}{p_{\tilde{W}}(w)}\bigg)\,dw \nonumber\\
&= \dkl(W\|\tilde{W}).
\end{align}

In the first equation above we used the total probability rule. On the other hand, we also have $p_{X_0^T|W}(x_0^T|w) p_W(w)=p_{X_0^T}(x_0^T) p_{W|X_0^T}(w|x_0^T)$. Hence,
\begin{align}
\label{eqn:kl2}
\dkl(X_0^T, & W\|\tilde{X}_0^T, \tilde{W})
= \int \sum_{x_0^T} p_{X_0^T}(x_0^T) p_{W|X_0^T}(w|x_0^T)\log\bigg(\frac{p_{X_0^T}(x_0^T)}{p_{\tilde{X}_0^T}(x_0^T)}\cdot \frac{p_{W|X_0^T}(w|x_0^T)}{p_{\tilde{W}|\tilde{X}_0^T}(w|x_0^T)}\bigg)\,dw \nonumber\\
=& \int \sum_{x_0^T} p_{X_0^T}(x_0^T) p_{W|X_0^T}(w|x_0^T)\log\bigg(\frac{p_{X_0^T}(x_0^T)}{p_{\tilde{X}_0^T}(x_0^T)}\bigg)\,dw \nonumber\\
&\qquad\qquad + \int \sum_{x_0^T} p_{X_0^T}(x_0^T) p_{W|X_0^T}(w|x_0^T)\log\bigg(\frac{p_{W|X_0^T}(w|x_0^T)}{p_{\tilde{W}|\tilde{X}_0^T}(w|x_0^T)}\bigg)\,dw \nonumber\\
=& \dkl(X_0^T\|\tilde{X}_0^T)+\sum_{x_0^T}p_{X_0^T}(x_0^T)\dkl(\{W|X_0^T=x_0^T\}\,\|\,\{\tilde{W}|\tilde{X}_0^T=x_0^T\}).
\end{align}

Notice that we used $\int p_{W|X_0^T}(w|x_0^T)dw=1$ in the last step, and $\dkl(\{W|X_0^T=x_0^T\}\,\|\,\{\tilde{W}|\tilde{X}_0^T=x_0^T\})$ is the conditional KL divergence between the law of two conditioned random variables $\{W|X_0^T=x_0^T\}$ and $\{\tilde{W}|\tilde{X}_0^T=x_0^T\}$. Therefore, by non-negativity of KL divergence, combining \cref{eqn:kl1} and \cref{eqn:kl2}, we conclude that 
\[\dkl(X_0^T\|\tilde{X}_0^T)\leq \dkl(X_0^T, W\|\tilde{X}_0^T, \tilde{W}) = \dkl(W\|\tilde{W}).\label{eqn:kl3}\]

By standard results in hypothesis testing, we also know that the best testing error $\mathcal{E}^*$ is 
\[\mathcal{E}^*=\frac{1}{2}(1-\dtv(X_0^T, \tilde{X}_0^T)).\]

Using Pinsker's inequality together with \cref{eqn:kl3}, we have
\[\mathcal{E}^*\geq \frac{1}{2}\bigg(1-\sqrt{\frac{1}{2}\dkl(X_0^T\|\tilde{X}_0^T)}\bigg)\geq \frac{1}{2}\bigg(1-\sqrt{\frac{1}{2}\dkl(W\|\tilde{W})}\bigg).\]

For some sufficiently small constant $\epsilon\in(0,1)$, if we can construct $A$ and $\tilde{A}$ such that $d(A,\tilde{A})\geq \epsilon$, and $\dkl(W\|\tilde{W})= O(\epsilon^2)$, then any algorithm that tests $H_0$ against $H_1$ will incur error that is at least $\frac{1}{2}-O(\epsilon)$. The performance is dismal if $\epsilon$ is small enough (close to randomly guessing which hypothesis is correct). 

We now plug in the mixing measure, \cref{eqn:mixing}, into the KL divergence formula and obtain
\begin{align}
\label{eqn:kl4}
\dkl(W\|\tilde{W})
&=  \int \log\bigg(\frac{d\mu_{v_0,A}(w)}{dw}\bigg\slash \frac{d\mu_{v_0,\tilde{A}}(w)}{dw}\bigg)\,d\mu_{v_0,A}(w)\nonumber \\
&=  \int \log\bigg(\frac{Z_{v_0, \tilde{A}}}{Z_{v_0, A}}\cdot \frac{w_{v_0}^{\frac{1}{2}}\prod_{e\in E} w_e^{a_e-1}}{\prod_{v\in V}w_v^{\frac{1}{2}(a_v+1)}}\cdot \frac{\prod_{v\in V}w_v^{\frac{1}{2}(\tilde{a}_v+1)}}{w_{v_0}^{\frac{1}{2}}\prod_{e\in E} w_e^{\tilde{a}_e-1}}\bigg)\,d\mu_{v_0,A}(w)\nonumber \\
&= \log\frac{Z_{v_0, \tilde{A}}}{Z_{v_0, A}} + \int \log\frac{\prod_{e\in E}w_e^{a_e-\tilde{a}_e}}{\prod_{v\in V}w_v^{\frac{1}{2}(a_v-\tilde{a}_v)}}\,d\mu_{v_0,A}(w).
\end{align}

To calculate the second term in the equation above, we need to use some tricks. Fixing arbitrary $e\in E$ and taking the derivative of both sides of \cref{eqn:normalizer1} with respect to $a_e$, we obtain
\begin{align*}
\frac{\partial}{\partial a_e}Z_{v_0, A}
&= \int_{\mathbb{R}_{++}^{|E|-1}} \frac{\partial}{\partial a_e}\Bigg(\frac{w_{v_0}^{\frac{1}{2}}\prod_{e\in E}w_e^{a_e-1}} {\prod_{v\in V}w_v^{\frac{1}{2}(a_v+1)}}
\Bigg)\sqrt{\sum_{T\in\mathcal{T}}\prod_{e\in T}w_e}\,dw_{-e_0}\\
&= \int_{\mathbb{R}_{++}^{|E|-1}} \bigg(\log w_e-\frac{1}{2}\sum_{v\in e}\log w_v\bigg)\frac{w_{v_0}^{\frac{1}{2}}\prod_{e\in E}w_e^{a_e-1}} {\prod_{v\in V}w_v^{\frac{1}{2}(a_v+1)}}
\sqrt{\sum_{T\in\mathcal{T}}\prod_{e\in T}w_e}\,dw_{-e_0}\\
&= Z_{v_0, A}\int\bigg(\log w_e-\frac{1}{2}\sum_{v\in e}\log w_v\bigg)d\mu_{v_0,A}(w).
\end{align*}

Here we used the fact that $\frac{\partial}{\partial a_e}w_e^{a_e-1}=(\log w_e)w_e^{a_e-1}$. Also, for any $v\in V\backslash\{v_0\}$ that is incident with $e$, we have $\frac{\partial}{\partial a_e}w_v^{-\frac{1}{2}(a_v+1)}=(-\frac{1}{2}\log w_v)w_v^{-\frac{1}{2}(a_v+1)}$; and for $v_0$, we have $\frac{\partial}{\partial a_e}w_{v_0}^{-\frac{1}{2}a_{v_0}}=(-\frac{1}{2}\log w_{v_0})w_{v_0}^{-\frac{1}{2}a_{v_0}}$.

Scale the above equation by $(a_e-\tilde{a}_e)$ and then sum it up over all $e\in E$. We then have
\begin{align}
\label{eqn:kl5}
\sum_{e\in E}(a_e-\tilde{a}_e)\cdot \frac{\partial}{\partial a_e}Z_{v_0, A}
&= Z_{v_0, A}\cdot\sum_{e\in E}\bigg((a_e-\tilde{a}_e)\int\bigg(\log w_e-\frac{1}{2}\sum_{v\in e}\log w_v\bigg)d\mu_{v_0,A}(w)\bigg)\nonumber \\
&= Z_{v_0, A}\cdot\int \bigg(\sum_{e\in E}(a_e-\tilde{a}_e)\log \frac{w_e}{\prod_{v\in e}w_v^{\frac{1}{2}}} \bigg) d\mu_{v_0,A}(w)\nonumber \\
&= Z_{v_0, A}\cdot\int \log \frac{\prod_{e\in E}w_e^{(a_e-\tilde{a}_e)}}{\prod_{e\in E}\prod_{v\in e}w_v^{\frac{1}{2}(a_e-\tilde{a}_e)}} d\mu_{v_0,A}(w)\nonumber \\
&= Z_{v_0, A}\cdot\int \log\frac{\prod_{e\in E}w_e^{(a_e-\tilde{a}_e)}}{\prod_{v\in V}w_v^{\frac{1}{2}(a_v-\tilde{a}_v)}}\,d\mu_{v_0,A}(w).
\end{align}

Here in the last step, we used the fact that 
\[\prod_{e\in E}\prod_{v\in e}w_v^{\frac{1}{2}(a_e-\tilde{a}_e)}=\prod_{v\in V}\prod_{e:v\in e}w_v^{\frac{1}{2}(a_e-\tilde{a}_e)}=\prod_{v\in V}w_v^{\sum_{e:v\in e}\frac{1}{2}(a_e-\tilde{a}_e)}=\prod_{v\in V}w_v^{\frac{1}{2}(a_v-\tilde{a}_v)}.\] 

Therefore, combining \cref{eqn:kl4} and \cref{eqn:kl5}, we conclude that
\begin{equation}
\dkl(W\|\tilde{W})=\log\frac{Z_{v_0, \tilde{A}}}{Z_{v_0, A}}+Z_{v_0,A}^{-1}\sum_{e\in E}(a_e-\tilde{a}_e)\cdot \frac{\partial}{\partial a_e}Z_{v_0, A}.\label{eqn:kl6}
\end{equation}

Recalling the explicit formula for $Z_{v_0, A}$, \cref{eqn:normalizer2}, we have
\begin{align}
\frac{\partial}{\partial a_e}Z_{v_0, A}
=& \frac{\partial}{\partial a_e}\bigg(\frac{\pi^{\frac{1}{2}(|V|-1)}}{2^{1-|V|+\sum_{e\in E}a_e}}\cdot \frac{\prod_{e\in E}\Gamma(a_e)}{\prod_{v\in V}\Gamma(\frac{1}{2}(a_v+1-\ind_{v=v_0}))}\bigg)\nonumber \\
=& \frac{\partial}{\partial a_e}\bigg(\frac{\pi^{\frac{1}{2}(|V|-1)}}{2^{1-|V|+\sum_{e\in E}a_e}}\bigg)\cdot \frac{\prod_{e\in E}\Gamma(a_e)}{\prod_{v\in V}\Gamma(\frac{1}{2}(a_v+1-\ind_{v=v_0}))}\nonumber \\
&\qquad\qquad+ \frac{\pi^{\frac{1}{2}(|V|-1)}}{2^{1-|V|+\sum_{e\in E}a_e}}\cdot \frac{\partial}{\partial a_e}\bigg(\frac{\prod_{e\in E}\Gamma(a_e)}{\prod_{v\in V}\Gamma(\frac{1}{2}(a_v+1-\ind_{v=v_0}))}\bigg)\nonumber \\
=& -\log 2\bigg(\frac{\pi^{\frac{1}{2}(|V|-1)}}{2^{1-|V|+\sum_{e\in E}a_e}}\cdot \frac{\prod_{e\in E}\Gamma(a_e)}{\prod_{v\in V}\Gamma(\frac{1}{2}(a_v+1-\ind_{v=v_0}))}\bigg)\nonumber \\
&+ \Big(\psi(a_e)-\sum_{v\in e}\frac{1}{2}\psi\Big(\frac{1}{2}(a_v+1-\ind_{v=v_0})\Big)\Big)\frac{\pi^{\frac{1}{2}(|V|-1)}}{2^{1-|V|+\sum_{e\in E}a_e}}\cdot \frac{\prod_{e\in E}\Gamma(a_e)}{\prod_{v\in V}\Gamma(\frac{1}{2}(a_v+1-\ind_{v=v_0}))}\nonumber \\
=&\Big(-\log 2 +\psi(a_e)-\sum_{v\in e}\frac{1}{2}\psi\Big(\frac{1}{2}(a_v+1-\ind_{v=v_0})\Big)\Big)Z_{v_0, A}.\label{eqn:kl7}
\end{align}

Here we used the notation for the digamma function $\psi(x):=\frac{d}{dx}\log \Gamma(x)=\frac{\Gamma'(x)}{\Gamma(x)}$. In the third equation, we used $\frac{d}{d x}\Gamma(x)=\Gamma'(x)=\psi(x)\Gamma(x)$, and $\frac{d}{dx}\frac{1}{\Gamma(x)}=-\frac{1}{(\Gamma(x))^2}\Gamma'(x)=-\frac{\psi(x)}{\Gamma(x)}$. Plugging \cref{eqn:kl7} into \cref{eqn:kl6}, and using \cref{eqn:normalizer2} again, we derive that
\begin{align}
\label{eqn:kl8}
\dkl(W\|\tilde{W})
=& \log\frac{Z_{v_0, \tilde{A}}}{Z_{v_0, A}}+\sum_{e\in E}(a_e-\tilde{a}_e) \Big(-\log 2 +\psi(a_e)-\sum_{v\in e}\frac{1}{2}\psi\Big(\frac{1}{2}(a_v+1-\ind_{v=v_0})\Big)\Big)\nonumber \\
=& \log\bigg(\frac{2^{\sum_{e\in E}a_e}}{2^{\sum_{e\in E}\tilde{a}_e}}\cdot \frac{\prod_{e\in E}\Gamma(\tilde{a}_e)}{\prod_{v\in V}\Gamma(\frac{1}{2}(\tilde{a}_v+1-\ind_{v=v_0}))} \cdot \frac{\prod_{v\in V}\Gamma(\frac{1}{2}(a_v+1-\ind_{v=v_0}))}{\prod_{e\in E}\Gamma(a_e)}\bigg)\nonumber \\
& \qquad +\sum_{e\in E}(a_e-\tilde{a}_e) \Big(-\log 2 +\psi(a_e)-\sum_{v\in e}\frac{1}{2}\psi\Big(\frac{1}{2}(a_v+1-\ind_{v=v_0})\Big)\Big)\nonumber \\
=& \sum_{e\in E}(a_e-\tilde{a}_e)\log 2 
+ \sum_{v\in V}\log\frac{\Gamma(\frac{1}{2}(a_v+1-\ind_{v=v_0}))}{\Gamma(\frac{1}{2}(\tilde{a}_v+1-\ind_{v=v_0}))}
+ \sum_{e\in E}\log \frac{\Gamma(\tilde{a}_e)}{\Gamma(a_e)}\nonumber \\
& \qquad +\sum_{e\in E}(a_e-\tilde{a}_e)\Big(-\log 2 +\psi(a_e)-\sum_{v\in e}\frac{1}{2}\psi\Big(\frac{1}{2}(a_v+1-\ind_{v=v_0})\Big)\Big).
\end{align}

We now construct the following instance. Fix a specific edge $\tilde{e}\in E$ that is not incident with $v_0$ and consider a connected $d$-regular graph with vertex set $V=[n]$ together with weights $a_e=1, \forall e\in E$ and $\tilde{a}_e=1+\epsilon \ind(e=\tilde{e}), \forall e\in E$. 
We want to test the hypothesis $H_0:L_0=A$ against $H_1:L_0=\tilde{A}$. We clearly have $d(A,\tilde{A})=\epsilon$. However, by plugging the values into \cref{eqn:kl8}, it is not hard to compute that
\begin{align*}
\dkl(W\|\tilde{W})
=& -\epsilon\log 2 
+ 2\log\frac{\Gamma(\frac{d+1}{2})}{\Gamma(\frac{d+1}{2}+\frac{\epsilon}{2})} + \log \Gamma(1+\epsilon) +\epsilon \Big(\log 2 -\psi(1)+\psi\Big(\frac{d+1}{2}\Big)\Big)\\
=&
- 2\log \Gamma\Big(\frac{d+1}{2}+\frac{\epsilon}{2}\Big)+2\log \Gamma\Big(\frac{d+1}{2}\Big)+ \log \Gamma(1+\epsilon) +\gamma\epsilon+\psi\Big(\frac{d+1}{2}\Big)\epsilon.
\end{align*}

Here we used the fact that $\Gamma(1)=1$ and $\psi(1)=-\gamma$, where $\gamma\approx 0.577$ is the Euler constant. Using the Taylor expansion, it's not hard to see that $\log\Gamma(1+\epsilon)=-\gamma \epsilon+\frac{1}{2}\psi'(1)\epsilon^2+O(\epsilon^3)$, and
\[\log\Gamma\Big(\frac{d+1}{2}+\frac{\epsilon}{2}\Big)=\log\Gamma\Big(\frac{d+1}{2}\Big)+\psi\Big(\frac{d+1}{2}\Big) \frac{\epsilon}{2} + \frac{1}{2}\psi'\Big(\frac{d+1}{2}\Big) \frac{\epsilon^2}{4} +O(\epsilon^3).\]

Using these Taylor approximations, we can compute that
\[\dkl(W\|\tilde{W})
= \bigg(2\psi'(1)-\psi'\bigg(\frac{d+1}{2}\bigg)\bigg)\frac{\epsilon^2}{4}+O(\epsilon^3).\]

Since $\psi'$ is positive on $(0,\infty)$ and strictly decreasing, we have $2\psi'(1)> 2\psi'(1)-\psi'(\frac{d+1}{2})> \psi'(1)>0$ whenever $d\geq 2$. This means that $\dkl(W\|\tilde{W})= O(\epsilon^2)$. By the previous discussion and Pinsker's inequality, we know that any testing algorithm will have test error $\geq \mathcal{E}^*= \frac{1}{2}-O(\epsilon)$, which would be close to $\frac{1}{2}$ when $\epsilon$ is a small constant. Therefore, for any small enough constant $\epsilon>0$, there exists an instance where $d(A,\tilde{A})\geq \epsilon$, but any testing algorithm will have performance close to random guessing in terms of distinguishing whether $H_0$ or $H_1$ is true if only given a single trajectory, even if it is infinitely long.

\begin{remark}[Positive-integer-valued edge weights]
A special case would be the scenario when all the initial edge weights are positive integers within some fixed range $\{\underline{a}, \underline{a}+1, \cdots, \overline{a}-1, \overline{a}\}$ for some $\underline{a}<\overline{a}$ both in $\mathbb{N}_+$. In this case, since there are only $(\overline{a}-\underline{a}+1)^{|E|}$ possible choices of initial edge weights $A$, one may expect to pin it down given enough samples. 

However, essentially the same analysis as above can be used to conclude that a single trajectory would not be enough in this case either. To see this, note that for any two different initial weights $A, \tilde{A}$ in this case, the optimal testing error is $\mathcal{E}^*=\frac{1}{2}(1-\dtv(X_0^T, \tilde{X}_0^T))\geq \frac{1}{2}(1-\dtv(W, \tilde{W}))$, where we used data-processing inequality for total variation distance in the second step. According to the magic formula \cref{eqn:mixing}, $\dtv(W, \tilde{W})$ is going to be some positive quantity bounded away from $1$, because $d\mu_{v_0,A}(w_{-{e_0}})$ and $d\mu_{v_0,\tilde{A}}(w_{-{e_0}})$ share the same support and each has a positive continuous density. Therefore, $\mathcal{E}^*$ will be bounded away from $0$ and a single infinitely long trajectory will not be able to drive the testing error to $0$.
\hfill $\Box$
\end{remark}

This suggests that we have to use multiple trajectories if we want to recover the parameters $A$. 
Based on the coupling with RWRE, we will explore the following natural two-level learning algorithm:
\begin{itemize}
    \item From the independent trajectories $X(1)_0^T, \cdots, X(K)_0^T$, we construct estimators for the corresponding random environments $W^{(1)}, W^{(2)}\cdots, W^{(K)}$, denoted as $\hat{W}^{(1)}, \cdots, \hat{W}^{(K)}$;
    \item Interpret the estimates $\hat{W}^{(1)}, \cdots, \hat{W}^{(K)}$ as approximately i.i.d. samples from the mixing measure $d\mu_{v_0,A}$, and recover $A$ from these approximate samples.
\end{itemize}

To make the algorithm work, we need $T$ large enough so we can estimate $W$'s well enough. At the same time, we also need $K$ large enough so we can then estimate $A$ well enough. Informally speaking, the above algorithm follows from the philosophy that $A\rightarrow W\rightarrow X_0^T$ is a Markov chain, so any information we can infer about $A$ from $X_0^T$ is essentially contained in the information that we know about $W$ from $X_0^T$. Hence, inferring $W$ given $X_0^T$ first and then inferring $A$ from $W$ gives a natural algorithm framework for our purpose. The algorithm we are going to present has a more intricate design than the outlined scheme, but its core concept still centers on leveraging the random environment inherent in ERRW.

\subsection{The maximum likelihood estimator}

Before discussing our algorithm, we deviate a bit to consider the following maximum likelihood estimator and see why it may be difficult to make it work.\footnote{We do not exclude the possibility of using this method in practice. It could still perform well on empirical data despite dismal performance in the worst case, and many local optimization algorithms like gradient descent are very efficient.} Given $K$ trajectories $\{X(i)_0^T: i\in [K]\}$, let's denote the number of undirected edge crossings in 
the trajectory $k \in [K]$ as $\{M_e(k), e\in E\}$. In other words, we have $M_e(k):=L_T^e(k)-L_0^e$, for any $e\in E$ and $k\in[K]$. For every
$k \in [K]$ let $\{N_v(k), v \in V\}$
denote the number of outgoing transitions from
vertex $v$ (note that we have 
$\sum_{v \in V} N_v(k) = T$ for all $k \in [K]$).
Then we have
\begin{equation*}
\Perrw^{v_0,A}(X(k)_0^T=x(k)_0^T, \forall k\leq K)=\prod_{k=1}^K \frac{\prod_{e\in E}\prod_{i=0}^{m_e(k)-1}(a_e+i)}{\prod_{i=0}^{n_{v_0}(k)-1}(a_{v_0}+2i) \prod_{v\in V\backslash\{v_0\}}\prod_{i=0}^{n_v(k)-1}(a_v+2i+1)}.
\end{equation*}

This is rigorously proved as Lemma 3.5 in \cite{kr00}. 
A natural estimator can be obtained via finding the optimal 
$(a^*_e, e \in E)$
that maximizes the log-likelihood. However, finding this optimizer can be computationally hard, as this log-likelihood function is non-concave in the parameters $\{a_e:\,e\in E\}$ in general. 

\subsection{Estimation via a generalized moment method}

To recover $A$, let's construct some statistics to help the estimation procedure. Note that, by the magic formula, fixing a special vertex $v_0\in V$ and a special edge $e_0\in E$ that is incident with $v_0$, we know that
\[d\mu_{v_0,A}(w_{-e_0}) = 
Z_{v_0, A}^{-1}\cdot\frac{w_{v_0}^{\frac{1}{2}}\prod_{e\in E}w_e^{a_e-1}} {\prod_{v\in V}w_v^{\frac{1}{2}(a_v+1)}}
\sqrt{\sum_{T\in\mathcal{T}}\prod_{e\in T}w_e}\,dw_{-e_0},\]
with the normalizing factor $Z_{v_0, A}$ given as in \cref{eqn:normalizer2}. Recall that in this formula we have $w_{e_0} =1$.

Without loss of generality, we can assume that the random walk starts at a vertex $v_0$ that's of degree $\deg(v_0)\geq 2$.\footnote{The case where we start from a degree-one node $v_0$ with initial weights $A$ is equivalent to the case where we start from the 
unique
neighbor of $v_0$ with initial weights $A'$ that differ from $A$ only by $1$ at the unique edge that $v_0$ is incident with.} Denote the transition matrix induced by the random weights $W$ by $(P_{ij})_{i,j\in V}$. Thus $P_{ij}=\frac{W_{ij}}{W_i}$ for any $i,j\in V$. Note that $P_{ij}$ is thought of as a random variable.

Define the matrix $U$ by
\begin{equation}        \label{eqn:udefinition}
U_{ij} :=
\begin{cases} 
 P_{ij}P_{ji}, \quad \text{ if } e =\{i,j\} \in E,\\
 0, \quad \text{ if } i \nsim j.
\end{cases}
\end{equation}
Note that $U_{ij} = U_{ji}$ when $i \sim j$,
so we can write this common value as $U_e$, where
$e = \{i,j\} \in E$.
We remark that similar quantities are also studied and used 
in~\cite{mr2} (e.g., Equation (6)) to prove the magic formula. 
We claim that we then have the following moment formulas involving $\Eerrw^{v_0,A} (\sqrt{U_e})$, $\Eerrw^{v_0,A} (U_e)$, etc. 

\vspace{.3em}
\begin{lem}[Moment formulas]
\label{lem:moment}
    Consider an ERRW on a connected graph $G=(V, E)$. Given a vertex $v_0\in V$ with degree $\deg(v_0)\geq 2$ and initial edge weights $A\in\mathbb{R}_{++}^E$, we sample a random environment $W$ according to the mixing measure in \cref{eqn:mixing}. Denote the random transition matrix induced by the random weight $W$ as $P=(P_{i,j})_{i,j\in V}$, and let $U_e:=P_{ij}P_{ji}$ for any $e=\{i,j\}\in E$. Then we have
    \begin{align*}
    \Eerrw^{v_0,A}(\sqrt{U_e})
    =& \frac{a_e}{2}\cdot \frac{\Gamma(\frac{1}{2}(a_i+1-\ind_{i=v_0}))\Gamma(\frac{1}{2}(a_j+1-\ind_{j=v_0}))}{\Gamma(\frac{1}{2}(a_i+2-\ind_{i=v_0}))\Gamma(\frac{1}{2}(a_j+2-\ind_{j=v_0}))},\, \forall e=\{i,j\}\in E;\\
    \Eerrw^{v_0,A} (U_e)
    =& \frac{a_e(a_e+1)}{(a_i+1-\ind_{i=v_0})(a_j+1-\ind_{j=v_0})}, \hspace{6.3em} \forall e=\{i,j\}\in E;\\
    \Eerrw^{v_0,A}(U_e^2)
    =& \frac{a_e(a_e+1)(a_e+2)(a_e+3)}{(a_i+1-\ind_{i=v_0})(a_i+3-\ind_{i=v_0})(a_j+1-\ind_{j=v_0})(a_j+3-\ind_{j=v_0})},\\
    & \hspace{20.3em} \forall e=\{i,j\}\in E;\\
    \Eerrw^{v_0,A} (U_e U_{e'})
    =& \frac{a_e(a_e+1)}{(a_i+1-\ind_{i=v_0})(a_j+1-\ind_{j=v_0})}\cdot \frac{a_{e'}(a_{e'}+1)}{(a_k+1-\ind_{k=v_0})(a_l+1-\ind_{l=v_0})}, \\
    & \hspace{16em} \forall e=\{i,j\}, e'=\{k,l\}\in E;\\
    \Eerrw^{v_0,A} (U_e U_{e'})
    &= \frac{a_e(a_e+1)a_{e'}(a_{e'}+1)}{(a_i+1-\ind_{i=v_0})(a_k+1-\ind_{k=v_0})(a_j+1-\ind_{j=v_0})(a_j+3-\ind_{j=v_0})}, \\
    & \hspace{16em} \forall e=\{i,j\}, e'=\{j,k\}\in E.\\
    \end{align*}
    Here $\Gamma(x), \forall x>0$ is the Gamma function, and in the conditions above we assume that $i,j,k,l$ are mutually distinct vertices.
\end{lem}

\begin{proof}
Consider some edge $e'=\{i,j\}$, we know that $\sqrt{P_{ij}P_{ji}}=\frac{W_{e'}}{\sqrt{W_iW_j}}$, and
\begin{equation}
\label{eqn:moment1}
\Eerrw^{v_0,A} (\sqrt{U_{e'}}) = \int \frac{w_{e'}}{\sqrt{w_iw_j}} d\mu_{v_0,A}(w) = Z_{v_0, A}^{-1}\int \frac{w_{v_0}^{\frac{1}{2}}\prod_{e\in E}w_e^{a_e'-1}} {\prod_{v\in V}w_v^{\frac{1}{2}(a_v'+1)}}
\sqrt{\sum_{T\in\mathcal{T}}\prod_{e\in T}w_e}\,dw_{-e_0}.
\end{equation}

In the above equation, we let $a_e':=a_e+\ind(e=e'), \forall e\in E$. From \cref{eqn:normalizer1}, we know that
\[\int \frac{w_{v_0}^{\frac{1}{2}a_{v_0}'}\prod_{e\in E}w_e^{a_e'-1}} {\prod_{v\in V}w_v^{\frac{1}{2}(a_v'+1)}}
\sqrt{\sum_{T\in\mathcal{T}}\prod_{e\in T}w_e}\,dw_{-e_0}=Z_{v_0,A'},\]hence we conclude that
\begin{align}
\label{eqn:moment2}
\Eerrw^{v_0,A}(\sqrt{U_{e'}})
= \frac{Z_{v_0, A'}}{Z_{v_0, A}}
&= \frac{2^{\sum_{e\in E}a_e}}{2^{\sum_{e\in E}a_e'}}\cdot \frac{\prod_{v\in V}\Gamma(\frac{1}{2}(a_v+1-\ind_{v=v_0}))}{\prod_{e\in E}\Gamma(a_e)}\cdot \frac{\prod_{e\in E}\Gamma(a_e')}{\prod_{v\in V}\Gamma(\frac{1}{2}(a_v'+1-\ind_{v=v_0}))}\nonumber \\
&= 2^{\sum_{e\in E}(a_e-a_e')}\cdot \frac{\prod_{v\in V}\Gamma(\frac{1}{2}(a_v+1-\ind_{v=v_0}))\prod_{e\in E}\Gamma(a_e')}{\prod_{v\in V}\Gamma(\frac{1}{2}(a_v'+1-\ind_{v=v_0})) \prod_{e\in E}\Gamma(a_e)}\nonumber \\
&= \frac{1}{2}\cdot \frac{\Gamma(\frac{1}{2}(a_i+1-\ind_{i=v_0}))\Gamma(\frac{1}{2}(a_j+1-\ind_{i=v_0}))}{\Gamma(\frac{1}{2}(a_i+2-\ind_{i=v_0}))\Gamma(\frac{1}{2}(a_j+2-\ind_{j=v_0}))}\cdot \frac{\Gamma(a_{e'}+1)}{\Gamma(a_{e'})}\nonumber \\
&= \frac{a_{e'}}{2}\cdot \frac{\Gamma(\frac{1}{2}(a_i+1-\ind_{i=v_0}))\Gamma(\frac{1}{2}(a_j+1-\ind_{j=v_0}))}{\Gamma(\frac{1}{2}(a_i+2-\ind_{i=v_0}))\Gamma(\frac{1}{2}(a_j+2-\ind_{j=v_0}))}.
\end{align}

In the second equality above, we used \cref{eqn:normalizer2} again.  This proves the first equation in \cref{lem:moment}. Similarly, we also have
\[\Eerrw^{v_0,A} (U_{e'})
= \int \frac{w_{e'}^2}{w_iw_j} d\mu_{v_0,A}(w)= \frac{Z_{v_0, A''}}{Z_{v_0, A}}.\]
Here $a_e''=a_e+2\cdot\ind(e=e'), \forall e\in E$. We can then use \cref{eqn:normalizer2} to conclude that
\begin{align}
\label{eqn:moment3}
\Eerrw^{v_0,A} (U_{e'})
&= \frac{1}{4}\cdot \frac{\Gamma(\frac{1}{2}(a_i+1-\ind_{i=v_0}))\Gamma(\frac{1}{2}(a_j+1-\ind_{i=v_0}))}{\Gamma(\frac{1}{2}(a_i+3-\ind_{i=v_0}))\Gamma(\frac{1}{2}(a_j+3-\ind_{j=v_0}))}\cdot \frac{\Gamma(a_{e'}+2)}{\Gamma(a_{e'})}\nonumber \\
&= \frac{1}{4}\cdot \frac{1}{\frac{1}{2}(a_i+1-\ind_{i=v_0})\cdot \frac{1}{2}(a_j+1-\ind_{j=v_0})}\cdot a_{e'}(a_{e'}+1)\nonumber \\
&= \frac{a_{e'}(a_{e'}+1)}{(a_i+1-\ind_{i=v_0})(a_j+1-\ind_{j=v_0})}.
\end{align}

Here we used the fact that $\frac{\Gamma(x+1)}{\Gamma(x)}=x$ for all $x>0$. This proves the second equation in the lemma.

Now let us fix any two edges $e'=\{i,j\}, e''=\{k,l\}\in E$ s.t. $|e'\cap e''|=0$, and consider the random variable $U_{e'} U_{e''}$. Since $|e'\cap e''|=0$, it's straightforward to compute that 
\begin{align}
\label{eqn:moment4}
\Eerrw^{v_0,A} (U_{e'} U_{e''})
=& \int \frac{w_{e'}^2}{w_iw_j}\cdot \frac{w_{e''}^2}{w_kw_l} d\mu_{v_0,A}(w) = \frac{Z_{v_0, A^{(3)}}}{Z_{v_0,A}}\nonumber \\
=& \frac{1}{4}\, \frac{\Gamma(\frac{1}{2}(a_i+1-\ind_{i=v_0}))\Gamma(\frac{1}{2}(a_j+1-\ind_{i=v_0}))}{\Gamma(\frac{1}{2}(a_i+3-\ind_{i=v_0}))\Gamma(\frac{1}{2}(a_j+3-\ind_{j=v_0}))}\cdot \frac{\Gamma(a_{e'}+2)}{\Gamma(a_{e'})}\nonumber \\
& \cdot \frac{1}{4}\, \frac{\Gamma(\frac{1}{2}(a_k+1-\ind_{i=v_0}))\Gamma(\frac{1}{2}(a_l+1-\ind_{i=v_0}))}{\Gamma(\frac{1}{2}(a_k+3-\ind_{i=v_0}))\Gamma(\frac{1}{2}(a_l+3-\ind_{j=v_0}))}\cdot \frac{\Gamma(a_{e''}+2)}{\Gamma(a_{e''})}\nonumber \\
=& \frac{a_{e'}(a_{e'}+1)}{(a_i+1-\ind_{i=v_0})(a_j+1-\ind_{j=v_0})}\cdot \frac{a_{e''}(a_{e''}+1)}{(a_k+1-\ind_{k=v_0})(a_l+1-\ind_{l=v_0})}.
\end{align}

The intermediate steps are fairly similar to those in \cref{eqn:moment1} and \cref{eqn:moment2}, while we have $a_e^{(3)}=a_e+2\cdot\ind(e=e')+2\cdot \ind(e=e'')$ here. On the other hand, if $e'=\{i,j\}$, $e''=\{j,k\}$ for $i\neq k$, i.e., $|e\cap e'|=1$, then 
\begin{align}
\label{eqn:moment5}
\Eerrw^{v_0,A} (U_{e'} U_{e''})
=& \int \frac{w_{e'}^2}{w_iw_j}\cdot \frac{w_{e''}^2}{w_jw_k} d\mu_{v_0,A}(w)\nonumber \\
=& \int \frac{w_{e'}^2w_{e''}^2}{w_iw_j^2w_k} d\mu_{v_0,A}(w) =\frac{Z_{v_0,A^{(3)}}}{Z_{v_0,A}}\nonumber \\
=& \frac{1}{16}\, \frac{\Gamma(\frac{1}{2}(a_i+1-\ind_{i=v_0}))\Gamma(\frac{1}{2}(a_k+1-\ind_{k=v_0}))}{\Gamma(\frac{1}{2}(a_i+3-\ind_{i=v_0}))\Gamma(\frac{1}{2}(a_k+3-\ind_{k=v_0}))}\nonumber \\
&\qquad\qquad \cdot \frac{\Gamma(\frac{1}{2}(a_j+1-\ind_{i=v_0}))}{\Gamma(\frac{1}{2}(a_j+5-\ind_{j=v_0}))}\cdot \frac{\Gamma(a_{e'}+2)}{\Gamma(a_{e'})}\cdot \frac{\Gamma(a_{e''}+2)}{\Gamma(a_{e''})}\nonumber \\
=& \frac{a_{e'}(a_{e'}+1)a_{e''}(a_{e''}+1)}{(a_i+1-\ind_{i=v_0})(a_k+1-\ind_{k=v_0})(a_j+1-\ind_{j=v_0})(a_j+3-\ind_{j=v_0})}.
\end{align}

Notice that since $e'$ and $e''$ have an intersection, so we have $a_j^{(3)}=a_j+4$ in this case, compared to the previous scenario. 

Similarly, we can compute $\Eerrw^{v_0,A}(U_{e'}^2)$ as
\begin{align}
\label{eqn:moment6}
\Eerrw^{v_0,A}(U_{e'}^2)
=& \int \frac{w_{e'}^4}{w_i^2w_j^2} d\mu_{v_0,A}(w) \nonumber \\
=& \frac{1}{16}\, \frac{\Gamma(\frac{1}{2}(a_i+1-\ind_{i=v_0}))\Gamma(\frac{1}{2}(a_j+1-\ind_{i=v_0}))}{\Gamma(\frac{1}{2}(a_i+5-\ind_{i=v_0}))\Gamma(\frac{1}{2}(a_j+5-\ind_{j=v_0}))}\cdot \frac{\Gamma(a_{e'}+4)}{\Gamma(a_{e'})}\nonumber \\
=& \frac{a_{e'}(a_{e'}+1)(a_{e'}+2)(a_{e'}+3)}{(a_i+1-\ind_{i=v_0})(a_i+3-\ind_{i=v_0})(a_j+1-\ind_{j=v_0})(a_j+3-\ind_{j=v_0})}.
\end{align}

The intermediate steps are again similar to those in \cref{eqn:moment1} and \cref{eqn:moment2}.
\end{proof}

If we could obtain estimates for each $\Eerrw^{v_0,A}(U_e)$, $e \in E$, from multiple trajectories of the reinforced random walk, then we can try to solve the set of equations
\begin{equation*}
\begin{aligned}
&\Eerrw^{v_0,A} (U_e)=\frac{a_e(a_e+1)}{(a_i+1-\ind_{i=v_0})(a_j+1-\ind_{j=v_0})}, \quad& \forall e = \{i,j\} \in E,\\
&a_v=\sum_{e:v\in e} a_e, \quad& \forall v\in V.
\end{aligned} 
\end{equation*}
with $\Eerrw^{v_0,A}(U_e)$ substituted for by our estimate for it. However, it is non-trivial to solve for the initial weights $A$ directly from this set of equations. But from these equations, we know that if we can approximate the value $(a_i+1-\ind_{i=v_0})(a_j+1-\ind_{j=v_0})\Eerrw^{v_0, A} (U_e)$, for
$e = \{i,j\}$, within multiplicative factor $(1\pm \epsilon)$, then we can approximate $a_e(a_e+1)$ within the same multiplicative factor, and moreover, we should be able to recover $a_e$ up to multiplicative error $(1\pm \epsilon)$ since $(1-\epsilon)x(x+1)\leq \hat{x}(\hat{x}+1)\leq (1+\epsilon)x(x+1), x,\hat{x}> 0$ implies $\hat{x}\in (1\pm\epsilon)x$.\footnote{
Note that if $0<\hat{x}< (1-\epsilon)x$, then $\hat{x}^2+\hat{x}< (1-\epsilon)^2x^2+(1-\epsilon)x 
<(1-\epsilon)x^2+(1-\epsilon)x$, a contradiction. Similarly, if $\hat{x}>(1+\epsilon)x$, then $\hat{x}^2+\hat{x}>(1+\epsilon)^2x+(1+\epsilon)x>(1+\epsilon)x^2+(1+\epsilon)x$, a contradiction. Therefore, we must have $\hat{x}\in(1\pm\epsilon)x$. } To estimate terms like $(a_i+1-\ind_{i=v_0})(a_j+1-\ind_{j=v_0})\Eerrw^{v_0, A} (U_e)$, for $e = \{i,j\}$, up to some multiplicative error, it suffices to approximate $\{\Eerrw^{v_0,A}(U_e), \forall e\in E\}$ and $\{o_i:=a_i+1-\ind_{i=v_0}, \forall i \in V\}$ up to some multiplicative error $\epsilon$ respectively. 

Consider two edges $e=\{i,j\}, e'=\{j,k\}$ such that $i\neq k$. Recall that by \cref{eqn:moment3} we have
\[\Eerrw^{v_0,A} (U_e)=\frac{a_e(a_e+1)}{(a_i+1-\ind_{i=v_0})(a_j+1-\ind_{j=v_0})},\quad \Eerrw^{v_0,A} (U_{e'})=\frac{a_{e'}(a_{e'}+1)}{(a_j+1-\ind_{j=v_0})(a_k+1-\ind_{k=v_0})}.\]

Using these in \cref{eqn:moment5}, we obtain
\[\Eerrw^{v_0,A} (U_eU_{e'})=\Eerrw^{v_0,A} (U_e)\Eerrw^{v_0,A} (U_{e'})\cdot \frac{a_j+1-\ind_{j=v_0}}{a_j+3-\ind_{j=v_0}}=\Eerrw^{v_0,A} (U_e)\Eerrw^{v_0,A} (U_{e'})\cdot \frac{o_j}{o_j+2}, \]
where we
recall that we defined $o_j:=a_j+1-\ind_{j=v_0}$. This equation indicates that we can estimate $o_j$'s via estimating the second moment $V_{e,e'}:=\Eerrw^{v_0,A} (U_eU_{e'})$ and the first moments $\Eerrw^{v_0,A}(U_e)$, $\Eerrw^{v_0,A} (U_{e'})$. Specifically, $\forall e,e'\in E$, if $e$ and $e'$ intersect at one vertex $j$, then

\begin{equation}
\label{eqn:degtwo}
o_j=\frac{2\Eerrw^{v_0,A} (U_eU_{e'})}{\Eerrw^{v_0,A}(U_e)\Eerrw^{v_0,A} (U_{e'})-\Eerrw^{v_0,A} (U_eU_{e'})}.
\end{equation}

From the moment formulas in \cref{lem:moment}, one can see that $\Eerrw^{v_0,A}(U_e)\Eerrw^{v_0,A} (U_{e'})\geq \Eerrw^{v_0,A} (U_eU_{e'})$, i.e., $U_e$ and $U_{e'}$ are negatively correlated.
Therefore, we would like to approximate $\Eerrw^{v_0,A} (U_eU_{e'})$ and $\Eerrw^{v_0,A}(U_e)\Eerrw^{v_0,A} (U_{e'})-\Eerrw^{v_0,A} (U_eU_{e'})$ up to multiplicative error $\epsilon$ in order to approximate $o_j$ up to multiplicative error $2\epsilon$. This procedure can be used to estimate $o_j$ for any vertex $j$ with $\deg(j)\geq 2$, since in this case, we can find two distinct edges that are incident at $j$.
\footnote{In our algorithm, we used an arbitrary pair of such $e,e'$ to construct an estimator for $o_j$. We remark that one can incorporate more pairs of $e,e'\in E$ that intersect at $j$ to possibly design a better estimator for $o_j$.}

On the other hand, for any vertex $j$ of degree one and the only edge incident with it, $e=\{i,j\}$, we have 
\begin{equation*}
\Eerrw^{v_0,A}(U_e)=\frac{a_e(a_e+1)}{(a_i+1-\ind_{i=v_0})(a_j+1-\ind_{j=v_0})}=\frac{a_j(a_j+1)}{o_i(a_j+1)}=\frac{a_j}{o_i}.
\end{equation*}

Here we used $a_j=a_e$ since $e$ is the only edge that incident with $j$. We also used the fact that $j\neq v_0$ since $\deg(j)=1$ while $\deg(v_0)\geq 2$ by assumption. Therefore, if we can estimate $\Eerrw^{v_0,A}(U_e)$ and $o_i$, we can then estimate $o_j=a_j+1$ via
\begin{equation}
\label{eqn:degone}o_j=a_j+1-\ind_{j=v_0}=a_j+1=o_i\Eerrw^{v_0,A}(U_e)+1.
\end{equation}
Since vertex $i$ must have degree $\deg(i)\geq 2$ given that $\deg(j)=1$ (else the graph has only a single edge and is not of interest), we can indeed estimate $o_i$ via the procedure discussed above. 

Fix some small constant $\epsilon\in(0,1)$. Assume that for any $e\in E$, we can obtain an $\epsilon$-multiplicative estimator for $\Eerrw^{v_0,A}(U_e)$, denoted as $\hat{U}_e$. 
Also assume that for any $e,e'\in E$
 such that $|e\cap e'|=1$ we can obtain an $\epsilon$-multiplicative estimator for $\Eerrw^{v_0,A}(U_eU_{e'})$ as $\hat{V}_{e,e'}$, and an $\epsilon$-multiplicative estimator for $\Delta_{e,e'}:=\Eerrw^{v_0,A}(U_e)\Eerrw^{v_0,A}(U_{e'})-\Eerrw^{v_0,A}(U_eU_{e'})$ denoted as $\hat{\Delta}_{e,e'}$. 
We can then estimate the initial weights $(a_e, e \in E)$ via the following procedure.
\vspace{1em}
\begin{breakablealgorithm}
\caption{Estimation via generalized moments}\label{alg:routine1}
\begin{algorithmic}[1]
\Require Estimators $\{\hat{U}_e,\,e\in E\}$ and $\{\hat{V}_{e,e'},\, e,e'\in E\}$, $\{\hat{\Delta}_{e,e'},\,e,e'\in E\}$; 
\State For every node $i\in V$ with $\deg(i)\geq 2$, choose two edges $e,e'$ that are incident at $i$. 
Plug in estimators of $\Eerrw^{v_0,A}(U_eU_{e'})$ and $\Eerrw^{v_0,A}(U_e)\Eerrw^{v_0,A}(U_{e'})-\Eerrw^{v_0,A}(U_eU_{e'})$ into \cref{eqn:degtwo} to obtain an estimator for $o_i$ as:
\begin{equation}
\label{eqn:est-two}
\hat{o}_i=\frac{2\hat{V}_{e,e'}}{\hat{\Delta}_{e,e'}}.
\end{equation}
\State For every node $j\in V$ with $\deg(j)=1$, let the only edge that is incident with $j$ be $e=\{i,j\}$. Plug in estimators of $\Eerrw^{v_0,A}(U_e)$, and $o_i$ (obtained in step 1, since $i$ should have degree at least $2$) into \cref{eqn:degone} to obtain an estimator for $o_j$ as:
\begin{equation}
\label{eqn:est-one}\hat{o}_j=\hat{o}_i\hat{U}_e+1.
\end{equation}
\State For every edge $e=\{i,j\}\in E$, obtain an estimator for $a_e$ by solving the equation
    \begin{equation}
    \label{eqn:solve_ae}
    x(x+1)=\hat{o}_i\hat{o}_j\hat{U}_e.
    \end{equation}
    Take the unique positive solution as the estimator for $a_e$:
    \begin{equation}
    \label{eqn:est-ae}
    \hat{a}_e=-\frac{1}{2} + \sqrt{\frac{1}{4} + \hat{o}_i \hat{o}_j \hat{U}_e}.
    \end{equation}
\end{algorithmic}
\end{breakablealgorithm}

We note that $\hat{\Delta}_{e,e'}$ is essentially an estimator for the covariance $\covar(U_e, U_{e'})$, albeit with a sign flip. Assuming that the estimators $\hat{U}_e$ and $\hat{o}_i$, $\hat{o}_j$ are all positive (as will be evident from \cref{alg:routine2}), \cref{eqn:solve_ae} is expected to yield one positive solution and one negative solution. Here, we will consider only the positive solution. In the next section, we will see how one can estimate the moments that we used to construct the above procedure. Specifically, we will give estimators 
$\{\hat{U}_e,\,e\in E\}$, $\{\hat{V}_{e,e'},\, e,e'\in E\}$, and $\{\hat{\Delta}_{e,e'},\,e,e'\in E\}$
that entry-wise approximate the corresponding quantities up to some small multiplicative error, 
and which will be seen to satisfy the positivity requirement.

\subsubsection{Estimating the Moments}  \label{sec:estimatingmoments}

Let's now consider the problem of estimating the following moments: 
\begin{itemize}
\item[i.] $\Eerrw^{v_0,A}(U_e), \forall e\in E$;
\item[ii.] $\Eerrw^{v_0,A}(U_e U_{e'}), \forall e,e'\in E\text{ s.t. }\,|e\cap e'|=1$;
\item[iii.] $\Eerrw^{v_0,A}(U_e)\Eerrw^{v_0,A}(U_{e'})-\Eerrw^{v_0,A}(U_eU_{e'}), \forall e,e'\in E\text{ s.t. }\,|e\cap e'|=1$.
\end{itemize}
Recall that we are given $K$ independent trajectories $\{X_t^{(k)}|0 \le t\leq T, 1 \le k\leq K\}$ from $\Perrw^{v_0, A}$. Via coupling with RWRE, we denote the random conductance corresponding to the $k$-th trajectory as $W^{(k)}$, and the corresponding transition matrix (also a random variable) as $P^{(k)}$. Let \(N_{ij}(k)\) denote the number of directed edge crossings over the edge \(e = \{i,j\}\), from \(i\) to \(j\), in the \(k\)-th trajectory. For any 
\(m \ge 1\),
let \(H_{ij}(k; m)\) represent the number of such transitions among the first \(m\) outgoing transitions from \(i\) in the \(k\)-th trajectory. Additionally, let \(N_i(k)\) denote the total number of outgoing transitions from \(i\) in the \(k\)-th trajectory. A natural estimator for $P^{(k)}$ is the simple counting-based estimator: 

\begin{equation}
\label{eqn:est-p}
\hat{P}_{ij}^{(k)}:=
\begin{cases}
    \frac{H_{ij}(k; m)}{m}, & \text{if } N_i(k) \geq m\text{ and } i\sim j; \\
    \frac{1}{\deg(i)}, & \text{if } N_i(k) < m\text{ and } i\sim j; \\
    0, & \text{otherwise}.
\end{cases}
\quad\forall i, j\in V.
\end{equation}
Note that $\hat{P}_{ij}^{(k)}$ depends on 
the choice of $m \ge 1$, but this dependency has been 
suppressed from the notation for readability.
Similarly, a natural estimator for $U^{(k)}$ is then
\begin{equation}
\label{eqn:est-uk}
\hat{U}_{ij}^{(k)}:= \begin{cases} \hat{P}_{ij}^{(k)}\hat{P}_{ji}^{(k)},\quad\forall i,j\in V\text{ s.t. }i\sim j,\\
0, \quad \text{ if } i \nsim j.
\end{cases}
\end{equation}
Note that $\hat{U}_{ij}^{(k)}$ also depends on the choice of $m \ge 1$. Also, since $\hat{U}_{ij}^{(k)}
= \hat{U}_{ji}^{(k)}$ for all $i \sim j$, we can 
write this common value as $\hat{U}_e^{(k)}$,
where $e = \{i,j\} = \{j,i\}$. 

For any $e\in E$, we can average over all $\{\hat{U}_e^{(k)},\,1\leq k\leq K\}$ to obtain an estimator for $\Eerrw^{v_0, A}(U_e)$. We summarize the estimation procedure as follows.
\vspace{1em}
\begin{breakablealgorithm}
\caption{Estimating the moments}\label{alg:routine2}
\begin{algorithmic}[1]
\Require 
The trajectories $\{X(k)_0^T,\,1\leq k\leq K\}$ and the parameter $m \ge 1$; 
\State For each trajectory, compute the local times and the corresponding estimator $\hat{P}^{(k)}$ according to \cref{eqn:est-p}, as well as the empirical $U$-matrix $\hat{U}^{(k)}$ according to \cref{eqn:est-uk};
\State Compute the average empirical U-matrix as 
\begin{equation}
\label{eqn:est-u}
\hat{U}_{ij}:=K^{-1}\sum_{k=1}^K \hat{U}_{ij}^{(k)}\quad\forall \{i,j\}\in E;
\end{equation}
Note that since $\hat{U}_{ij} = \hat{U}_{ji}$
for all $i \sim j$, we can write this common value as $\hat{U}_e$, where $e = \{i,j\} = \{j,i\}$.
\State For any $e,e'\in E$ s.t. $|e\cap e'|=1$, compute 
\begin{equation}
\label{eqn:est-v}
\hat{V}_{e,e'}:=K^{-1}\sum_{k=1}^K \hat{U}_e^{(k)}\hat{U}_{e'}^{(k)},
\end{equation}
as well as 
\begin{equation}
\label{eqn:est-delta}
\hat{\Delta}_{e,e'}:=\begin{cases}
& \hat{U}_e\hat{U}_{e'}-\hat{V}_{e,e'}, \quad\text{ if } \hat{U}_e\hat{U}_{e'}>\hat{V}_{e,e'};\\
& 1, \quad\text{ otherwise. }
\end{cases}
\end{equation}
Note that each $\hat{U}_e^{(k)}$, $\hat{U}_e$,
$\hat{V}_{e,e'}$, and $\hat{\Delta}_{e,e'}$
depends on the choice of $m \ge 1$.
\end{algorithmic}
\end{breakablealgorithm}

Given the above two procedures, the overall estimation scheme is just chaining them together by feeding the outputs of \cref{alg:routine2} to \cref{alg:routine1}.\footnote{
In the analysis later, we use the fact that the probability of $\{\exists\,e,e'\in E,\, |e \cap e'| =1,\, \hat{U}_e\hat{U}_{e'}\leq \hat{V}_{e,e'}\}$ is eventually smaller than the confidence level $\delta$ that we chose (when the parameters $K$, $T$, and $m$ are large enough). So letting $\hat{\Delta}_{e,e'}=1$ here is simply a placeholder for an error event. }  We denote the average \(U\)-matrix as \(\overline{U} = K^{-1} \sum_{k=1}^K U^{(k)}\). 
(Here the matrix $U^{(k)}$ is defined as in \cref{eqn:udefinition}.) Informally, for any \(e \in E\), we expect \(\overline{U}_e\) to converge in probability to \(\mathbb{E}_{\mathrm{ERRW}}^{v_0, A}(U_e)\) as \(K \to \infty\). Similarly, we anticipate \(\hat{U}_e \to \overline{U}_e\) in probability as \(T \to \infty\) and \( m \to \infty\) with $m$ being sufficiently smaller than $T$ that the probability of having made at least $m$ transitions out of every vertex approaches $1$ for any fixed trajectory. Consequently, as \(K, T, m \to \infty\), it is natural to expect that \(\hat{U}_e\) converges to \(\mathbb{E}_{\mathrm{ERRW}}^{v_0, A}(U_e)\) in probability. A non-asymptotic sample complexity analysis is provided in the next section.

\subsection{Non-asymptotic analysis}    \label{sec:nonasymptotic}

To obtain non-asymptotic bounds on sample complexity, it is essential to deepen our understanding of the
random conductance
associated with ERRWs. In particular, leveraging a ``lifted'' magic formula that links the random conductance to the so-called hyperbolic Gaussian densities allows us to derive tighter bounds on the cover time of the graph and hence yields a 
sample complexity bound
for our estimation problem. The analysis of the 
cover time is needed to understand how big we can choose the parameter $m$ as $T$ grows.

The observation we exploit is that, given a connected simple graph $G=(V, E)$, and initial edge local times $A=(a_e)_{e\in E}$, ERRW starting from node $v_0$ is equivalent in distribution to the simple random walk with random conductance $Q$ determined by the following process:
\begin{enumerate}
    \item Generate independent 
    \(
    \beta_e\sim\Gamma(a_e, 1)
    \)
    for each edge $e\in E$.
    \item Given $(\beta_e)_{e\in E}$, fixing $\phi_{v_0}=0$, generate $\{\phi_i:\,i\neq 0\}$ according to the following hyperbolic Gaussian density:
    \begin{equation*}
        p_\beta(\phi)=\frac{1}{(2\pi)^{(n-1)/2}}\,e^{-\sum_{e = \{i,j\} \in E}\beta_e(\cosh(\phi_i-\phi_j)-1)}e^{-\sum_{i\in V\backslash \{v_0\} }\phi_i} \sqrt{\sum_{T\in\mathcal{T}}\prod_{e = \{i,j\} \in T}\beta_e e^{\phi_i+\phi_j}}.
    \end{equation*}
    \item Perform a simple random walk with initial weights 
    \[Q_e:=\beta_e e^{\phi_i+\phi_j},\quad \forall e = \{i,j\} \in E.\] 
\end{enumerate}
We remark that $Q$ is related to the random environment $W$ (\cref{eqn:mixing}) in the sense that $\{Q_e/Q_{e_0}:\,e\in E\}$ will be equal in distribution to $\{W_e:\,e\in E\}$, as is proven in Section 5 of \cite{st}. 

To distinguish the measure on $Q$ from that on $W$, we use $\nu_{v_0, A}$ to denote the measure on $Q$. Note that there is no need to choose any edge $e_0$ incident on $v_0$ in the definition of $\nu_{v_0, A}$.

The process above directly reflects the probabilistic essence of Theorem 1 and Theorem 2 in~\cite{st}.
\footnote{Theorem 2 in \cite{st} uses a normalization condition that restricts the $\phi_i$'s to the set $\{\phi:\sum_{i}\phi_i=0\}$, while we restrict them to  the set $\{\phi:\phi_{v_0}=0\}$. One can deduce the formula we presented here from the one in~\cite{st}
by a simple change-of-variable. 
} It offers an alternative perspective on the mixing measure, uncovering the independent Gamma field (which we dub the $\beta$-field) and the hyperbolic Gaussian field (which we dub the $\phi$-field) embedded within the mixing measure. This perspective, highlighting both the independence and the (hyperbolic) Gaussian nature, provides a more effective approach to bounding the fluctuations of the random conductance.

In this subsection, let \( Q \) represent the random environment coupled with the ERRW. We assume that \( \beta \) and \( \phi \) are also naturally coupled with \( Q \) and the ERRW, consistent with the process described above.

Before proceeding with the analysis, we introduce the standard notations 
\(O\), \(\Omega\), and \(\Theta\) to describe growth rates of functions:
\begin{enumerate}
    \item \(f(x) = O(g(x))\): \(f(x) \leq c \cdot g(x)\) for constants \(c > 0\), \(x \geq x_0\) (asymptotic upper bound);
    \item \(f(x) = \Omega(g(x))\): \(f(x) \geq c \cdot g(x)\) for constants \(c > 0\), \(x \geq x_0\) (asymptotic lower bound);
    \item \(f(x) = \Theta(g(x))\): \(c_1 \cdot g(x) \leq f(x) \leq c_2 \cdot g(x)\) for \(c_1, c_2 > 0\), \(x \geq x_0\) (tight asymptotic bound up to constant factors).
\end{enumerate}

\subsubsection{Hyperbolic Gaussian measure on a tree}

We begin by 
studying the supremum of the absolute value
of the $\phi$-field on a tree.
Studying a tree serves as a simple case that illustrates the core idea needed in the case of general graphs.
Consider an arbitrary orientation of each edge, where for any edge $e = \{i,j\}$, we select either the directed edge $\vv{e} = (i \to j)$ or $\vv{e} = (j \to i)$. Let $\vv{E}$ denote the set of directed edges under this orientation.  

Next, we define the gradient field of the $\phi$-field as $\nabla \phi: \vv{E} \to \mathbb{R}$, where for each directed edge $\vv{e} = (i \to j) \in \vv{E}$, the gradient is given by $\nabla \phi(\vv{e}) := \phi_i - \phi_j$. This forms an edge field $\{ \nabla \phi(\vv{e}) : \vv{e} \in \vv{E} \}$, which can be shown to be comprised of mutually independent random variables when the underlying graph is a tree.\footnote{Readers familiar with the Gaussian Free Field (GFF) may recognize similarities to the gradient-GFF on a tree; the gradient components of GFF on a tree are also independent.} In the following lemma and its proof, we identify $a_{\vv{e}}$ with $a_e$, where $e$ is the unoriented edge corresponding to $\vv{e}\in E$. Similarly, we define $\beta_{\vv{e}}$ to be the same as its unoriented counterpart $\beta_e$. We use $\{y_{\vv{e}} : \vv{e}\in \vv{E}\}$ to denote a realization of $\{ \nabla \phi(\vv{e}) : \vv{e} \in \vv{E} \}$.

\vspace{.3em}
\begin{lem}[Independence of the gradient field]
\label{lem:grad_indep}
Consider a tree $G=(V=[n], E)$ with a root $v_0\in V$, and positive edge weights given by $A:=(a_e)_{e\in E}$. Orient all the edges towards the root to obtain $\vv{E}$. Let the $\phi$-field on this graph be $\{\phi_i:\,i\in V\}$ and let the gradient field be $\{ \nabla \phi(\vv{e}) : \vv{e} \in \vv{E} \}$. Then the distribution of the gradient field is given by the following product density:
\[p(y)
    = \prod_{\vv{e}\in \vv{E}}\frac{\Gamma\big(a_{\vv{e}}+\frac{1}{2}\big)}{\sqrt{2\pi}\Gamma(a_{\vv{e}})} \,e^{-\frac{1}{2}y_{\vv{e}}} (\cosh(y_{\vv{e}}))^{-(a_{\vv{e}}+\frac{1}{2})},\quad\forall y\in\mathbb{R}^{\vv{E}}.\]
\end{lem}

\begin{proof}
The hyperbolic Gaussian density is
\begin{equation*}
\begin{aligned}
    p_\beta(\phi)
    &= \frac{1}{(2\pi)^{(n-1)/2}}\,e^{-\sum_{\vv{e}=(i\to j)\in \vv{E}}\beta_{\vv{e}}(\cosh(\phi_i-\phi_j)-1)}e^{-\sum_{i\in V\backslash\{v_0\} }\phi_i} \prod_{\vv{e}=(i\to j)\in \vv{E}}\beta_{\vv{e}}^{\frac{1}{2}} e^{\frac{1}{2}(\phi_i+\phi_j)}\\
    &= \frac{1}{(2\pi)^{(n-1)/2}}\,\prod_{\vv{e}=(i\to j)\in \vv{E}}e^{-\beta_{\vv{e}}(\cosh(\phi_i-\phi_j)-1)}\cdot e^{-\phi_i}\cdot \beta_{\vv{e}}^{\frac{1}{2}} e^{\frac{1}{2}(\phi_i+\phi_j)}\\
    &= \prod_{\vv{e}=(i\to j)\in \vv{E}}\sqrt{\frac{\beta_{\vv{e}}}{2\pi}} \,e^{-\beta_{\vv{e}}(\cosh(\phi_i-\phi_j)-1)} e^{\frac{1}{2}(\phi_j-\phi_i)}.
\end{aligned}
\end{equation*}
In the second step above, we used the fact that $e^{-\sum_{i\in V\backslash\{v_0\} }\phi_i}=\prod_{\vv{e}=(i\to j)\in \vv{E}} e^{-\phi_i}$.


Using the change of variables \( y_{\vv{e}} = \nabla \phi(\vv{e}) = \phi_i - \phi_j \), we obtain
\[
\prod_{\vv{e} \in \vv{E}} dy_{\vv{e}} = \prod_{i \neq v_0} d\phi_i.
\]
This leads to following expression for the probability density in the new variables:
\[
p_{\beta}(y)= \prod_{\vv{e} \in \vv{E}} \sqrt{\frac{\beta_{\vv{e}}}{2\pi}} \, e^{-\beta_{\vv{e}}(\cosh(y_{\vv{e}}) - 1)} e^{-\frac{1}{2} y_{\vv{e}}}.
\]

Integrating out the independent Gamma field $\{\beta_{\vv{e}}: \vv{e}\in \vv{E}\}$, with $b_{\vv{e}}$ denoting the realization of $\beta_{\vv{e}}$, we have
\begin{equation*}
\begin{aligned}
    p(y)
    &= \int_{\mathbb{R}_+^{n-1}}\bigg(\prod_{\vv{e}\in \vv{E}}\sqrt{\frac{b_{\vv{e}}}{2\pi}} \,e^{-b_{\vv{e}}(\cosh(y_{\vv{e}})-1)} e^{-\frac{1}{2}y_{\vv{e}}}\bigg)\prod_{\vv{e}\in \vv{E}}\frac{1}{\Gamma(a_{\vv{e}})}\,b_{\vv{e}}^{a_{\vv{e}}-1}e^{-b_{\vv{e}}} db_{\vv{e}}\\
    &= \prod_{\vv{e}\in \vv{E}}\bigg(\int_{\mathbb{R}_+}\frac{1}{\sqrt{2\pi}\Gamma(a_{\vv{e}})} \,e^{-\frac{1}{2}y_{\vv{e}}} e^{-b_{\vv{e}}\cosh(y_{\vv{e}})} \,b_{\vv{e}}^{a_{\vv{e}}-\frac{1}{2}} db_{\vv{e}}\bigg)\\
    &= \prod_{\vv{e}\in \vv{E}}\frac{\Gamma\big(a_{\vv{e}}+\frac{1}{2}\big)}{\sqrt{2\pi}\Gamma(a_{\vv{e}})} \,e^{-\frac{1}{2}y_{\vv{e}}} (\cosh(y_{\vv{e}}))^{-(a_{\vv{e}}+\frac{1}{2})}.
\end{aligned}
\end{equation*}

In the last step above we used the so-called Gamma integral identity:
\begin{equation}    \label{eqn:gammaidentity}
\int_{\mathbb{R}_+} \frac{\theta_2^{\theta_1}}{\Gamma(\theta_1)}x^{\theta_1 -1} e^{-\theta_2 x} dx = 1, \quad \forall \theta_1,\theta_2>0,
\end{equation}
setting $x = b_{\vv{e}}, \quad \theta_1 = a_{\vv{e}} + \frac{1}{2}, \quad \theta_2 = \cosh(y_{\vv{e}})$.
This concludes the proof.
\end{proof}

Using this characterization, it is not hard to 
get upper bounds on $\sup_{\vv{e} \in \vv{E}} |\nabla\phi(\vv{e})|$ and $\sup_{i\in V}|\phi_i|$.

\begin{lem}[Upper bounds on $\sup_{\vv{e} \in \vv{E}} |\nabla\phi(\vv{e})|$ and $\sup_{i\in V}|\phi_i|$ ]
\label{lem:fluc_grad}
In the same setting as \cref{lem:grad_indep}, in particular that the graph is a tree, assume that $\underline{a}\leq \min_{e\in E} a_e\leq \max_{e\in E} a_e \leq \overline{a}$ for some positive constants $\underline{a}, \overline{a}$. For any $c>0$, we have
\[
\Perrw^{v_0,A}\bigg(\sup_{\vv{e} \in \vv{E}} |\nabla\phi(\vv{e})| \geq \frac{2(c+\overline{a}+3)}{\underline{a}}n\bigg) \leq \Perrw^{v_0,A}\bigg(\sum_{\vv{e} \in \vv{E}} |\nabla\phi(\vv{e})| \geq \frac{2(c+\overline{a}+3)}{\underline{a}}n\bigg) \leq e^{-cn}.
\]
In particular, we also have
$\Perrw^{v_0,A}\bigg(\sup_{i \in V} |\phi_i| \geq \frac{2(c+\overline{a}+3)}{\underline{a}}n\bigg) \leq e^{-cn}$.
\end{lem}

\begin{proof}

Using \cref{lem:grad_indep}, we know that for any $\lambda$ such that $\lambda_{\vv{e}} < a_{\vv{e}}$ for all $\vv{e} \in \vv{E}$,
\begin{equation*}
\begin{aligned}
    \int_0^\infty e^{\sum_{\vv{e} \in \vv{E}} \lambda_{\vv{e}} |y_{\vv{e}}|} p(y) dy
    &= \int_{-\infty}^\infty \prod_{\vv{e} \in \vv{E}} \bigg(\frac{\Gamma\big(a_{\vv{e}}+\frac{1}{2}\big)}{\sqrt{2\pi}\Gamma(a_{\vv{e}})} \, e^{-\frac{1}{2}y_{\vv{e}} + \lambda_{\vv{e}} |y_{\vv{e}}|} (\cosh(y_{\vv{e}}))^{-(a_{\vv{e}}+\frac{1}{2})}\bigg) dy_{\vv{e}} \\
    &= \prod_{\vv{e} \in \vv{E}} \int_{-\infty}^\infty \bigg(\frac{\Gamma\big(a_{\vv{e}}+\frac{1}{2}\big)}{\sqrt{2\pi}\Gamma(a_{\vv{e}})} \, e^{-\frac{1}{2}y_{\vv{e}} + \lambda_{\vv{e}} |y_{\vv{e}}|} (\cosh(y_{\vv{e}}))^{-(a_{\vv{e}}+\frac{1}{2})}\bigg) dy_{\vv{e}} \\
    &= \prod_{\vv{e} \in \vv{E}} \int_{-\infty}^\infty \bigg(\frac{\Gamma\big(a_{\vv{e}}+\frac{1}{2}\big)}{\sqrt{\pi}\Gamma(a_{\vv{e}})} \, \bigg(\frac{e^{-y_{\vv{e}}}}{e^{y_{\vv{e}}} + e^{-y_{\vv{e}}}}\bigg)^{\frac{1}{2}} \bigg(\frac{2}{e^{y_{\vv{e}}} + e^{-y_{\vv{e}}}}\bigg)^{a_{\vv{e}}} e^{\lambda_{\vv{e}} |y_{\vv{e}}|}\bigg) dy_{\vv{e}}.
\end{aligned}
\end{equation*}

Using the inequalities \( \frac{e^{-y_{\vv{e}}}}{e^{y_{\vv{e}}} + e^{-y_{\vv{e}}}} < 1 \) and \( e^{|y_{\vv{e}}|} < e^{y_{\vv{e}}} + e^{-y_{\vv{e}}} \), we obtain
\begin{equation*}
\begin{aligned}
\int_0^\infty e^{\sum_{\vv{e} \in \vv{E}} \lambda_{\vv{e}} |y_{\vv{e}}|} p(y) dy
    &\leq \prod_{\vv{e} \in \vv{E}} \int_{-\infty}^\infty \bigg(\frac{\Gamma\big(a_{\vv{e}}+\frac{1}{2}\big)}{\sqrt{\pi}\Gamma(a_{\vv{e}})} \, \bigg(\frac{2}{e^{|y_{\vv{e}}|}}\bigg)^{a_{\vv{e}}} e^{\lambda_{\vv{e}} |y_{\vv{e}}|}\bigg) dy_{\vv{e}} \\
    &= \prod_{\vv{e} \in \vv{E}} \int_0^\infty \bigg(\frac{\Gamma\big(a_{\vv{e}}+\frac{1}{2}\big)}{\sqrt{\pi}\Gamma(a_{\vv{e}})} \, 2^{a_{\vv{e}}+1} e^{-(a_{\vv{e}}-\lambda_{\vv{e}}) y_{\vv{e}}}\bigg) dy_{\vv{e}} \\
    &= \prod_{\vv{e} \in \vv{E}} \frac{\Gamma\big(a_{\vv{e}}+\frac{1}{2}\big)}{\sqrt{\pi}\Gamma(a_{\vv{e}})} \, 2^{a_{\vv{e}}+1} \frac{1}{a_{\vv{e}} - \lambda_{\vv{e}}}.
\end{aligned}
\end{equation*}

Here we used the symmetry of the integrand (which is an even function) in the second step. Using the inequality \( \Gamma(a_{\vv{e}} + \frac{1}{2}) \leq 2\sqrt{\pi} \Gamma(a_{\vv{e}}+1) \) for all \( a_{\vv{e}} \geq 0 \), \footnote{To see this, note that $\Gamma(a)$ is decreasing for $a \in (0, \mu)$, where $\mu\approx1.46$ with $\Gamma(\mu) \approx 0.8856$, and is increasing for $a \in (\mu,\infty)$. If $a_e \geq \mu-\frac{1}{2}$, then $\Gamma(a_e+\frac{1}{2})\leq \Gamma(a_e+1)$; if $0 < a_e<\mu-\frac{1}{2}$, then $\Gamma(a_e+\frac{1}{2})\leq \Gamma(\frac{1}{2})=\sqrt{\pi}$, $\Gamma(a_e+1)\geq\Gamma(\mu)>\frac{1}{2}$, therefore $\Gamma(a_e+\frac{1}{2})\leq 2\sqrt{\pi}\Gamma(a_e+1)$.} we conclude that
\[
\Eerrw^{v_0,A}\big(e^{\sum_{\vv{e}}\lambda_{\vv{e}} |\nabla\phi(\vv{e})|}\big) \leq \prod_{\vv{e} \in \vv{E}} 2^{a_{\vv{e}}+2} \frac{a_{\vv{e}}}{a_{\vv{e}} - \lambda_{\vv{e}}}.
\]

Thus, for any \( 0 < \theta < \underline{a} \), we conclude that
\[
\Eerrw^{v_0,A}(e^{\theta \sum_{\vv{e} \in \vv{E}} |\nabla\phi(\vv{e})|}) \leq 2^{(\overline{a}+2)|E|} \bigg(\frac{\underline{a}}{\underline{a} - \theta}\bigg)^{|E|}.
\]
Thus we have the following tail bound for any $\alpha>0$ and any \( 0 < \theta < \underline{a} \):
\[\Perrw^{v_0,A}\bigg(\sum_{\vv{e} \in \vv{E}} |\nabla\phi(\vv{e})|\geq \alpha\bigg)\leq \Eerrw^{v_0,A}(e^{\theta \sum_{\vv{e} \in \vv{E}} |\nabla\phi(\vv{e})|})\cdot e^{-\theta \alpha}\leq 2^{(\overline{a}+2)|E|}\bigg(\frac{\underline{a}}{\underline{a}-\theta}\bigg)^{|E|}e^{-\theta \alpha}.\]

Take $\theta =\frac{\underline{a}}{2}$. We get $\Perrw^{v_0,A}(\sum_{\vv{e} \in \vv{E}} |\nabla\phi(\vv{e})|\geq \alpha)\leq 2^{(\overline{a}+3)|E|}e^{-\underline{a}\alpha/2}\leq e^{(\overline{a}+3)n-\underline{a}\alpha/2}$,
where we have noted that for a tree we have $|E|=n-1$. For any $c>0$, let $\alpha=\frac{2(c+\overline{a}+3)}{\underline{a}}n$. Then we have $\Perrw^{v_0,A}(\sum_{\vv{e} \in \vv{E}} |\nabla\phi(\vv{e})|\geq \alpha)\leq e^{-cn}$, as claimed.

Since the graph is assumed to be a tree with the edges directed towards the root $v_0$ and with 
$\phi_{v_0} = 0$, the bound on 
$\sum_{\vv{e} \in \vv{E}} |\nabla\phi(\vv{e})|$
immediately implies the same bound on 
$\sup_{i\in V}|\phi_i|$.
\end{proof}

\vspace{.3em}
\begin{rem}[Tightness of the bound on $\sup_{i\in V}|\phi_i|$.]
\label{rem:tight_fluc}
Consider a path that consists of $(n+1)$ vertices, let $A = \mathbf{1}$,
i.e. all edges have initial local time $1$.
Label the left-most vertex by $v_0$ and the right-most vertex by $v_n$. Denote their corresponding entries in $\phi$ as $\phi_0$ and $\phi_n$, respectively. Then we have $\phi_n=\phi_n-\phi_0=\sum_{\vv{e}\in\vv{E}}y_{\vv{e}}$. 

Note that $\{y_{\vv{e}}: \vv{e}\in \vv{E}\}$ are i.i.d.. For each $\vv{e}\in\vv{E}$, we have
\[\Eerrw^{v_0,A}(y_{\vv{e}}) = \int_{\mathbb{R}}\frac{1}{2\sqrt{2}} \,y_{\vv{e}} e^{-\frac{1}{2}y_{\vv{e}}} (\cosh(y_{\vv{e}}))^{-\frac{3}{2}} dy_{\vv{e}}\approx -0.98.\]

Also, the second moment is bounded as
\[\Eerrw^{v_0,A}(y_{\vv{e}}^2) = \int_{\mathbb{R}}\frac{1}{2\sqrt{2}} \,y_{\vv{e}}^2 e^{-\frac{1}{2}y_{\vv{e}}} (\cosh(y_{\vv{e}}))^{-\frac{3}{2}} dy_{\vv{e}}\approx 3.00.\]

Then it is clear from the central limit theorem perspective that $\phi_n$, as the summation of $n$ i.i.d. random variables, each with mean $-0.98$ and bounded variance, should be of order $-0.98n+O(\sqrt{n})$ with high probability when $n$ is large enough. Therefore $\sup_{i\in V}|\phi_i|$ is indeed $\Theta(n)$ with high probability in this example.
\end{rem}

\subsubsection{Hyperbolic Gaussian measure for a general graph}

When the graph is not a tree, the independent structure of the gradient field that was identified for trees no longer holds, requiring a different approach to establish bounds on $\sup_{i\in V}|\phi_i|$. The following observation plays a crucial role.

\vspace{.3em}
\begin{lem}[Generating function bound]
\label{lem:general_mgf}
Consider a connected graph \( G = (V, E) \) with positive edge weights \( A := (a_e)_{e \in E} \). Fix a vertex \( v_0 \in V \) and generate the \( \beta \)-field and \( \phi \)-field accordingly. Choose an arbitrary orientation of the edges, and let \( \vv{E} \) represent the set of directed edges, inducing the gradient field \( \nabla \phi \). For any \( \lambda \in \mathbb{R}^{\vv{E}} \) such that \( \lambda_{\vv{e}} < \beta_{\vv{e}} \) for all \( \vv{e} \in \vv{E} \), we have
\begin{equation}    \label{eqn:genfnlemma}
\Eerrw^{v_0,A}\Bigg(e^{\sum_{\vv{e}=(i\to j)\in \vv{E}} \lambda_{\vv{e}}(\cosh(\phi_i - \phi_j) - 1)} \, \Bigg| \, \beta \Bigg) \leq \max_{\vv{T} \in \vv{\mathcal{T}}} \prod_{\vv{e} \in \vv{T}} \sqrt{\frac{\beta_{\vv{e}}}{\beta_{\vv{e}} - \lambda_{\vv{e}}}}.
\end{equation}
Here $\vv{T}$ is the notation for a spanning tree with oriented edges, and $\vv{\mathcal{T}}$ is the family of all such oriented spanning trees.
Note that $\vv{\mathcal{T}}$ is just $\mathcal{T}$ where the edges are oriented according to the chosen orientation that is fixed once and for all. 
\end{lem}
The inequality in \cref{eqn:genfnlemma} becomes an equality if the graph is a tree.

\begin{proof}
Note that
\begin{equation*}
\begin{aligned}
\text{LHS of \cref{eqn:genfnlemma}}
&= \int_{\mathbb{R}^{n-1}} (2\pi)^{-\frac{n-1}{2}} e^{\sum_{\vv{e}=(i\to j)\in \vv{E}} \lambda_{\vv{e}} (\cosh(\phi_i - \phi_j) - 1)} e^{-\sum_{\vv{e}=(i\to j)\in \vv{E}} \beta_{\vv{e}} (\cosh(\phi_i - \phi_j) - 1) - \sum_{k \in V \setminus \{v_0\}} \phi_k} \\
& \hspace{19em} \cdot \sqrt{\sum_{\vv{T} \in \vv{\mathcal{T}}} \prod_{\vv{e}=(i\to j)\in \vv{E}} \beta_{\vv{e}} e^{\phi_i + \phi_j}} d\phi_{-v_0} \\
&= \int_{\mathbb{R}^{n-1}} (2\pi)^{-\frac{N-1}{2}} e^{-\sum_{\vv{e}=(i\to j)\in \vv{E}} (\beta_{\vv{e}} - \lambda_{\vv{e}}) (\cosh(\phi_i - \phi_j) - 1) - \sum_{k \in V \setminus \{v_0\}} \phi_k} \\
& \hspace{19em} \cdot \sqrt{\sum_{\vv{T} \in \vv{\mathcal{T}}} \prod_{\vv{e}=(i\to j)\in \vv{E}} \beta_{\vv{e}} e^{\phi_i + \phi_j}} d\phi_{-v_0}.
\end{aligned}
\end{equation*}

Now let \( \beta_{\vv{e}}' = \beta_{\vv{e}} - \lambda_{\vv{e}} \) for \( \vv{e} \in \vv{E} \). Then,
\begin{equation*}
\begin{aligned}
\text{LHS of \cref{eqn:genfnlemma}}
&= \int_{\mathbb{R}^{n-1}} (2\pi)^{-\frac{n-1}{2}} e^{-\sum_{\vv{e}=(i\to j)\in \vv{E}} \beta_{\vv{e}}' (\cosh(\phi_i - \phi_j) - 1) - \sum_{k \in V \setminus \{v_0\}} \phi_k} \\
& \hspace{8em} \cdot \sqrt{\sum_{\vv{T} \in \vv{\mathcal{T}}} \prod_{\vv{e}=(i\to j)\in \vv{E}} \beta_{\vv{e}}' e^{\phi_i + \phi_j}} \cdot \sqrt{\frac{\sum_{\vv{T} \in \vv{\mathcal{T}}} \prod_{\vv{e}=(i\to j)\in \vv{E}} \beta_{\vv{e}} e^{\phi_i + \phi_j}}{\sum_{\vv{T} \in \vv{\mathcal{T}}} \prod_{\vv{e}=(i\to j)\in \vv{E}} \beta_{\vv{e}}' e^{\phi_i + \phi_j}}} d\phi_{-v_0} \\
& \leq \int_{\mathbb{R}^{n-1}} (2\pi)^{-\frac{n-1}{2}} e^{-\sum_{\vv{e}=(i\to j)\in \vv{E}} \beta_{\vv{e}}' (\cosh(\phi_i - \phi_j) - 1) - \sum_{k \in V \setminus \{v_0\}} \phi_k} \\
& \hspace{9em} \cdot \sqrt{\sum_{\vv{T} \in \vv{\mathcal{T}}} \prod_{\vv{e}=(i\to j)\in \vv{E}} \beta_{\vv{e}}' e^{\phi_i + \phi_j}} \cdot \sqrt{\max_{\vv{T} \in \vv{\mathcal{T}}}\frac{ \prod_{\vv{e}=(i\to j)\in \vv{E}} \beta_{\vv{e}} e^{\phi_i + \phi_j}}{ \prod_{\vv{e}=(i\to j)\in \vv{E}} \beta_{\vv{e}}' e^{\phi_i + \phi_j}}} d\phi_{-v_0}\\
& = \max_{\vv{T} \in \vv{\mathcal{T}}} \prod_{\vv{e} \in \vv{T}} \sqrt{\frac{\beta_{\vv{e}}}{\beta_{\vv{e}}'}} \bigg(\int_{\mathbb{R}^{n-1}} (2\pi)^{-\frac{n-1}{2}} e^{-\sum_{\vv{e}=(i\to j)\in \vv{E}} \beta_{\vv{e}}' (\cosh(\phi_i - \phi_j) - 1) - \sum_{k \in V \setminus \{v_0\}} \phi_k} \\
& \hspace{22em} \cdot \sqrt{\sum_{\vv{T} \in \vv{\mathcal{T}}} \prod_{\vv{e}=(i\to j)\in \vv{E}} \beta_{\vv{e}}' e^{\phi_i + \phi_j}} d\phi_{-v_0} \bigg)\\
&= \max_{\vv{T} \in \vv{\mathcal{T}}} \prod_{\vv{e} \in \vv{T}} \sqrt{\frac{\beta_{\vv{e}}}{\beta_{\vv{e}}'}}.
\end{aligned}
\end{equation*}

The inequality follows from the fact that $\frac{\sum_i a_i}{\sum_i b_i} \leq \max_i \frac{a_i}{b_i}$ for positive sequences \( \{a_i\}, \{b_i\} \).
\end{proof}

A natural consequence is that for any $i\sim j$, $|\phi_i-\phi_j|$ has a sub-exponential tail, as we now show.

\vspace{0.3em}
\begin{cor}[Sub-exponential tail of gradients of the $\phi$-field]
\label{cor:grad_subexp}
In the same setting as \cref{lem:general_mgf}, we have \( \forall \vv{e}=(i\to j)\in \vv{E} \), and \( \forall s > 0 \),
\[
\Perrw^{v_0,A}(|\phi_i - \phi_j| \geq s) \leq \sqrt{2} \Big(\frac{1}{2} + \frac{1}{4} e^s \Big)^{-a_{\vv{e}}} \leq 2^{2a_{\vv{e}}+1} e^{-a_{\vv{e}} s}.
\]
\end{cor}

\begin{proof}
By \cref{lem:general_mgf}, for any \( \lambda \in \mathbb{R}^{\vv{E}} \) such that \( \lambda_{\vv{e}} < \beta_{\vv{e}} \) for all \( \vv{e} \in \vv{E} \), we have
\[
\Eerrw^{v_0,A}\Bigg(e^{\sum_{\vv{e} \in \vv{E}} \lambda_{\vv{e}}(\cosh(\phi_i - \phi_j) - 1)} \, \Bigg| \, \beta \Bigg) \leq \max_{\vv{T} \in \vv{\mathcal{T}}} \prod_{\vv{e} \in \vv{T}} \sqrt{\frac{\beta_{\vv{e}}}{\beta_{\vv{e}} - \lambda_{\vv{e}}}}.
\]

Fix any \( \vv{e} = (i \to j) \in \vv{E} \), and let \( \lambda_{\vv{e}'} = 0 \) for any \( \vv{e}' \neq \vv{e} \). Then, 
\[
\Eerrw^{v_0,A}\left(e^{\lambda_{\vv{e}}(\cosh(\phi_i - \phi_j) - 1)} \,\bigg|\, \beta \right) \leq \sqrt{\frac{\beta_{\vv{e}}}{\beta_{\vv{e}} - \lambda_{\vv{e}}}}, \quad \forall \lambda_{\vv{e}} < \beta_{\vv{e}}.
\]

Thus, for any $\alpha>0$,
\[
\Perrw^{v_0,A}(\cosh(\phi_i - \phi_j) - 1 \geq \alpha \mid \beta) \leq \sqrt{\frac{\beta_{\vv{e}}}{\beta_{\vv{e}} - \lambda_{\vv{e}}}} e^{-\lambda_{\vv{e}} \alpha}, \quad \forall 0 \le \lambda_{\vv{e}} < \beta_{\vv{e}}.
\]

Taking \( \lambda_{\vv{e}} = \frac{1}{2} \beta_{\vv{e}} \) gives
\[
\Perrw^{v_0,A}(\cosh(\phi_i - \phi_j) - 1 \geq \alpha \mid \beta) \leq \sqrt{2} e^{-\frac{1}{2} \beta_{\vv{e}} \alpha}.
\] 

Integrating over \( \beta_{\vv{e}} \sim \Gamma(a_{\vv{e}}, 1) \), we have
\begin{equation*}
\begin{aligned}
\Perrw^{v_0,A}(\cosh(\phi_i - \phi_j) - 1 \geq \alpha)
&\leq \sqrt{2} \int_0^\infty \frac{1}{\Gamma(a_{\vv{e}})} e^{-\frac{1}{2} \beta_{\vv{e}} \alpha} \beta_{\vv{e}}^{a_{\vv{e}} - 1} e^{-\beta_{\vv{e}}} d\beta_{\vv{e}} \\
&= \sqrt{2} \frac{1}{\Gamma(a_{\vv{e}})} \int_0^\infty e^{-(\frac{1}{2} \alpha + 1) \beta_{\vv{e}}} \beta_{\vv{e}}^{a_{\vv{e}} - 1} d\beta_{\vv{e}} \\
&= \sqrt{2} \Big( 1 + \frac{1}{2} \alpha \Big)^{-a_{\vv{e}}}.
\end{aligned}
\end{equation*}

In the final step, we used the Gamma integral identity of \cref{eqn:gammaidentity} again,
with $\theta_1=\vv{a}_e, \theta_2=1+\frac{1}{2} \alpha$. Since
\[
\frac{1}{2} e^{|\phi_i - \phi_j|} \leq \frac{1}{2}(e^{\phi_i - \phi_j} + e^{-(\phi_i - \phi_j)}) = \cosh(\phi_i - \phi_j),
\]
it follows that
\[
\Perrw^{v_0,A}\left(\frac{1}{2} e^{|\phi_i - \phi_j|} \geq \alpha + 1 \right) \leq \Perrw^{v_0,A}\left(\cosh(\phi_i - \phi_j) \geq \alpha + 1 \right) \leq \sqrt{2} \Big( 1 + \frac{1}{2} \alpha \Big)^{-a_{\vv{e}}}, \quad \forall \alpha > 0.
\]

By the change-of-variable $s=\log(2\alpha+2)\in(\log2,\infty)$, this is equivalent to
\[
\Perrw^{v_0,A}(|\phi_i - \phi_j| \geq s) \leq \sqrt{2} \Big( \frac{1}{2} + \frac{1}{4} e^s \Big)^{-a_{\vv{e}}}, \quad \forall s > \log 2.
\]

Note that the inequality is also true when $0<s\leq \log 2$ since the right-hand side will be greater than $1$. We can further bound the right-hand side as
\[\sqrt{2} \Big(\frac{1}{2} + \frac{1}{4} e^s \Big)^{-a_{\vv{e}}} \leq \sqrt{2} \Big(\frac{1}{4} e^s\Big)^{-a_{\vv{e}}} = 2^{2a_{\vv{e}}+\frac{1}{2}} e^{-a_{\vv{e}} s} \leq 2^{2a_{\vv{e}}+1} e^{-a_{\vv{e}} s}.\]
\end{proof}

We are now able to estimate the tail probabilities of 
$\sup_{i\in V}|\phi_i|$ on a general graph.

\vspace{0.3em}
\begin{thm}[Tail probabilities of $\sup_{i\in V}|\phi_i|$ on a general graph]
\label{thm:fluc_phi}
In the same setting as \cref{lem:general_mgf}, let \( \diam \) be the diameter of the graph $G$. Assume that $\underline{a}\leq \min_{e\in E} a_e\leq \max_{e\in E} a_e \leq \overline{a}$ for some positive constants $\underline{a}, \overline{a}$. For any \( \delta \in (0,1) \), we have
\[
\Perrw^{v_0,A}\bigg(\sup_{i \in V} |\phi_i| \leq \diam\cdot \Big(\frac{2\overline{a}+1}{\underline{a}}+\frac{1}{\underline{a}}\cdot \log\,\frac{n}{\delta}\Big) \bigg) \geq 1 - \delta.
\]
\end{thm}

\begin{proof}
Consider the graph with unit distance on each edge \( e \in E \). Given any root $v_0\in V$, there exists a shortest-path tree \( T_{v_0} \) such that the distance between any vertex $i\in V$ and $v_0$ in the graph $G$ is identical to their distance in $T_{v_0}$. Hence, the depth of the tree should be no more than $\diam$. Let $\vv{T}_{v_0}$ be the oriented version of the tree such that every edge is oriented towards the root. For any $s>0$, the probability that there exists some $\vv{e} = (i \to j) \in \vv{T}_{v_0}$, such that $|\phi_i - \phi_j| \geq s$ is  bounded as
\begin{equation}    \label{eqn:interbound}
\Perrw^{v_0,A}(\exists \vv{e} = (i \to j) \in \vv{T}_{v_0}, |\phi_i - \phi_j| \geq s) \leq \sum_{\vv{e} \in \vv{T}} 2^{2 a_{\vv{e}} + 1} e^{-a_{\vv{e}} s} \leq n 2^{2 \overline{a} + 1} e^{- \underline{a} s}\leq n e^{2 \overline{a} + 1 - \underline{a} s}.
\end{equation}
In the first step above, we used \cref{cor:grad_subexp} and a union bound over the edges of the tree. Now, for any \( \delta \in (0,1) \), to ensure that $n e^{2 \overline{a} + 1 - \underline{a} s}$ is bounded by \( \delta \), we require
\[
s = \frac{2\overline{a}+1}{\underline{a}}+\frac{1}{\underline{a}}\cdot\log\Big(\frac{n}{\delta}\Big).
\]

Since with probability $\geq 1-\delta$, we have \( \forall \vv{e} = (i \to j) \in \vv{T}_{v_0}\) that \(|\phi_i - \phi_j| \leq \frac{2\overline{a}+1}{\underline{a}}+\frac{1}{\underline{a}}\cdot\log \frac{n}{\delta} \), it follows that
\[
\sup_{i \in V} |\phi_i| \leq \diam\cdot \Big(\frac{2\overline{a}+1}{\underline{a}}+\frac{1}{\underline{a}}\cdot \log\,\frac{n}{\delta}\Big),
\]
with probability at least \( 1 - \delta \).
\end{proof}

\subsubsection{An improvement based on ($\psi$, $\tilde \beta$)-field}

Given the beta field and the hyperbolic field
\[
(\beta,\phi)\in (0,\infty)^E\times \mathbb{R}^{V},
\]
together with an extra independent $\gamma\sim\Gamma(\tfrac12,1)$ placed at the root $v_0$, define the vertex field $\psi=(\psi_i)_{i\in V}$ by
\begin{equation}\label{eqn:psi}
\psi_{i}:=\frac12\sum_{j:j\sim i}\beta_{ij}e^{\phi_j-\phi_i}\quad (i\neq v_0),
\qquad
\psi_{v_0}:=\frac12\sum_{j:j\sim v_0}\beta_{v_0j}e^{\phi_j-\phi_{v_0}}+\gamma.
\end{equation}

We write $B$ for the $n\times n$ symmetric matrix with $(i,j)$ entry
$B_{ij}=\beta_{ij}$ if $\{i,j\}\in E$ and $B_{ij}=0$ otherwise, and $B_{ii}=0$. Sabot-Tarres-Zeng studied the conditioned $\psi$-field and proved that(\cite[Lemma 1]{stz})
\begin{equation}
    \label{eqn:p_beta_psi}
    p_\beta(\psi):=p(\psi\mid\beta)=\Big(\frac{2}{\pi}\Big)^{n/2}\frac{e^{-\sum_i \psi_i+\sum_e\beta_e}}{\sqrt{|2\psi-B|}}\cdot \ind(2\psi-B\succ 0),
\end{equation}
with conditional moment-generating function given by
\begin{equation}
    \label{eqn:p_beta_psi_mgf}
    \E\big(e^{\sum_i\lambda_i\psi_i}\mid \beta\big)
    = \exp\bigg(\sum_{e=\{i,j\}\in E}\beta_e(1-\sqrt{1-\lambda_i}\sqrt{1-\lambda_j})\bigg)\prod_{i\in V} \frac{1}{\sqrt{1-\lambda_i}},
\end{equation}
for any $\{\lambda_i\mid i\in V\}$ such that $\lambda_i<1$ for all $i\in V$.

We also define the normalized edge weights $\tilde\beta=(\tilde\beta_e)_{e\in E}$ by
\begin{equation}\label{eqn:tildebeta}
\tilde\beta_{ij}:=\frac{\beta_{ij}}{\sqrt{\psi_i\psi_j}}
\qquad (\{i,j\}\in E).
\end{equation}
Equivalently, $\beta_{ij}=\tilde\beta_{ij}\sqrt{\psi_i\psi_j}$ for all edges.

When $\beta_e\sim \gam(a_e,1)$ are distributed as independent Gamma random variables, the induced measure on $\psi$, after averaging on $\beta$, becomes independent Gamma random variables.

\vspace{.5em}
\begin{lem}[Unconditioned $\psi$-field]
    The unconditioned $\psi$-field is distributed as independent Gamma random variables where $\psi_i\sim\gam(\frac{1}{2}(a_i+1), 1)$.
\end{lem}
\begin{proof}
    We verify it by calculating the mgf. Note that for any $\{\lambda_i\mid i\in V\}$ such that $\lambda_i<1$ for all $i\in V$, 
    \begin{align*}
        \E\big(e^{\sum_i\lambda_i\psi_i}\big)
        &= \int_{\mathbb{R}_{++}^E}\E\big(e^{\sum_i\lambda_i\psi_i}\mid \beta\big)p(\beta) \,d\beta\\
        &= \int_{\mathbb{R}_{++}^E}e^{\sum_{e=\{i,j\}\in E}\beta_e(1-\sqrt{1-\lambda_i}\sqrt{1-\lambda_j})}\prod_{i\in V} \frac{1}{\sqrt{1-\lambda_i}}\prod_{e\in E}\frac{1}{\Gamma(a_e)}\beta_e^{a_e-1}e^{-\beta_e}\,d\beta\\
        &= \prod_{i\in V} \frac{1}{\sqrt{1-\lambda_i}}\prod_{e\in E}\bigg(\int_{\mathbb{R}_{++}}e^{\beta_e(1-\sqrt{1-\lambda_i}\sqrt{1-\lambda_j})}\frac{1}{\Gamma(a_e)}\beta_e^{a_e-1}e^{-\beta_e}\,d\beta_e\bigg)\\
        &= \prod_{i\in V} \frac{1}{\sqrt{1-\lambda_i}}\prod_{e\in E}\bigg(\int_{\mathbb{R}_{++}}\frac{1}{\Gamma(a_e)}\beta_e^{a_e-1}e^{-\beta_e\sqrt{1-\lambda_i}\sqrt{1-\lambda_j}}\,d\beta_e\bigg)\\
        &= \prod_{i\in V} \frac{1}{\sqrt{1-\lambda_i}}\prod_{e\in E}\bigg(\frac{1}{\sqrt{1-\lambda_i}\sqrt{1-\lambda_j}}\bigg)^{a_e}\\
        &= \prod_{i\in V} (1-\lambda_i)^{-\frac{1}{2}(a_i+1)}.
    \end{align*}
    By uniqueness of the Laplace transform, we know that $\psi$ should be distributed as independent $\gam(\frac{1}{2}(a_i+1), 1)$ random variables. 
\end{proof}

One can then compute the conditional density of $\beta$ given $\psi$.

\vspace{.5em}
\begin{lem}[Conditioned $\beta$ on $\psi$]
    \label{lem:p_psi_beta}
    The conditional density of $\beta$ on $\psi$ is:
    \begin{equation}
        \label{eqn:p_psi_beta}
        p_\psi(\beta) := p(\beta\mid \psi)
        = \frac{1}{Z_A}\cdot \frac{\prod_{e\in E}\beta_e^{a_e-1}}{\prod_{i\in V}\psi_i^{\frac{1}{2}(a_i-1)}}\cdot \frac{1}{\sqrt{|2\psi-B|}}\cdot \ind(2\psi-B\succ 0),
    \end{equation}
    where the normalizing constant is $Z_A:=\big(\frac{\pi}{2}\big)^{n/2}\frac{\prod_{e\in E}\Gamma(a_e)}{\prod_{i\in V}\Gamma(\frac{1}{2}(a_i+1))}$ and $\beta\in\mathbb{R}_{++}^E$.
\end{lem}
\begin{proof}
    Note that by Bayes rule,
    \[p_\psi(\beta)
        =\frac{p(\beta,\psi)}{p(\psi)}
        =\frac{p(\beta)p_\beta(\psi)}{p(\psi)}.\]
    Since $\beta$ and $\psi$ are marginally independent Gamma fields, we know that $p(\beta)=\prod_{e\in E}\frac{1}{\Gamma(a_e)}\beta_e^{a_e-1}e^{-\beta_e}$, and $p(\psi)=\prod_{i\in V}\frac{1}{\Gamma(\frac{1}{2}(a_i+1))}\psi_i^{\frac{1}{2}(a_i-1)}e^{-\psi_i}$. We then continue to compute that
    \begin{align*}
        p_\psi(\beta)
        &=\frac{\prod_{e\in E}\frac{1}{\Gamma(a_e)}\beta_e^{a_e-1}e^{-\beta_e}}{\prod_{i\in V}\frac{1}{\Gamma(\frac{1}{2}(a_i+1))}\psi_i^{\frac{1}{2}(a_i-1)}e^{-\psi_i}}\cdot \Big(\frac{2}{\pi}\Big)^{n/2}\frac{e^{-\sum_i \psi_i+\sum_e\beta_e}}{\sqrt{|2\psi-B|}}\cdot \ind(2\psi-B\succ 0)\\
        &=\Big(\frac{2}{\pi}\Big)^{n/2}\frac{\prod_{i\in V}\Gamma(\frac{1}{2}(a_i+1))}{\prod_{e\in E}\Gamma(a_e)}\frac{\prod_{e\in E}\beta_e^{a_e-1}}{\prod_{i\in V}\psi_i^{\frac{1}{2}(a_i-1)}}\cdot \frac{1}{\sqrt{|2\psi-B|}}\cdot \ind(2\psi-B\succ 0).
    \end{align*}
\end{proof}

A key structural input is that,
the induced pair $(\psi,\tilde\beta)$ factorizes:
\begin{equation}\label{eq:factorization}
(\psi_i)_{i\in V}\,\perp\!\!\!\perp\,(\tilde\beta_e)_{e\in E}.
\end{equation}

\vspace{.5em}
\begin{lem}[Independence of $\tilde \beta$ and $\psi$]
    \label{lem:p_psi_beta_norm}
    Suppose $\beta\sim p_\psi(\beta)$ for some $p_\psi\in \mathcal{F}^\beta$. Then the ``normalized'' version of $\beta$, defined as $\tilde{\beta}_e:=\frac{\beta_e}{\sqrt{\psi_i\psi_j}}$ for all $e=\{i,j\}\in E$, is distributed as $\forall \,\psi\in\mathbb{R}_{++}^V$,
    \begin{equation}
        \label{eqn:p_psi_beta_norm}
        p_{\psi}(\tilde{\beta})=p(\tilde{\beta}\mid \psi)
        = \frac{1}{Z_A}\cdot \frac{\prod_{e\in E}\tilde{\beta}_e^{a_e-1}}{\sqrt{|2I-\tilde{B}|}}\cdot \ind(2I-\tilde{B}\succ 0),
    \end{equation}
    with $Z_A$ defined in \cref{lem:p_psi_beta} and $\tilde{\beta}\in\mathbb{R}_{++}^E$. In other words, $\tilde{\beta}\sim p_\allone(\beta)$. Note that $\tilde{\beta}$ is independent of $\psi$.
\end{lem}
\begin{proof}
    Under the transformation, we know that $\beta_e=\tilde{\beta}_e\sqrt{\psi_i\psi_j}$ for all $e\in E$ and $d\beta=(\prod_{e\in E}\sqrt{\psi_i\psi_j})d\tilde{\beta}$. Moreover, whenever $2\psi-B\succ 0$, we also have $2I-\psi^{-1/2}B\psi^{-1/2}=2I-\tilde{B}\succ 0$ and $|2\psi_i-B|=(\prod_i\psi_i)|2I-\tilde{B}|$. Therefore, we conclude that
    \begin{align*}
        p(\tilde{\beta}\mid \psi) 
        &= \frac{1}{Z_A}\frac{\prod_{e\in E}(\tilde{\beta}_e\sqrt{\psi_i\psi_j})^{a_e-1}}{\prod_{i\in V}\psi_i^{\frac{1}{2}(a_i-1)}}\cdot \frac{1}{\sqrt{(\prod_i\psi_i)|2I-\tilde{B}|}}\cdot \ind(2I-\tilde{B}\succ 0)\cdot \big(\prod_{e\in E}\sqrt{\psi_i\psi_j}\big)\\
        &= \frac{1}{Z_A}\frac{\prod_{e\in E}\tilde{\beta}_e^{a_e-1}\,\prod_{e\in E}(\sqrt{\psi_i\psi_j})^{a_e-1}}{\prod_{i\in V}\psi_i^{\frac{1}{2}(a_i-1)}}\cdot \frac{\prod_{e\in E}\sqrt{\psi_i\psi_j}}{\prod_i\psi_i^{\frac{1}{2}}}\frac{1}{\sqrt{|2I-\tilde{B}|}}\cdot \ind(2I-\tilde{B}\succ 0)\\
        &= \frac{1}{Z_A}\frac{\prod_{e\in E}(\sqrt{\psi_i\psi_j})^{a_e}}{\prod_{i\in V}\psi_i^{\frac{1}{2}a_i}}\cdot \frac{\prod_{e\in E}\tilde{\beta}_e^{a_e-1}}{\sqrt{|2I-\tilde{B}|}}\cdot \ind(2I-\tilde{B}\succ 0)\\
        &= \frac{1}{Z_A}\frac{\prod_{e\in E}\tilde{\beta}_e^{a_e-1}}{\sqrt{|2I-\tilde{B}|}}\cdot \ind(2I-\tilde{B}\succ 0).
    \end{align*}
    Note that we used the fact that $\frac{\prod_{e\in E}(\sqrt{\psi_i\psi_j})^{a_e}}{\prod_{i\in V}\psi_i^{\frac{1}{2}a_i}}=1$ in the last step.
\end{proof}

This induces the following interesting way of generating the correlated $(\beta, \psi)$ field:
\begin{itemize}
    \item Generate $\psi$ as independent Gamma random variables, i.e., $\psi_i\sim\Gamma(\frac{1}{2}(a_i+1),1)$;
    \item Generate $\tilde{\beta}$ as independent random variable that has the law \cref{eqn:p_psi_beta_norm};
    \item Let $\beta_e=\tilde{\beta}_e\sqrt{\psi_i\psi_j}$ for all edge $e\in E$.
\end{itemize}

One can compute the moments of $\{\beta_e\mid e\in E\}$ as follows. Given real values $\{t_e\mid e\in E\}$ such that $t_e>-a_e, \forall e\in E$, we compute
\begin{align*}
    \E\Big(\prod_{e\in E}\tilde{\beta}_e^{t_e} \mid \psi\Big)
    &= \int_{\mathbb{R}_{++}^E} \frac{1}{Z_A}\prod_{e\in E}\tilde{\beta}_e^{t_e+a_e-1}\cdot \frac{1}{\sqrt{|2I-\tilde{B}|}}\cdot \ind(2I-\tilde{B}\succ 0) \,d\tilde{\beta}\\
    &= \frac{Z_{T+A}}{Z_A}\int_{\mathbb{R}_{++}^E} \frac{1}{Z_{T+A}}\prod_{e\in E}\tilde{\beta}_e^{t_e+a_e-1}\cdot \frac{1}{\sqrt{|2I-\tilde{B}|}}\cdot \ind(2I-\tilde{B}\succ 0) \,d\tilde{\beta}.
\end{align*}

This is exactly $\frac{Z_{T+A}}{Z_A}$ since the last integral evaluates to 1. Using these nice properties of the $(\psi, \tilde\beta)$-field, one can eventually improve the previous bounds on $\sup_i |\phi_i|$ as follows.

\vspace{.5em}
\begin{lem}[Pathwise telescoping bound]\label{lem:path-telescope}
Fix a connected simple graph $G=(V,E)$, a root $v_0\in V$, and let $(\beta,\phi,\psi)$ be the fields defined as in \eqref{eqn:psi} (with the independent $\gamma\sim\Gamma(\tfrac12,1)$ at $v_0$). For each $v\in V$, fix a shortest path
\[
P(v):\quad v_0=v_0,\,v_1,\,\ldots,\,v_k=v,
\]
and write $|P(v)|:=k$. Then
\begin{equation}\label{eq:path-tele}
e^{-\phi_v}\le 2^{\,|P(v)|}\,\sqrt{\frac{\psi_v}{\psi_{v_0}}}\,\prod_{e\in P(v)}\tilde\beta_e^{-1}.
\end{equation}
\end{lem}

\begin{proof}
By the definition of $\psi$ in \eqref{eqn:psi}, for any vertex $i'\neq v_0$ we have
\[
\psi_{i'}=\frac12\sum_{j:j\sim i'}\beta_{i'j}e^{\phi_j-\phi_{i'}}\ge \frac12\,\beta_{i\,i'}\,e^{\phi_i-\phi_{i'}},
\]
where $i$ denotes the neighbor of $i'$ along the fixed path $P(v)$ in the direction towards $v_0$. Rearranging gives
\begin{equation}\label{eq:one-step}
e^{-\phi_{i'}}\le \frac{2\psi_{i'}}{\beta_{i\,i'}}\,e^{-\phi_i}.
\end{equation}

Iterating \eqref{eq:one-step} along the path $v_k\to v_{k-1}\to\cdots\to v_0$ yields
\[
e^{-\phi_v}\le \prod_{r=1}^{k}\frac{2\psi_{v_r}}{\beta_{v_{r-1}v_r}}
= \prod_{r=1}^{k}\frac{2}{\tilde\beta_{v_{r-1}v_r}}\,
\sqrt{\frac{\psi_{v_r}}{\psi_{v_{r-1}}}}
= 2^{\,k}\,\sqrt{\frac{\psi_{v_k}}{\psi_{v_0}}}\,\prod_{r=1}^{k}\tilde\beta_{v_{r-1}v_r}^{-1},
\]
where in the middle equality we used $\beta_{ij}=\tilde\beta_{ij}\sqrt{\psi_i\psi_j}$, and in the last equality the $\sqrt{\psi}$-factors telescope. Since $k=|P(v)|$ and $v_k=v$, this is exactly \eqref{eq:path-tele}.
\end{proof}

\begin{lem}[Endpoint Gamma–ratio control]\label{lem:endpoint-ratio}
Assume the setup preceding Lemma~\ref{lem:path-telescope}. In particular, the unconditioned $\psi$-field is independent with
\[
\psi_i\sim \Gamma\Big(s_i,1\Big),\qquad s_i:=\frac{a_i+1}{2},\quad a_i=\sum_{e\ni i} a_e,
\]
and $0<\underline{a}\le a_e\le \overline{a}<\infty$ for all $e\in E$. Then for any $\alpha\geq 0$,
\begin{equation}\label{eq:endpoint-ratio}
\Pr\left(\max_{v\in V}\sqrt{\frac{\psi_v}{\psi_{v_0}}}> n^{\alpha}\sqrt{2(1+\alpha)\log n+(\overline{a}\,d_{\max}+1)\log 2} \right)\le 3n^{-\alpha}.
\end{equation}
\end{lem}

\begin{proof}
Fix $\theta=\frac12\in(0,1)$. For each $v$,
\[
\Pr(\psi_v\ge x)\le (1-\theta)^{-s_v}e^{-\theta x} \le 2^{s_{\max}} e^{-x/2},
\quad s_{\max}:=\max_{i\in V}s_i.
\]
By a union bound over $v\in V$,
\begin{equation}\label{eq:max-psi-upper}
\Pr\left(\max_{v\in V}\psi_v\ge x\right)\le n\,2^{s_{\max}}e^{-x/2}\le n\,2^{\frac{\overline{a}\,d_{\max}+1}{2}}e^{-x/2}.
\end{equation}
Choose
\[
x:=2\log n+(\overline{a}\,d_{\max}+1)\log 2+2\alpha \log n,
\]
to get for any $\alpha\geq 0$,
\begin{equation}\label{eq:max-psi-chosen}
\Pr\left(\max_{v}\psi_v\,\ge 2\log n+(\overline{a}\,d_{\max}+1)\log 2+2\alpha \log n\right) \le\,e^{-\alpha \log n} = n^{-\alpha}.
\end{equation}

By the definition \eqref{eqn:psi}, at the root $v_0$ we have the extra independent $\gamma\sim\Gamma(\tfrac12,1)$ term and hence $\psi_{v_0} \ge \gamma$. Therefore, for all $y>0$,
\[
\Pr(\psi_{v_0}\le y)\le \Pr(\gamma\le y).
\]
Since $\gamma\sim\Gamma(\tfrac12,1)$ has cdf $\Pr(\gamma\le y)=\operatorname{erf}(\sqrt{y})$, we use the standard bound $\operatorname{erf}(t) \le \frac{2}{\sqrt{\pi}}\,t$ to obtain, for all $\alpha\ge 0$,
\begin{equation}\label{eq:gamma-lb}
\Pr\Big(\psi_{v_0}\le \frac{1}{n^{2\alpha}}\Big)\le \Pr\Big(\gamma\le \frac{1}{n^{2\alpha}}\Big)\le \frac{2}{\sqrt{\pi}}\,n^{-\alpha}.
\end{equation}

On the intersection of the events in \eqref{eq:max-psi-chosen} and \eqref{eq:gamma-lb} (which has probability at least $1-3n^{-\alpha}$), we have
\[
\max_{v}\sqrt{\frac{\psi_v}{\psi_{v_0}}}
\,\le\
\sqrt{\frac{2(1+\alpha)\log n+(\overline{a}\,d_{\max}+1)\log 2}{n^{-2\alpha}}}.
\]
This proves the claim.
\end{proof}

\begin{lem}[Chernoff upper bound for $Y_P$]\label{lem:chebyshev-upper}
Let $G=(V,E)$ be connected, fix a root $v_0$, and assume $0<\underline{a}\le a_e\le \overline{a}<\infty$ for all $e\in E$.
For a fixed (shortest) path $P$ from $v_0$ to some $v\in V$, note that $|P|\le \diam$ and define
\[
Y_P:= \sum_{e\in P}\big(-\log \tilde\beta_e\big),
\qquad 
\tilde\beta_{ij}:=\frac{\beta_{ij}}{\sqrt{\psi_i\psi_j}}\,\,(\{i,j\}\in E).
\]
Then:
\begin{enumerate}
\item[\emph{(a)}] The mean is
\begin{equation}\label{eq:EY-simple}
\E[Y_P]= -\sum_{e\in P}\Psi(a_e)+ 
\frac12
\sum_{i\in V}\deg_P(i)\,\Psi\Big(\frac{a_i+1}{2}\Big),
\end{equation}
where $\Psi$ is the digamma function and $\deg_P(i)\in\{0,1,2\}$ counts edges of $P$ incident to $i$.
\item[\emph{(b)}] The variance admits the upper bound
\begin{equation}\label{eq:VarY-simple-upper}
\var(Y_P)\le \sum_{e\in P}\Psi_1(a_e)\le |P|\,\Psi_1(\underline{a})\le \diam\cdot\Psi_1(\underline{a}),
\end{equation}
where $\Psi_1$ is the trigamma function.
\item[\emph{(c)}] Let $v_*:=\Psi_1(\underline{a}/2)$. For all $t\ge 0$,
\begin{equation}\label{eq:Chernoff-tail}
\Pr\big(Y_P>\E[Y_P]+t\big)\le \exp\!\left(-\min\left\{\frac{t^2}{2\,v_*\,|P|},\,\frac{\underline{a}}{4}\,t\right\}\right).
\end{equation}
In particular, for any $\alpha>0$ (with $x:=(1+\alpha)\log n$),
\begin{equation}\label{eq:Chernoff-nalpha}
\Pr\Big(Y_P>\E[Y_P]+\sqrt{2\,v_*\,|P|\,x}+\frac{4}{\underline{a}}\,x\Big)\le e^{-x}=n^{-(1+\alpha)}.
\end{equation}
\end{enumerate}
\end{lem}

\begin{proof}
For $\theta$ in a neighborhood of $0$ (small enough so that all arguments of $\Gamma(\cdot)$ below remain positive),
\[
\E\big[e^{\theta Y_P}\big]
=\E\Big[\prod_{e\in P}\tilde\beta_e^{-\theta}\Big]
=\frac{Z_{A-\theta \mathbf 1_P}}{Z_A},
\]
where $\mathbf 1_P(e)=\mathbf 1\{e\in P\}$ and
\[
Z_{A-\theta \mathbf 1_P}
=\Big(\tfrac{\pi}{2}\Big)^{n/2}\,
\frac{\prod_{e\in E}\Gamma(a_e-\theta\,\mathbf 1_P(e))}
{\prod_{i\in V}\Gamma\big(\tfrac{a_i-\theta\,\deg_P(i)+1}{2}\big)},
\qquad a_i:=\sum_{e\ni i}a_e.
\]
Define the cumulant generating function (cgf)
\[
\kappa_P(\theta):=\log \E[e^{\theta Y_P}]
=\sum_{e\in P}\big(\log\Gamma(a_e-\theta)-\log\Gamma(a_e)\big)
-\sum_{i\in V}\left[\log\Gamma\Big(\tfrac{a_i-\theta\,\deg_P(i)+1}{2}\Big)-\log\Gamma\Big(\tfrac{a_i+1}{2}\Big)\right].
\]

Differentiate using the digamma function $\Psi =\Gamma'/\Gamma$:
\[
\kappa_P'(\theta)
=-\sum_{e\in P}\Psi(a_e-\theta)
+\frac12 \sum_{i\in V}\deg_P(i)\,\Psi\Big(\tfrac{a_i-\theta\,\deg_P(i)+1}{2}\Big),
\]
hence $\kappa_P'(0)=\E[Y_P]$, which yields \eqref{eq:EY-simple}.

Differentiate again using $\Psi_1=\Psi'$:
\[
\kappa_P''(\theta)
=\sum_{e\in P}\Psi_1(a_e-\theta)
-\frac14\sum_{i\in V}\deg_P(i)^2\,\Psi_1\Big(\tfrac{a_i-2\theta\,\deg_P(i)+1}{2}\Big).
\]
Evaluating at $\theta=0$ gives
\[
\kappa_P''(0)=\var(Y_P)
=\sum_{e\in P}\Psi_1(a_e)
-\frac14\sum_{i\in V}\deg_P(i)^2\,\Psi_1\Big(\tfrac{a_i+1}{2}\Big)
\le \sum_{e\in P}\Psi_1(a_e),
\]
which implies \eqref{eq:VarY-simple-upper}.

For the tail, fix $\theta\in[0,\underline{a}/2]$. Since $\Psi_1$ is decreasing and $a_e\ge \underline{a}$,
\[
\kappa_P''(s)\le \sum_{e\in P}\Psi_1(a_e-s)\le |P|\,\Psi_1(\underline{a}-s)\le |P|\,\Psi_1(\underline{a}/2)=:v_*\,|P|
\qquad\forall s\in[0,\theta].
\]
By Taylor’s theorem with integral remainder,
\[
\kappa_P(\theta)-\theta\kappa_P'(0)
=\int_0^\theta (\theta-s)\kappa_P''(s)\,ds
\le \frac{\theta^2}{2}\,v_*\,|P|.
\]
Therefore, by Chernoff/Markov,
\[
\Pr(Y_P-\E[Y_P]>t)
\le \exp\!\Big(-\theta t+\kappa_P(\theta)-\theta\E[Y_P]\Big)
\le \exp\!\Big(-\theta t+\tfrac{\theta^2}{2}\,v_*\,|P|\Big),
\qquad \forall \theta\in[0,\underline{a}/2].
\]
Optimizing over $\theta$ yields \eqref{eq:Chernoff-tail}. For \eqref{eq:Chernoff-nalpha},
set $x:=(1+\alpha)\log n$ and $t=\sqrt{2v_*|P|x}+\frac{4}{\underline{a}}x$, then
\[
\min\left\{\frac{t^2}{2v_*|P|},\frac{\underline{a}}{4}t\right\}\ge x,
\]
so $\Pr(Y_P>\E[Y_P]+t)\le e^{-x}=n^{-(1+\alpha)}$.
\end{proof}

\begin{thm}[Uniform upper bound for $e^{-\phi}$]\label{thm:uniform-upper-clean}
Assume the setup preceding Lemma~\ref{lem:path-telescope} and that $0<\underline{a}\le a_e\le \overline{a}<\infty$ for all $e\in E$. Let $d_{\max}:=\max_{v\in V}\deg(v)$. Then for any $\alpha>0$, with probability at least $1-4n^{-\alpha}$,
\begin{equation}
\begin{aligned}
\max_{v\in V} e^{-\phi_v}
\le 
2^{\diam}\,n^{\alpha} & \sqrt{2(1+\alpha)\log n+2(\log 2)\,\overline{a}\,d_{\max}}\, \\
& \exp\Big(
\diam\cdot\Lambda+ \sqrt{2(1+\alpha)\,\Psi_1(\underline{a}/2)\,\diam\cdot\log n}+ \frac{4}{\underline{a}}(1+\alpha)\log n
\Big), \label{eq:uniform-upper-clean}
\end{aligned}
\end{equation}
where
\[
\Lambda:= -\Psi(\underline{a})+\log\Big(\frac{\overline{a}\,d_{\max}+1}{2}\Big).
\]

Note that the upper bound can be simplified as  $\max_{v\in V} e^{-\phi_v} \le \poly(n) \, d_{\max}^{\,O(\diam)}$, assuming $\alpha, a, b$ to be constant.
\end{thm}

\begin{proof}
Fix, for each $v\in V$, a shortest path $P(v)$ from $v_0$ to $v$ so that $|P(v)|\le \diam$. By Lemma~\ref{lem:path-telescope},
\[
e^{-\phi_v}\le 2^{|P(v)|}\,\sqrt{\frac{\psi_v}{\psi_{v_0}}}\,\exp\big(Y_{P(v)}\big),
\qquad
Y_{P(v)}:=\sum_{e\in P(v)}(-\log\tilde\beta_e).
\]
Taking the maximum over $v$ and using $|P(v)|\le \diam$ gives
\[
\max_{v} e^{-\phi_v}\le 2^{\diam}\,\Big(\max_{u}\sqrt{\tfrac{\psi_u}{\psi_{v_0}}}\Big)\,\exp\Big(\max_{v} Y_{P(v)}\Big).
\]
For the endpoint factor, Lemma~\ref{lem:endpoint-ratio} yields
\[
\Pr\left(\max_{u\in V}\sqrt{\frac{\psi_u}{\psi_{v_0}}}> n^{\alpha}\sqrt{2(1+\alpha)\log n+2(\log 2)\,\overline{a}\,d_{\max}}\right)\le 3n^{-\alpha}.
\]
Let $x:=(1+\alpha)\log n$. By Lemma~\ref{lem:chebyshev-upper}\emph{(c)}
and a union bound over the at most $n$ root--shortest paths,
\[
\Pr\Big(\max_{v\in V}Y_{P(v)}
>\max_{v\in V}\E[Y_{P(v)}]+\sqrt{2\Psi_1(\underline{a}/2)\,\diam\cdot x}+\frac{4}{\underline{a}}x\Big)
\le n\,e^{-x}
= n^{-\alpha}.
\]

It remains to bound $\max_{v}\E[Y_{P(v)}]$ deterministically in terms of $D$ and $(\underline{a},\overline{a},d_{\max})$. From Lemma~\ref{lem:chebyshev-upper},
\[
\E[Y_{P}]
= -\sum_{e\in P}\Psi(a_e)+ \frac12\sum_{i\in V}\deg_P(i)\,\Psi\Big(\frac{a_i+1}{2}\Big),
\qquad a_i=\sum_{e\ni i}a_e,\;\deg_P(i)\in\{0,1,2\}.
\]
Using that $\Psi$ is increasing, $a_e\ge \underline{a}$, $a_i\le \overline{a}\,\deg(i)\le \overline{a}\,d_{\max}$, and $\sum_i\deg_P(i)=2|P|$,
\[
-\sum_{e\in P}\Psi(a_e)\le -|P|\Psi(\underline{a}),
\quad
\frac12\sum_{i}\deg_P(i)\,\Psi\Big(\frac{a_i+1}{2}\Big)
\le |P|\Psi\Big(\frac{\overline{a}\,d_{\max}+1}{2}\Big)
\le |P| \log\Big(\frac{\overline{a}\,d_{\max}+1}{2}\Big).
\]
Hence, for any shortest path $P$,
\[
\E[Y_P]\le |P|\cdot \Big(-\Psi(\underline{a})+\log\Big(\frac{\overline{a}\,d_{\max}+1}{2}\Big)\Big)
\le \diam\cdot\Lambda.
\]

Combining the three displays and intersecting the two high-probability events (probability at least $1-4n^{-\alpha}$) yields \eqref{eq:uniform-upper-clean}.
\end{proof}

\begin{remark}
This improves the $n^{O(\diam)}$ bound whenever $d_{\max} = o(n)$. Getting ideal $2^{O(n)}$ upper bound for the cover time of a $n$-path with probability $1-o(1)$.
\end{remark}

\vspace{.3em}

To upper bound $e^{\phi_i}$, we use the following re-rooting identity. In this lemma,
$p_\beta^{(r)}$ denotes the hyperbolic Gaussian density with the same edge field $\beta$,
but pinned at the vertex $r$, i.e. with $\phi_r=0$.

\begin{lem}[Exponential moment]
\label{lem:unit}
Fix $G=(V,E)$, edge parameters $A=(a_e)_{e\in E}$, and a fixed edge field $\beta$.
For any $v_0,j\in V$ and any real $t$ for which the integrals below are finite, one has
\[
\int_{\mathbb{R}^{V\setminus\{v_0\}}} e^{\,t\phi_j}\,p_\beta^{(v_0)}(\phi)\,d\phi_{V\setminus\{v_0\}}
\;=\;
\int_{\mathbb{R}^{V\setminus\{j\}}} e^{\,(1-t)\phi_{v_0}}\,p_\beta^{(j)}(\phi)\,d\phi_{V\setminus\{j\}}.
\]
Equivalently,
\[
\E_{p_\beta^{(v_0)}}\big[e^{\,t\phi_j}\big]
=
\E_{p_\beta^{(j)}}\big[e^{\,(1-t)\phi_{v_0}}\big].
\]
In particular, at $t=1$,
\[
\E_{p_\beta^{(v_0)}}\big[e^{\phi_j}\big]=1.
\]
\end{lem}

\begin{proof}
Write
\[
p_\beta^{(v_0)}(\phi)
=
\mathcal Z_\beta\,e^{-\Phi(\phi)}\,e^{-L_{v_0}(\phi)}\,\sqrt{\Xi(\phi)},
\]
where
\[
\Phi(\phi):=\sum_{\{i,k\}}\beta_{ik}\big(\cosh(\phi_i-\phi_k)-1\big),
\qquad
L_{v_0}(\phi):=\sum_{i\ne v_0}\phi_i,
\]
and
\[
\mathcal Z_\beta:=(2\pi)^{-(n-1)/2},
\qquad
\Xi(\phi):=\sum_{T\in\mathcal T}\prod_{\{i,k\}\in T}\beta_{ik}\,e^{\phi_i+\phi_k}.
\]
Fix \(j\in V\) and set \(\widetilde\phi_i:=\phi_i-\phi_j\) for all \(i\).
Then \(\widetilde\phi_j=0\), \(\phi_j=-\widetilde\phi_{v_0}\), and the change of coordinates
\(\phi_{\neq v_0}\mapsto\widetilde\phi_{\neq j}\) has unit Jacobian.

We track the factors in \(e^{\,t\phi_j}p_\beta^{(v_0)}(\phi)\).
First, \(\Phi(\phi)=\Phi(\widetilde\phi)\), since \(\Phi\) depends only on differences. Next,
\[
L_{v_0}(\phi)
=
\sum_{i\ne v_0}(\widetilde\phi_i+\phi_j)
=
\Big(\sum_{i\ne j}\widetilde\phi_i-\widetilde\phi_{v_0}\Big)
+(n-1)\phi_j,
\]
and therefore, using \(\phi_j=-\widetilde\phi_{v_0}\),
\[
e^{-L_{v_0}(\phi)}
=
e^{-\sum_{i\ne j}\widetilde\phi_i}\,e^{\,n\widetilde\phi_{v_0}}.
\]
Moreover, since every spanning tree has \(n-1\) edges,
\[
\Xi(\phi)
=
e^{\,2(n-1)\phi_j}\Xi(\widetilde\phi)
=
e^{-2(n-1)\widetilde\phi_{v_0}}\Xi(\widetilde\phi),
\]
and hence
\[
\sqrt{\Xi(\phi)}
=
e^{-(n-1)\widetilde\phi_{v_0}}\sqrt{\Xi(\widetilde\phi)}.
\]
Finally, \(e^{t\phi_j}=e^{-t\widetilde\phi_{v_0}}\). Multiplying these identities gives
\[
e^{\,t\phi_j}\,p_\beta^{(v_0)}(\phi)\,d\phi_{\neq v_0}
=
p_\beta^{(j)}(\widetilde\phi)\,
e^{\,(1-t)\widetilde\phi_{v_0}}\,
d\widetilde\phi_{\neq j}.
\]
Integrating over the corresponding pinned spaces gives the desired identity.
\end{proof}

\subsubsection{Bounds on $P_{ij}$, $\pi_*$, and the cover time}

Given the bound on the tail probabilities of $\sup_{i\in V}|\phi_i|$ proved in \cref{thm:fluc_phi}, we can establish bounds on the cover time of the ERRW. We accomplish this from the perspective of RWRE, by deriving high probability bounds on the cover time of the random walk with respect to the random conductance. To this end, we recall several bounds on the cover time of a random walk 
in terms of the underlying conductances.

Consider a connected graph \( G = (V, E) \) with edge conductances \( (q_e)_{e \in E} \). Let \( \mathcal{C} := \sum_{e \in E} q_e \) denote the total conductance. Denote the effective resistance between \( i, j \in V \) as \( R(i \leftrightarrow j) \), and let 
\[
\mathcal{R} := \max_{i,j \in V} R(i \leftrightarrow j)
\]
represent the maximum effective resistance of the network. For the simple random walk on this graph, starting at any $i \in V$, define the random hitting time of vertex \( j \in V \) as $\tau_{\text{hit}}(j) := \inf \{ t \geq 0 : X_t = j \}$, and let the hitting time be defined as $t_{\text{hit}} := \max_{i,j \in V} \Esrw^{i,q}(\tau_{\text{hit}}(j))$, where 
$\Esrw^{i,q}$ denotes expectations conditioned on 
starting at $i \in V$ with the conductances given by $q$.
Similarly, the commute time between vertices \( i, j \in V \) is defined as $t_{\text{comm}}(i,j) := \Esrw^{i,q}(\tau_{\text{hit}}(j)) + \Esrw^{j,q}(\tau_{\text{hit}}(i))$. Define the random cover time $\tau_{\text{cov}}$ as the first time when all states are visited at least once, and define the cover time as $t_{\text{cov}} := \max_{i \in V} \Esrw^{i,q}(\tau_{\text{cov}})$.

The following bounds on the cover time hold.

\vspace{.3em}
\begin{lem}[Conductance bound on the cover time]
\label{lem:cond_cover}
Given a connected graph $G=(V=[n], E)$ with positive edge conductances $(q_e)_{e\in E}$, we have
\[\frac{1}{2}\mc\mr\leq t_\cov\leq \mc\mr \log n.\]
\end{lem}
\begin{proof}
Recall that by Matthew's bound (Theorem 11.2 in \cite{lp}) we have $t_\hit\leq t_\cov\leq t_\hit\log n$. Also, by the commute time identity (Proposition 10.6 in \cite{lp}), we know that $t_\comm(i,j)=\mc R(i\leftrightarrow j)$. Since we have
\begin{align*}
t_\hit
&=\max_{i,j}\Esrw^{i,q}(\tau_\hit(j))\\
&\leq \max_{i,j}\big(\Esrw^{i,q}(\tau_\hit(j))+\Esrw^{j,q}(\tau_\hit(i))\big)\\
&\leq \max_{i,j}\Esrw^{i,q}(\tau_\hit(j))+\max_{i,j}\Esrw^{j,q}(\tau_\hit(i))\\
&=2t_\hit,
\end{align*}

and 
\[\max_{i,j}\big(\Esrw^{i,q}(\tau_\hit(j))+\Esrw^{j,q}(\tau_\hit(i))\big)=\max_{i,j}t_\comm(i,j)=\max_{i,j}\mc R(i\leftrightarrow j)=\mc\mr,\]
we conclude that $t_\hit\leq \mc\mr\leq 2t_\hit$, or equivalently, $\frac{1}{2}\mc\mr\leq t_\hit\leq \mc\mr$. Now using Matthew's bound, we finally deduce that $t_\cov\geq t_\hit\geq \frac{1}{2}\mc\mr$ and $t_\cov\leq t_\hit\log n\leq \mc\mr\log n$.
\end{proof}

This lemma states that, up to a factor of \( \log n \), the cover time is of the same order as \( \mathcal{C} \mathcal{R} \). Thus, determining the order of the cover time reduces to analyzing the order of the effective resistance and total conductance. To proceed, we require a comparison lemma.

\vspace{.3em}
\begin{lem}[Comparison of resistive networks]
\label{lem:comp_res}
For two resistive networks on the same graph with respective conductances \( (q_e)_{e \in E} \) and \( (q_e')_{e \in E} \), if
\[
\frac{q_e}{q_e'} \in [\underline{q}, \overline{q}] \quad \forall e \in E,
\]
then
\[
\frac{\mathcal{R}'}{\mathcal{R}} \in [\underline{q}, \overline{q}],
\]
where \( \mathcal{R} \) and \( \mathcal{R}' \) are the respective maximum effective resistances.
\end{lem}

\begin{proof}
By Dirichlet's principle, see Exercise 1, Chapter 3 of \cite{DS84}, the effective resistance between \( i \) and \( j \) is
\[
R(i \leftrightarrow j) = \bigg(\min_{v\in\mathbb{R}^V : v_i=1, v_j=0} \sum_{e=\{k,l\} \in E} q_e (v_k-v_l)^2\bigg)^{-1},
\]
and similarly for \( R'(i \leftrightarrow j) \) with \( q_e' \). 
Since for every edge $e\in E$ we have
\[
\underline{q} \leq \frac{q_e}{q_e'} \leq \overline{q},
\]
we know $\underline{q}\,q_e'(v_k-v_l)^2\leq q_e (v_k-v_l)^2\leq \overline{q}\,q_e'(v_k-v_l)^2$ 
for all $\{k,l\} \in E$, and so
\[\underline{q}\sum_{e=\{k,l\} \in E}q_e'(v_k-v_l)^2\leq \sum_{e=\{k,l\} \in E}q_e (v_k-v_l)^2\leq \overline{q}\sum_{e=\{k,l\} \in E}q_e'(v_k-v_l)^2.\]

Taking the minimum over $\{v\in\mathbb{R}^V : v_i - v_j = 1\}$, it follows that
\[
\underline{q} R'(i \leftrightarrow j)^{-1} \leq R(i \leftrightarrow j)^{-1} \leq \overline{q} R'(i \leftrightarrow j)^{-1},
\]
or equivalently,
\[
\underline{q} R(i \leftrightarrow j) \leq R'(i \leftrightarrow j) \leq \overline{q} R(i \leftrightarrow j).
\]

Now taking the maximum over \( i, j \in V\) yields the desired result.
\end{proof}

We can now demonstrate how to derive bounds on the cover time. Bounds on the order of the \( \beta \)-field, combined with those on the order of the \( \phi \)-field, provide bounds on the order of the random conductance \( W \). These bounds on \( W \) lead to bounds on the effective resistance \( \mr \), which then yield upper bounds on the cover time.

\vspace{.3em}
\begin{thm}[Upper bound on the cover time]
\label{thm:bound_cover}
Let $G=(V=[n],E)$ be a connected, weighted graph with positive edge weights
$A=(a_e)_{e\in E}$ and diameter $\diam$. Let
\[
d_{\max}:=\max_{v\in V}\deg(v).
\]
Assume that each $a_e$ satisfies $\underline a\le a_e\le \overline a$, where
$\underline a,\overline a$ are positive constants. There exists a constant \(g_1(\underline a,\overline a)>0\), depending only on
\(\underline a,\overline a\), such that for any \(\delta\in(0,1)\), the random environment $Q$ coupled with the ERRW satisfies
\[
t_\cov
\le
n^3\log n
\left(\frac{d_{\max}}{\delta}\right)^{2g_1(\underline a,\overline a)\diam}
\]
with probability at least $1-\delta$.
\end{thm}

\begin{proof}
If $d_{\max}=1$, then, since $G$ is connected, $G$ consists of a single edge, and the claim is immediate. We therefore assume $d_{\max}\ge2$. We repeatedly use
\[
n\le 1+d_{\max}+\cdots+d_{\max}^{\diam}\le 2d_{\max}^{\diam},
\qquad
|E|\le \frac{nd_{\max}}{2}\le d_{\max}^{2\diam}.
\]
The first bound follows by exploring the graph from any vertex up to distance $\diam$, and the second follows from the handshaking lemma. Write $R:=d_{\max}/\delta$. Since $d_{\max}\ge2$ and $\delta\in(0,1)$, we have $R\ge2$.

We first control the random edge field $\beta$. Since $\beta_e\sim\Gamma(a_e,1)$, for $0<\epsilon<1$,
\[
\Perrw^{v_0,A}(\beta_e\le \epsilon)
=
\int_0^\epsilon \frac{1}{\Gamma(a_e)}x^{a_e-1}e^{-x}\,dx
\le
\frac{\epsilon^{a_e}}{\Gamma(a_e+1)}
\le
2\epsilon^{\underline a},
\]
where we used $e^{-x}\le1$ and $\Gamma(x)>1/2$ for $x>0$. For the upper tail, exponential Markov's inequality with $\lambda=1/2$ gives, for every $\alpha>0$,
\[
\Perrw^{v_0,A}(\beta_e\ge\alpha)
\le
\Eerrw^{v_0,A}[e^{\beta_e/2}]e^{-\alpha/2}
=
2^{a_e}e^{-\alpha/2}
\le
e^{-\alpha/2+\overline a}.
\]

Choose
\[
\epsilon_\beta:=\left(\frac{\delta}{8|E|}\right)^{1/\underline a},
\qquad
\alpha_\beta:=2\left(\overline a+\log\frac{4|E|}{\delta}\right).
\]
Then the union bound gives
\[
\Perrw^{v_0,A}(\exists e\in E:\beta_e<\epsilon_\beta)
\le
2|E|\epsilon_\beta^{\underline a}
=
\frac{\delta}{4},
\]
and
\[
\Perrw^{v_0,A}(\exists e\in E:\beta_e>\alpha_\beta)
\le
|E|e^{-\alpha_\beta/2+\overline a}
=
\frac{\delta}{4}.
\]
Thus, with probability at least $1-\delta/2$, all edges satisfy $\epsilon_\beta\le\beta_e\le\alpha_\beta$.

Next, by \cref{lem:unit} with $t=1$, $\Eerrw^{v_0, A}[e^{\phi_i}]=1$ for every $i\in V$. Hence Markov's inequality and the union bound give, for any $x>0$,
\[
\Perrw^{v_0,A}\left(\max_{i\in V}e^{\phi_i}>x\right)
\le
\frac{n}{x}.
\]
Choosing $x_\phi:=4n/\delta$, we obtain
\[
\Perrw^{v_0,A}\left(\max_{i\in V}e^{\phi_i}>x_\phi\right)
\le
\frac{\delta}{4}.
\]

For the negative side of the $\phi$-field, choose $\alpha_0:=\log(16/\delta)/\log n$, so that $4n^{-\alpha_0}=\delta/4$. Applying \cref{thm:uniform-upper-clean} with this choice gives
\[
\Perrw^{v_0,A}\left(\max_{i\in V}e^{-\phi_i}>y_\phi\right)
\le
\frac{\delta}{4},
\]
where
\[
\begin{aligned}
y_\phi
:={}&
2^{\diam}n^{\alpha_0}
\sqrt{2(1+\alpha_0)\log n+2(\log2)\overline a d_{\max}} \\
&\times
\exp\left(
\diam\Lambda
+
\sqrt{2(1+\alpha_0)\Psi_1(\underline a/2)\diam\log n}
+
\frac{4}{\underline a}(1+\alpha_0)\log n
\right),
\end{aligned}
\]
and $\Lambda:=-\Psi(\underline a)+\log((\overline a d_{\max}+1)/2)$.

By the union bound, with probability at least $1-\delta$, all four estimates hold simultaneously:
\[
\epsilon_\beta\le\beta_e\le\alpha_\beta\text{ for all }e\in E,
\qquad
\max_i e^{\phi_i}\le x_\phi,
\qquad
\max_i e^{-\phi_i}\le y_\phi.
\]
On this event, for every edge $e=\{i,j\}$,
\[
Q_e=\beta_e e^{\phi_i+\phi_j}
\le
\alpha_\beta x_\phi^2,
\qquad
Q_e=\beta_e e^{\phi_i+\phi_j}
\ge
\epsilon_\beta y_\phi^{-2}.
\]
Thus
\begin{equation}
\label{eqn:Q-real-thresholds}
\forall e\in E,\qquad
Q_e\in[\epsilon_\beta y_\phi^{-2},\,\alpha_\beta x_\phi^2].
\end{equation}

It remains to compare the four thresholds with powers of \(R=d_{\max}/\delta\).
Recall that \(R\ge2\), \(\diam\ge1\), \(n\le2d_{\max}^{\diam}\le2R^\diam\), and
\(|E|\le d_{\max}^{2\diam}\le R^{2\diam}\). We also use repeatedly that
\(\log R\ge\log2\). Define
\[
C_\epsilon:=\frac{6}{\underline a},
\qquad
C_\alpha:=\frac{2\overline a}{\log2}+10,
\qquad
C_x:=5,
\]
and
\[
C_y
:=
8
+
\frac{\log(14+2(\log2)\overline a)}{\log2}
+
\frac{|\Psi(\underline a)|+\log(\overline a+1)}{\log2}
+
\sqrt{\frac{14\Psi_1(\underline a/2)}{\log2}}
+
\frac{28}{\underline a}.
\]
Set
\[
g_1(\underline a,\overline a)
:=
\max\{C_\alpha+2C_x,\ C_\epsilon+2C_y\}.
\]

First,
\[
\log\epsilon_\beta^{-1}
=
\frac{1}{\underline a}\log\frac{8|E|}{\delta}.
\]
Since \(|E|\le R^{2\diam}\) and \(\delta^{-1}\le R\),
\[
\log\frac{8|E|}{\delta}
\le
\log 8+(2\diam+1)\log R.
\]
Moreover, \(\log 8=3\log2\le3\diam\log R\), and
\(\log R\le\diam\log R\). Therefore
\[
\log\epsilon_\beta^{-1}
\le
\frac{1}{\underline a}\bigl(3\diam\log R+2\diam\log R+\diam\log R\bigr)
=
\frac{6}{\underline a}\diam\log R.
\]
Hence $\epsilon_\beta^{-1}\le R^{C_\epsilon\diam}$.

Next,
\[
\alpha_\beta
=
2\left(\overline a+\log\frac{4|E|}{\delta}\right)
\le
2\left(\overline a+\log4+(2\diam+1)\log R\right).
\]
Using \(\diam\ge1\) and \(\log R\ge\log2\), this gives
\[
\alpha_\beta
\le
\left(\frac{2\overline a}{\log2}+10\right)\diam\log R = C_\alpha\diam\log R = \log (R^{C_\alpha\diam}) \le R^{C_\alpha\diam}.
\]

Also,
\[
x_\phi=\frac{4n}{\delta}
\le
8R^{\diam+1},
\]
because \(n\le2R^\diam\) and \(\delta^{-1}\le R\). Taking logarithms gives
\[
\log x_\phi
\le
\log8+(\diam+1)\log R
\le
3\diam\log R+\diam\log R+\diam\log R
=
5\diam\log R.
\]
Thus
\(
x_\phi\le R^{C_x\diam}.
\)

It remains to bound \(y_\phi\). Recall that
\[
\alpha_0:=\frac{\log(16/\delta)}{\log n},
\qquad
n^{\alpha_0}=\frac{16}{\delta}.
\]
Hence \((1+\alpha_0)\log n=\log n+\log(16/\delta)\). Since \(n\le2R^\diam\),
\(\delta^{-1}\le R\), \(R\ge2\), and \(\diam\ge1\), we have
\[
\log n\le2\diam\log R,
\qquad
\log(16/\delta)\le5\diam\log R,
\]
and therefore
\[
(1+\alpha_0)\log n\le7\diam\log R.
\]

Taking logarithms in the definition of \(y_\phi\), we get
\[
\begin{aligned}
\log y_\phi
\le{}&
\diam\log2+\alpha_0\log n
+\frac12\log\left(2(1+\alpha_0)\log n+2(\log2)\overline a d_{\max}\right) \\
&\quad
+\diam\Lambda
+\sqrt{2(1+\alpha_0)\Psi_1(\underline a/2)\diam\log n}
+\frac{4}{\underline a}(1+\alpha_0)\log n ,
\end{aligned}
\]
where \(\Lambda=-\Psi(\underline a)+\log((\overline a d_{\max}+1)/2)\).

We now bound the terms on the right crudely. First,
\[
\diam\log2+\alpha_0\log n
=
\diam\log2+\log(16/\delta)
\le6\diam\log R.
\]
For the logarithmic square-root prefactor, since \((1+\alpha_0)\log n\le7\diam\log R\) and \(d_{\max}\le R\),
\[
2(1+\alpha_0)\log n+2(\log2)\overline a d_{\max}
\le
14\diam\log R+2(\log2)\overline a R.
\]
Using \(\diam\log R\le R^\diam\) and \(R\le R^\diam\), the quantity inside the logarithm is at most
\((14+2(\log2)\overline a)R^\diam\). Thus, after absorbing the constant into
\(\diam\log R\), this logarithmic term is at most
\[
\left(1+\frac{\log(14+2(\log2)\overline a)}{\log2}\right)\diam\log R.
\]

Next, since \(d_{\max}\le R\),
\[
\Lambda\le |\Psi(\underline a)|+\log(\overline a+1)+\log R,
\]
and hence
\[
\diam\Lambda
\le
\left(
1+\frac{|\Psi(\underline a)|+\log(\overline a+1)}{\log2}
\right)\diam\log R.
\]
For the Gaussian fluctuation term, using \((1+\alpha_0)\log n\le7\diam\log R\),
\[
\sqrt{2(1+\alpha_0)\Psi_1(\underline a/2)\diam\log n}
\le
\sqrt{14\Psi_1(\underline a/2)}\,\diam\sqrt{\log R}
\le
\frac{\sqrt{14\Psi_1(\underline a/2)}}{\sqrt{\log2}}\,\diam\log R.
\]
Finally,
\[
\frac{4}{\underline a}(1+\alpha_0)\log n
\le
\frac{28}{\underline a}\diam\log R.
\]
Combining these estimates gives
\[
\log y_\phi\le C_y\diam\log R,
\]
where one may take
\[
C_y
:=
8
+
\frac{\log(14+2(\log2)\overline a)}{\log2}
+
\frac{|\Psi(\underline a)|+\log(\overline a+1)}{\log2}
+
\sqrt{\frac{14\Psi_1(\underline a/2)}{\log2}}
+
\frac{28}{\underline a}.
\]
Equivalently, \(y_\phi\le R^{C_y\diam}\).

We have shown
\[
\epsilon_\beta^{-1}\le R^{C_\epsilon\diam},
\qquad
\alpha_\beta\le R^{C_\alpha\diam},
\qquad
x_\phi\le R^{C_x\diam},
\qquad
y_\phi\le R^{C_y\diam}.
\]
Combining these with \eqref{eqn:Q-real-thresholds}, we obtain
\[
Q_e\le R^{(C_\alpha+2C_x)\diam},
\qquad
Q_e\ge R^{-(C_\epsilon+2C_y)\diam}.
\]
Thus, with
\[
g_1(\underline a,\overline a)
:=
\max\{C_\alpha+2C_x,\ C_\epsilon+2C_y\},
\]
we have, with probability at least \(1-\delta\),
\[
\forall e\in E,\qquad
Q_e\in
\left[
R^{-g_1(\underline a,\overline a)\diam},
R^{g_1(\underline a,\overline a)\diam}
\right].
\]

On this event, \cref{lem:comp_res} implies that, if $\mr$ is the maximum effective resistance of the unit-conductance network and $\mr'$ is the corresponding maximum effective resistance for the conductances $Q$, then $\mr'\le \mr R^{g_1\diam}$. Also, $\mc':=\sum_{e\in E}Q_e\le |E|R^{g_1\diam}$. Using \cref{lem:cond_cover}, we conclude that
\[
t_\cov
\le
\mc'\mr'\log n
\le
|E|\mr\log n\,R^{2g_1\diam}.
\]
Finally, since $|E|\le n^2$ and $\mr\le n$\footnote{Note that the maximum resistance one could get is the resistance between two endpoints of an $n$-path.}, we obtain
\[
t_\cov
\le
n^3\log n
\left(\frac{d_{\max}}{\delta}\right)^{2g_1\diam}.
\]
This proves the theorem.
\end{proof}

We summarized the cover time bound specialized to some common graph families in the following table.

\begin{table}[h]
\label{table:cover-bd}
\centering
\begin{tabular}{|l|c|c|c|}
\hline
\textbf{Graph/Model} & $d_{\max}$ & $\diam$ & $t_\cov$ \\
\hline
$n$-path / $n$-cycle & $2$
& $\Theta(n)$ & $e^{O(n)}$ \\
\hline
$d$-dim grid/torus (constant $d\ge 2$) & $2d$
& $\Theta(n^{1/d})$ & $e^{O(n^{1/d})}$ \\
\hline
Hypercube & $\log n$
& $\log n$ & $n^{O(\log\log n)}$ \\
\hline
full $r$-ary tree (constant $r\ge 2$) & $r+1$
& $\Theta(\log n)$ & $\poly(n)$ \\
\hline
Star graph (coupon collector) & $n-1$
& $2$ & $\poly(n)$ \\
\hline
$d$-regular expander (constant $d$) & $d$
& $O(\log n)$ & $\poly(n)$ \\
\hline
\end{tabular}
\caption{The cover time bound for ERRW on classical graph families.}
\end{table}

We also have the following bound on the entries of the random transition matrix $P$ as well as the minimum value  $\pi_*:=\min_{i\in V}\pi_i$
of the stationary distribution associated with it.

\vspace{.3em}
\begin{cor}[Lower bounds on $P_{ij}$ and $\pi_*$]
  \label{cor:pi_star_gamma}
  In the same setting as \cref{thm:bound_cover}, for any $\delta\in(0,1)$, the following hold with probability at least $1-\delta$:
  \[
    \pi_* \ge (nd_{\max})^{-1}
    \left(\frac{d_{\max}}{\delta}\right)^{-2g_1(\underline a,\overline a)\diam},
    \qquad
    \min_{i,j\in V:\,i\sim j} P_{ij}
    \ge
    d_{\max}^{-1}
    \left(\frac{d_{\max}}{\delta}\right)^{-2g_1(\underline a,\overline a)\diam}.
  \]
\end{cor}

\begin{proof}
  By the proof of \cref{thm:bound_cover}, with probability at least $1-\delta$, every edge conductance satisfies
  \[
    Q_e\in [B^{-1},B],
    \qquad
    B:=\left(\frac{d_{\max}}{\delta}\right)^{g_1\diam}.
  \]
  On this event, since $G$ is connected, each vertex has at least one incident edge, and hence $Q_i:=\sum_{j\sim i}Q_{ij}\ge B^{-1}$. Also, $\sum_{v\in V}Q_v=2\sum_{e\in E}Q_e\le 2|E|B\le nd_{\max}B$. Therefore, for every $i\in V$,
  \[
    \pi_i=\frac{Q_i}{\sum_{v\in V}Q_v}
    \ge
    \frac{B^{-1}}{nd_{\max}B}
    =
    (nd_{\max})^{-1}
    \left(\frac{d_{\max}}{\delta}\right)^{-2g_1\diam}.
  \]
  Taking the minimum over $i$ gives the lower bound on $\pi_*$. Similarly, for any edge $e=\{i,j\}$, we have $Q_{ij}\ge B^{-1}$ and $Q_i=\sum_{k\sim i}Q_{ik}\le d_{\max}B$, so
  \[
    P_{ij}=\frac{Q_{ij}}{Q_i}
    \ge
    \frac{B^{-1}}{d_{\max}B}
    =
    d_{\max}^{-1}
    \left(\frac{d_{\max}}{\delta}\right)^{-2g_1\diam}.
  \]
  Taking the minimum over adjacent pairs completes the proof.
\end{proof}

\subsection{A non-asymptotic sample complexity bound}

Recall the setup and notation in \cref{sec:estimatingmoments}, where $U_e$ is defined as (\Cref{eqn:udefinition}):
\begin{equation*}
U_{ij} :=
\begin{cases} 
 P_{ij}P_{ji}, \quad \text{ if } e =\{i,j\} \in E,\\
 0, \quad \text{ if } i \nsim j.
\end{cases}
\end{equation*}
Given $K$ independent trajectories $\{X_t^{(k)}\,|\,t\leq T, k\leq K\}$ from $\Perrw^{v_0, A}$, we want to estimate the following moment-based statistics connected with $\{U_e:e\in E\}$: 
\begin{itemize}
\item[i.] $\Eerrw^{v_0,A}(U_e), \forall e\in E$;
\item[ii.] $\Eerrw^{v_0,A}(U_e U_{e'}), \forall e,e'\in E\text{ s.t. }\,|e\cap e'|=1$;
\item[iii.] $\Eerrw^{v_0,A}(U_e)\Eerrw^{v_0,A}(U_{e'})-\Eerrw^{v_0,A}(U_eU_{e'}), \forall e,e'\in E\text{ s.t. }\,|e\cap e'|=1$.
\end{itemize}
Also, recall that we can couple a random conductance $Q^{(k)}$ with the $k$-th trajectory. We denote the corresponding transition matrix as $P^{(k)}$. 
We can assume that each of the $K$ trajectories is a truncation of an infinitely long trajectory. For any \(m \in \mathbb{N}_+\), let \(H_{ij}^\infty(k; m)\) represent the number of transitions from $i$ to $j$ among the first \(m\) outgoing transitions from \(i\) in the original, non-truncated \(k\)-th trajectory. The following random variable is always well-defined due to the ergodic property of irreducible Markov chains: 
\begin{equation}
\label{eqn:errw-est-theta}
\Theta_{ij}^{(k)}:= \begin{cases}
    \frac{H_{ij}^\infty(k; m)}{m}, & \quad\forall \, i, j\in V\text{ s.t. }i\sim j,\\
    0 & \quad \text{otherwise}.
    \end{cases}
\end{equation}
Recall that $N_i(k)$ denoted the number of outgoing transitions from $i$ in the $k$-th (truncated) trajectory,
and $N_{ij}(k)$ denoted the number of transitions from $i$ to $j$ in this trajectory, see 
\cref{sec:estimatingmoments}.
Then, on the event $\{N_i(k)\geq m\}$ we have $\Theta_{ij}^{(k)} = \hat{P}_{ij}^{(k)}$ for all $i \sim j$, where $\hat{P}_{ij}^{(k)}$ is defined in
\cref{eqn:est-p}.

Given the results on the cover time $t_\cov$ associated with a random environment that we have already proved in \cref{thm:bound_cover} and Corollary~\ref{cor:pi_star_gamma}, we want to bound the complexity of approximating the three families of moment-based statistics above, up to a small multiplicative error, non-asymptotically. 

\textbf{Step 1. Estimating $\Eerrw^{v_0,A} (U_e),\forall e\in E$.}

The proposed estimator for $\Eerrw^{v_0,A} (U_e)$
is $\hat{U}_e$, which is
defined in \cref{eqn:est-u}.
To bound the estimation error of $\hat{U}_e$ for $\Eerrw^{v_0,A}(U_e)$, the first step is to bound the estimation error of $\hat{P}^{(k)}$ for $P^{(k)}$. For $\hat{P}^{(k)}$ to achieve a small estimation error for $P^{(k)}$, one must gather sufficiently many transitions from each state $i\in V$, and then reconstruct the outgoing transition probabilities of state $i$ based on those transitions.

Let's first consider the subproblem of statistical estimation for an $n$-state discrete distribution 
\[
  \mathbf{p} := (p_1, p_2, \ldots, p_n)
\]
such as to ensure that each entry is accurate up to a small multiplicative factor $\eta \in (0,1)$.  
Suppose we draw $m$ i.i.d.\,samples from $\mathbf{p}$, and let $\hat{p}_i$ be the empirical estimator for $p_i$, i.e., the proportion of times that $i$ appears in the $m$ samples. By Hoeffding's inequality, for each $i\in[n]$ we have:
\begin{equation}
\label{eqn:errw-discrete_dist1}
  \Pr\bigl(\,\lvert \hat{p}_i - p_i\rvert 
            \ge \eta \, p_i\bigr)
  \le
  2\, \exp\!\bigl(-2\,\eta^2\,p_i^2\,m\bigr).
\end{equation}

Taking a union bound over $i\in[n]$, it follows that
\begin{equation}
\label{eqn:errw-discrete_dist2}
  \Pr\!\Bigl(\exists \, i \in [n] \colon\,
       \lvert \hat{p}_i - p_i\rvert 
       \ge \eta \,p_i\Bigr)
  \le
  \sum_{i=1}^n \Pr\bigl(\,\lvert \hat{p}_i - p_i\rvert 
            \ge \eta \, p_i\bigr)
  \le
  2\,n \,\exp\!\bigl(-2\,\eta^2\,p_*^2\,m\bigr),
\end{equation}
where $p_* := \min_{i\in[n]} \, p_i$. Consequently, with probability at least $1 - 2\,n\,\exp(-2\,\eta^2 p_*^2 m)$ the empirical estimator $\{\hat{p}_i:\,i\in V\}$ is an entry-wise $\eta$-multiplicative 
approximation of \textbf{p}. 

Fix any $\delta'\in(0,1)$. In the ERRW setting, with probability at least $1 -\delta'$, 
for each $k \in [K]$ the random environment satisfies $ Q_e^{(k)} \,\in\, [(\frac{d_{\max}}{\delta'})^{-g_1(\underline{a},\overline{a})\diam}, (\frac{d_{\max}}{\delta'})^{g_1(\underline{a},\overline{a})\diam}]$ for all $e \in E$. 
Let
\begin{equation}     
\label{eqn:errw-goodQset}
\mathcal{Q}
:=\left\{
q\in\mathbb{R}_{++}^E, 
q_e\in 
\left[
\bigg(\frac{d_{\max}}{\delta'}\bigg)^{-g_1(\underline{a},\overline{a})\diam}, 
\bigg(\frac{d_{\max}}{\delta'}\bigg)^{g_1(\underline{a},\overline{a})\diam}
\right], 
\forall e\in E
\right\}. 
\end{equation}
Conditioned on $Q^{(k)} \in \mathcal{Q}$,
we have for all $e=\{i,j\} \in E$ that
\begin{equation}
\label{eqn:errw-pij_bound1}
P_{ij}^{(k)} \ge d_{\max}^{-1}\,\Big(\frac{d_{\max}}{\delta'}\Big)^{-2g_1(\underline{a},\overline{a})\,\diam},
\end{equation}
as established in Corollary~\ref{cor:pi_star_gamma}.


From \cref{eqn:errw-discrete_dist2}, for all $Q \in \mathcal{Q}$ we also have
\begin{equation}
\label{eqn:errw-pij_bound_variant}
\Psrw^{v_0,Q}\bigl(\exists\,j\sim i,\, \lvert \Theta_{ij}^{(k)} - P_{ij}^{(k)}\rvert 
            \ge \eta \, P_{ij}^{(k)}\bigr)
  \le \sum_{j\sim i} 2 \exp\!\bigl(-2\,\eta^2\,(P_{ij}^{(k)})^2\,m\bigr).
\end{equation}
Then we have
\begin{align}
\label{eqn:errw-core1}
  \Perrw^{v_0,A}\bigl(\exists i\in V,\,\exists\,j\sim i,\, & \lvert \Theta_{ij}^{(k)} - P_{ij}^{(k)}\rvert 
            \ge\eta \, P_{ij}^{(k)}\bigr)\nonumber \\
  &= \int_{\mathcal{Q}}\Psrw^{v_0,q^{(k)}}\bigl(\exists i\in V,\, \exists\,j\sim i,\, \lvert \Theta_{ij}^{(k)} - P_{ij}^{(k)}\rvert 
            \ge \eta \, P_{ij}^{(k)}\bigr) d\nu_{v_0,A}(q^{(k)}) \nonumber \\
  & \hspace{1em} + \int_{\mathcal{Q}^c}\Psrw^{v_0,q^{(k)}}\bigl(\exists i\in V,\, \exists\,j\sim i,\, \lvert \Theta_{ij}^{(k)} - P_{ij}^{(k)}\rvert 
            \ge \eta \, P_{ij}^{(k)}\bigr) d\nu_{v_0,A}(q^{(k)})\nonumber \\
  &\le \int_{\mathcal{Q}} \sum_{i\in V}\sum_{j\sim i} 2 \exp\!\bigl(-2\,\eta^2\,(P_{ij}^{(k)})^2\,m\bigr) d\nu_{v_0,A}(q^{(k)}) + \int_{\mathcal{Q}^c} \,1\, d\nu_{v_0,A}(q^{(k)})\nonumber \\
  & \le \int_{\mathcal{Q}} 4\,|E| \,\exp\Bigl(-2\,\eta^2 
      \Big(d_{\max}^{-1}\,\Big(\frac{d_{\max}}{\delta'}\Big)^{-2g_1(\underline{a},\overline{a})\,\diam}\Big)^2 
      \,m \Bigr) d\nu_{v_0,A}(q^{(k)}) \nonumber \\
  & \hspace{19em} + \Perrw^{v_0,A}(Q^{(k)}\in\mathcal{Q}^c)\nonumber \\
  & \leq 2\,nd_{\max} \Big(\int_{\mathcal{Q}}\,1\,d\nu_{v_0,A}(q^{(k)})\Big) \,\exp\Bigl(-2\,\eta^2 m 
      d_{\max}^{-2}\,\Big(\frac{d_{\max}}{\delta'}\Big)^{-4g_1(\underline{a},\overline{a})\,\diam}  \Bigr)\nonumber \\
  & \hspace{19em} + \delta'\nonumber \\
  & \le 2\,nd_{\max} \,\exp\Bigl(-2\,\eta^2 m 
      d_{\max}^{-2}\,\Big(\frac{d_{\max}}{\delta'}\Big)^{-4g_1(\underline{a},\overline{a})\,\diam}  \Bigr) + \delta'.
\end{align}
Here we used the union bound and \cref{eqn:errw-pij_bound_variant} in the second step, and \cref{eqn:errw-pij_bound1} in the third step. Also, we used \cref{thm:bound_cover} in the fourth step. In the last step, we used $\int_{\mathcal{Q}}\,1\,d\nu_{v_0,A}(q^{(k)})=\Perrw^{v_0,A}(Q^{(k)}\in\mathcal{Q})\leq 1$. 

To obtain $m$ samples per state, we use the notion of the $m$-cover time 
$t_{\cov}(m)$. To define $t_{\cov}(m)$, we first define the random $m$-cover time $\tau_{\cov}(m)$ as the first time when each state has been visited at least $m$ times. Then the $m$-cover time can be defined as $t_{\cov}(m):=\max_{v_0\in V}\Esrw^{v_0,q}(\tau_{\cov}(m))$.

For reversible random walks, it is known (Theorem 16, \cite{cdl}) that there exists a universal constant $g_2>0$ (which does not depend on the size of the connected graph or its structure) such that
\[
  t_{\cov}(m) \leq g_2\cdot \Big(t_{\cov} + \tfrac{m}{\pi_*}\Big).
\]
From the discussion that led to \cref{thm:bound_cover} and Corollary~\ref{cor:pi_star_gamma}, we know that
on the event $\mathcal{Q}$, which has probability $\geq 1-\delta'$, we have
\[
t_\cov \leq n^3\log n \left( \frac{d_{\max}}{\delta'}\right)^{2g_1(\underline{a},\overline{a})\,\diam},
\]
and $\pi_* \ge (nd_{\max})^{-1}\,(\frac{d_{\max}}{\delta'})^{-2g_1(\underline{a},\overline{a})\,\diam}$. Hence, on the event $\mathcal{Q}$ we have
\[
  t_{\cov}(m+1) \le g_2\cdot \Big(n^3\log n + (m+1)nd_{\max}\Big)\cdot \left( \frac{d_{\max}}{\delta'}\right)^{2g_1(\underline{a},\overline{a})\,\diam}.
\]

Now realize that by Markov inequality, $\Psrw^{v_0,q}(\tau_\cov(m+1)\geq e t_\cov(m+1))\leq e^{-1}$. 
Actually according to Lemma 5 in~\cite{cdl}, we have $\forall l\in\mathbb{Z}_+$ that $\Psrw^{v_0,q}(\tau_\cov(m+1)\geq el \cdot t_\cov(m+1))\leq e^{-l}$. 

We now use the notation $\tau_\cov^{(k)}(m)$ for the $m$-cover time of the trajectory associated with the $k$-th ERRW trajectory. Conditioned on $Q^{(k)}=q$, the $k$-th ERRW is simply a random walk from $\Psrw^{v_0,q}$. The conditional probability of not having at least $m+1$ visits to each state is bounded by
\[
  \Psrw^{v_0,q}(\tau_\cov^{(k)}(m+1)\geq T)\leq \exp\Bigl(-\Big\lfloor\tfrac{T}{e\,t_\cov^{(k)}(m+1)}\Big\rfloor\Bigr)\leq \exp\Bigl(-\tfrac{T}{e\,t_\cov^{(k)}(m+1)}+1\Bigr).
\]

Similar to the procedure in \cref{eqn:errw-core1}, we have
\begin{align}
\label{eqn:errw-core2}
  \Perrw^{v_0,A}\bigl(\tau_\cov^{(k)}(m+1)\geq T\bigr)
  &= \int_{\mathcal{Q}}\Psrw^{v_0,q^{(k)}}\bigl(\tau_\cov^{(k)}(m+1)\geq T\bigr) d\nu_{v_0,A}(q^{(k)}) \nonumber \\
  & \hspace{4em} + \int_{\mathcal{Q}^c}\Psrw^{v_0,q^{(k)}}\bigl(\tau_\cov^{(k)}(m+1)\geq T\bigr) d\nu_{v_0,A}(q^{(k)})\nonumber \\
  &\le \int_{\mathcal{Q}} \exp\Bigl(-\tfrac{T}{e\,t_\cov^{(k)}(m+1)}+1\Bigr) d\nu_{v_0,A}(q^{(k)}) + \int_{\mathcal{Q}^c} \,1\, d\nu_{v_0,A}(q^{(k)})\nonumber \\
  & \le \exp\Bigl(-\tfrac{T}{e\,g_2\cdot(n^3\log n + (m+1)nd_{\max})\cdot \left( \frac{d_{\max}}{\delta'}\right)^{2g_1(\underline{a},\overline{a})\,\diam}}+1\Bigr) + \delta'.
\end{align}

On the event \(\{\tau_\cov^{(k)}(m+1)\le T\}\), each state has at least \(m\) outgoing transitions before time \(T\), and hence \(\hat P_{ij}^{(k)}=\Theta_{ij}^{(k)}\) for all \(i\sim j\). Therefore, combining \cref{eqn:errw-core1} and \cref{eqn:errw-core2}, we conclude that
\begin{align}
\label{eqn:errw-core3}
\Perrw^{v_0,A} & \bigl( \{\forall i\in V,\,\forall\,j\sim i,\, \lvert \hat{P}_{ij}^{(k)} - P_{ij}^{(k)}\rvert \le\eta \, P_{ij}^{(k)}\}\cap \{\tau_\cov^{(k)}(m+1)\leq T\}\bigr)\nonumber\\
& = \Perrw^{v_0,A} \bigl( \{\forall i\in V,\,\forall\,j\sim i,\, \lvert \Theta_{ij}^{(k)} - P_{ij}^{(k)}\rvert \le\eta \, P_{ij}^{(k)}\}\cap \{\tau_\cov^{(k)}(m+1)\leq T\}\bigr)\nonumber \\
& \geq 1 - 2\,nd_{\max} \,\exp\Bigl(-2\,\eta^2 m 
      d_{\max}^{-2}\,\Big(\frac{d_{\max}}{\delta'}\Big)^{-4g_1(\underline{a},\overline{a})\,\diam}  \Bigr) \nonumber\\
& \hspace{10em} - \exp\Bigl(-\tfrac{T}{e\,g_2\cdot(n^3\log n + (m+1)nd_{\max})\cdot \left( \frac{d_{\max}}{\delta'}\right)^{2g_1(\underline{a},\overline{a})\,\diam}}+1\Bigr) - 2\delta'
\end{align}
Here we used the fact that on the event $\{N_i(k)\geq m, \forall i\in V\}$ we have $\Theta_{ij}^{(k)} = \hat{P}_{ij}^{(k)}$ for all $i,j\in V$ s.t. $i\sim j$. Also the condition $\{N_i(k)\geq m, \forall i\in V\}$ is implied by the condition
$\{\tau_\cov^{(k)}(m+1)\leq T\}$,

We now want to choose \(T\) and \(m\) appropriately 
so that with high probability, we simultaneously (i) gather at least $m$ samples per state and (ii) achieve the desired \(\eta\)-approximation of all transition probabilities. We then take 
\[m=\frac{d_{\max}^2}{2 \eta^2}\Big(\frac{d_{\max}}{\delta'}\Big)^{4g_1(\underline{a},\overline{a})\diam}\cdot\log\Big(\frac{2nd_{\max}}{\delta'}\Big)\]
to make the term 
\[2\,n d_{\max} \,\exp\Bigl(-2\,\eta^2 m d_{\max}^{-2}\,\Big(\frac{d_{\max}}{\delta'}\Big)^{-4g_1(\underline{a},\overline{a})\,\diam}  \Bigr)\leq \delta'.\]

We further take
\begin{align*}
T 
&= e\,g_2(n^3\log n +(m+1)nd_{\max})\Big(\frac{d_{\max}}{\delta'}\Big)^{2g_1(\underline{a},\overline{a})\diam}\log\Big(\frac{e}{\delta'}\Big)\\
\end{align*}
to make the term
\[\exp\Bigl(-\tfrac{T}{e\,g_2\cdot(n^3\log n + (m+1)nd_{\max})\cdot \left( \frac{d_{\max}}{\delta'}\right)^{2g_1(\underline{a},\overline{a})\,\diam}}+1\Bigr)\leq \delta'.\]

Given this choice of $m$ and $T$, we have
\[\Perrw^{v_0,A} \bigl( \{\forall i\in V,\,\forall\,j\sim i,\, \lvert \hat{P}_{ij}^{(k)} - P_{ij}^{(k)}\rvert \le\eta \, P_{ij}^{(k)}\}\cap \{\tau_\cov^{(k)}(m+1)\leq T\}\bigr)\geq 1-4\delta'.\]

Chernoff bound will be used for bounding concentration over the $K$ trajectories. 
\medskip
\begin{lem}[Multiplicative Chernoff for bounded variables]
Let \(X_1,\dots,X_K\) be i.i.d. random variables in \([0,1]\), with mean \(\mu\). Then for \(\epsilon\in(0,1)\),
\[
\Pr\left(\left|\frac1K\sum_{k=1}^K X_k-\mu\right|\ge \epsilon\mu\right)
\le
2\exp\left(-\frac{\epsilon^2}{3}\mu K\right).
\]
\end{lem}

For later use, define
\[
\mu_*:=\frac{\underline a(\underline a+1)}{(d_{\max}\overline a+1)^2},
\qquad
c_0:=\frac{\underline a}{\underline a+1},
\qquad
\mathcal P_1:=\{\{e,e'\}:e,e'\in E,\ e\neq e',\ |e\cap e'|=1\}.
\]

We now fix arbitrary $\epsilon\in(0,1/2)$. We first control the bias introduced by
using the finite trajectory estimator $\hat U_e$ in place of the ideal quantity $U_e$.
As before,
\[
\Eerrw^{v_0,A}(\hat U_e)
=
\big(\Eerrw^{v_0,A}(\hat U_e)-\Eerrw^{v_0,A}(\overline U_e)\big)
+
\Eerrw^{v_0,A}(U_e),
\]
where we used $\Eerrw^{v_0,A}(\overline U_e)=\Eerrw^{v_0,A}(U_e)$. Therefore
\[
\big|\Eerrw^{v_0,A}(\hat U_e)-\Eerrw^{v_0,A}(U_e)\big|
\le
\frac1K\sum_{k=1}^K
\Eerrw^{v_0,A}\big|\hat U_e^{(k)}-U_e^{(k)}\big|.
\]

We next bound the summands. For every $k$,
\begin{align}
\label{eqn:errw-core7}
\Eerrw^{v_0,A}\big|\hat U_e^{(k)}-U_e^{(k)}\big|
&\le
2\Perrw^{v_0,A}\big(|\hat U_e^{(k)}-U_e^{(k)}|>\eta U_e^{(k)}\big)
+\eta.
\end{align}
Indeed, on the event
$|\hat U_e^{(k)}-U_e^{(k)}|\le \eta U_e^{(k)}$ the contribution is at most
$\eta\E(U_e^{(k)})\le\eta$, while on the complement we use the crude bound
$|\hat U_e^{(k)}-U_e^{(k)}|\le2$.

Moreover, if
\[
\hat P_{ij}^{(k)}\in\left(1\pm\frac{\eta}4\right)P_{ij}^{(k)}
\quad\text{and}\quad
\hat P_{ji}^{(k)}\in\left(1\pm\frac{\eta}4\right)P_{ji}^{(k)},
\]
then $\hat U_e^{(k)}\in(1\pm\eta)U_e^{(k)}$ for every $\eta\in(0,1)$.
Therefore,
\begin{align*}
\Perrw^{v_0,A}\big(|\hat U_e^{(k)}-U_e^{(k)}| & >\eta U_e^{(k)}\big) \\
& \le
\Perrw^{v_0,A}\left(
\left\{\exists i\in V,\ j\sim i:
|\hat P_{ij}^{(k)}-P_{ij}^{(k)}|
\ge \frac{\eta}4 P_{ij}^{(k)}
\right\}
\cup
\{\tau_\cov^{(k)}(m+1)>T\}
\right).
\end{align*}
By the transition-matrix estimate above, the last probability is at most $4\delta'$ provided
\begin{equation}
\label{eqn:errw-mchoice}
m\ge
\frac{8d_{\max}^2}{\eta^2}
\left(\frac{d_{\max}}{\delta'}\right)^{4g_1\diam}
\log\left(\frac{2nd_{\max}}{\delta'}\right),
\end{equation}
and
\begin{equation}
\label{eqn:errw-Tchoice}
T\ge
e\,g_2\big(n^3\log n+(m+1)nd_{\max}\big)
\left(\frac{d_{\max}}{\delta'}\right)^{2g_1\diam}
\log\left(\frac{e}{\delta'}\right),
\end{equation}
where we write $g_1$ for $g_1(\underline a,\overline a)$.

Using this in \eqref{eqn:errw-core7}, we get
\[
\Eerrw^{v_0,A}\big|\hat U_e^{(k)}-U_e^{(k)}\big|
\le
8\delta'+\eta.
\]
Taking $\eta=\delta'$ gives
\[
\Eerrw^{v_0,A}\big|\hat U_e^{(k)}-U_e^{(k)}\big|\le9\delta',
\]
and hence
\begin{equation}
\label{eqn:errw-core10}
\big|\Eerrw^{v_0,A}(\hat U_e)-\Eerrw^{v_0,A}(U_e)\big|
\le9\delta'.
\end{equation}

From the explicit moment formula,
\begin{equation}
\label{eqn:errw-ue_lb}
\Eerrw^{v_0,A}(U_e)
=
\frac{a_e(a_e+1)}
{(a_i+1-\ind_{i=v_0})(a_j+1-\ind_{j=v_0})}
\ge
\mu_*.
\end{equation}
If
\begin{equation}
\label{eqn:errw-deltaprimecondition}
\delta'\le \frac{\epsilon}{9}\mu_*,
\end{equation}
then \eqref{eqn:errw-core10} implies
\begin{equation}
\label{eqn:errw-core-add1}
\big|\Eerrw^{v_0,A}(\hat U_e)-\Eerrw^{v_0,A}(U_e)\big|
\le
\epsilon\,\Eerrw^{v_0,A}(U_e).
\end{equation}
In particular, since $\epsilon<1/2$,
\[
\Eerrw^{v_0,A}(\hat U_e)
\ge
(1-\epsilon)\Eerrw^{v_0,A}(U_e)
\ge
\frac12\mu_*.
\]

We now use multiplicative Chernoff instead of Chebyshev. Since
$\hat U_e^{(1)},\ldots,\hat U_e^{(K)}$ are i.i.d. random variables in $[0,1]$ and
\[
\hat U_e=\frac1K\sum_{k=1}^K\hat U_e^{(k)},
\]
the bounded-variable multiplicative Chernoff bound gives
\[
\Perrw^{v_0,A}\big(
|\hat U_e-\Eerrw^{v_0,A}(\hat U_e)|
\ge
\epsilon\,\Eerrw^{v_0,A}(\hat U_e)
\big)
\le
2\exp\left(
-\frac{\epsilon^2}{3}\Eerrw^{v_0,A}(\hat U_e)K
\right).
\]
Using $\Eerrw^{v_0,A}(\hat U_e)\ge\mu_*/2$, we obtain
\begin{equation}
\label{eqn:errw-U_chernoff_single}
\Perrw^{v_0,A}\big(
|\hat U_e-\Eerrw^{v_0,A}(\hat U_e)|
\ge
\epsilon\,\Eerrw^{v_0,A}(\hat U_e)
\big)
\le
2\exp\left(
-\frac{\epsilon^2\mu_*K}{6}
\right).
\end{equation}

Combining \eqref{eqn:errw-core-add1} and \eqref{eqn:errw-U_chernoff_single}, we get
\[
\Perrw^{v_0,A}\big(
|\hat U_e-\Eerrw^{v_0,A}(U_e)|
\ge
3\epsilon\,\Eerrw^{v_0,A}(U_e)
\big)
\le
2\exp\left(
-\frac{\epsilon^2\mu_*K}{6}
\right).
\]
Finally, taking a union bound over $e\in E$ yields
\begin{equation}
\label{eqn:errw-U_chernoff_union}
\Perrw^{v_0,A}\big(
\exists e\in E:\,
|\hat U_e-\Eerrw^{v_0,A}(U_e)|
\ge
3\epsilon\,\Eerrw^{v_0,A}(U_e)
\big)
\le
2|E|\exp\left(
-\frac{\epsilon^2\mu_*K}{6}
\right).
\end{equation}
In particular, since $|E|\le nd_{\max}/2$, it suffices to take
\[
K
\ge
\frac{6}{\epsilon^2\mu_*}
\log\left(\frac{nd_{\max}}{\delta}\right)
\ge 
\frac{6}{\epsilon^2\mu_*}
\log\left(\frac{2|E|}{\delta}\right)
\]
to make the probability in \eqref{eqn:errw-U_chernoff_union} at most $\delta$.

\textbf{Step 2. Estimating $\Eerrw^{v_0,A}(U_e U_{e'}), \forall e,e'\in E\text{ s.t. }\,|e\cap e'|=1$.}

Recall that the proposed estimator for $\Eerrw^{v_0,A}(U_eU_{e'})$ is
\[
\hat V_{e,e'}:=\frac1K\sum_{k=1}^K \hat U_e^{(k)}\hat U_{e'}^{(k)},
\]
as defined in \cref{eqn:est-v}. We first control the bias of this estimator. As in the preceding step,
\[
\Eerrw^{v_0,A}(\hat V_{e,e'})
=
\frac1K\sum_{k=1}^K
\Eerrw^{v_0,A}\big(\hat U_e^{(k)}\hat U_{e'}^{(k)}\big).
\]
Adding and subtracting \(U_e^{(k)}U_{e'}^{(k)}\), we get
\begin{align*}
\big|\Eerrw^{v_0,A}(\hat V_{e,e'})
-\Eerrw^{v_0,A}(U_eU_{e'})\big|
&\le
\frac1K\sum_{k=1}^K
\Eerrw^{v_0,A}\big|\hat U_e^{(k)}\hat U_{e'}^{(k)}
-U_e^{(k)}U_{e'}^{(k)}\big|  \\
&\le
\frac1K\sum_{k=1}^K
\Eerrw^{v_0,A}\big|\hat U_e^{(k)}-U_e^{(k)}\big|
+
\frac1K\sum_{k=1}^K
\Eerrw^{v_0,A}\big|\hat U_{e'}^{(k)}-U_{e'}^{(k)}\big|,
\end{align*}
where we used \(0\le \hat U_e^{(k)},U_e^{(k)},\hat U_{e'}^{(k)},U_{e'}^{(k)}\le1\).
By \cref{eqn:errw-core10}, under the choices of \(m,T\) in \cref{eqn:errw-mchoice,eqn:errw-Tchoice} and with \(\eta=\delta'\),
\[
\Eerrw^{v_0,A}\big|\hat U_e^{(k)}-U_e^{(k)}\big|\le 9\delta'
\qquad\text{for every }e\in E.
\]
Therefore
\begin{equation}
\label{eqn:errw-core11}
\big|\Eerrw^{v_0,A}(\hat V_{e,e'})
-\Eerrw^{v_0,A}(U_eU_{e'})\big|
\le 18\delta'.
\end{equation}

Next, we lower bound the target moment. 
By \cref{eqn:errw-ue_lb}, \(\Eerrw^{v_0,A}(U_e)\ge\mu_*\) for every \(e\in E\). If \(e\) and \(e'\) share the vertex \(j\), then the moment identity in Lemma~\ref{lem:moment} gives
\[
\Eerrw^{v_0,A}(U_eU_{e'})
=
\Eerrw^{v_0,A}(U_e)\Eerrw^{v_0,A}(U_{e'})
\cdot
\frac{a_j+1-\ind_{j=v_0}}{a_j+3-\ind_{j=v_0}}.
\]
Moreover,
\[
\frac{a_j+1-\ind_{j=v_0}}{a_j+3-\ind_{j=v_0}}
\ge
\frac{a_j}{a_j+2}
\ge
\frac{2\underline a}{2\underline a+2}
=
\frac{\underline a}{\underline a+1}.
\]
Then, for all adjacent edge pairs \((e,e')\),
\begin{equation}
\label{eqn:errw-ue2_lb}
\Eerrw^{v_0,A}(U_eU_{e'})
\ge
c_0\mu_*^2.
\end{equation}

Assume
\begin{equation}
\label{eqn:errw-deltaprimeconditionV}
\delta'\le \frac{\epsilon}{18}c_0\mu_*^2.
\end{equation}
Then \cref{eqn:errw-core11} implies
\[
\big|\Eerrw^{v_0,A}(\hat V_{e,e'})
-\Eerrw^{v_0,A}(U_eU_{e'})\big|
\le
\epsilon\,\Eerrw^{v_0,A}(U_eU_{e'}).
\]
In particular, if \(\epsilon<1/2\), then
\[
\Eerrw^{v_0,A}(\hat V_{e,e'})
\ge
(1-\epsilon)\Eerrw^{v_0,A}(U_eU_{e'})
\ge
\frac12 c_0\mu_*^2.
\]

We now apply multiplicative Chernoff. For fixed \(e,e'\), the variables
\[
\hat U_e^{(1)}\hat U_{e'}^{(1)},\ldots,
\hat U_e^{(K)}\hat U_{e'}^{(K)}
\]
are i.i.d. and take values in \([0,1]\), with average \(\hat V_{e,e'}\). Therefore,
\[
\Perrw^{v_0,A}\Big(
|\hat V_{e,e'}-\Eerrw^{v_0,A}(\hat V_{e,e'})|
\ge
\epsilon\,\Eerrw^{v_0,A}(\hat V_{e,e'})
\Big)
\le
2\exp\left(
-\frac{\epsilon^2}{3}\Eerrw^{v_0,A}(\hat V_{e,e'})K
\right).
\]
Using \(\Eerrw^{v_0,A}(\hat V_{e,e'})\ge \frac12c_0\mu_*^2\), we get
\begin{equation}
\label{eqn:errw-V_chernoff_single}
\Perrw^{v_0,A}\Big(
|\hat V_{e,e'}-\Eerrw^{v_0,A}(\hat V_{e,e'})|
\ge
\epsilon\,\Eerrw^{v_0,A}(\hat V_{e,e'})
\Big)
\le
2\exp\left(
-\frac{\epsilon^2c_0\mu_*^2K}{6}
\right).
\end{equation}

Combining the concentration event with the bias bound gives, as in Step 1,
\[
|\hat V_{e,e'}-\Eerrw^{v_0,A}(U_eU_{e'})|
\le
3\epsilon\,\Eerrw^{v_0,A}(U_eU_{e'})
\]
outside the exceptional event in \cref{eqn:errw-V_chernoff_single}. Indeed,
\[
|\hat V_{e,e'}-\E(U_eU_{e'})|
\le
\epsilon\E(\hat V_{e,e'})
+\epsilon\E(U_eU_{e'})
\le
3\epsilon\E(U_eU_{e'}),
\]
where we used \(\E(\hat V_{e,e'})\le(1+\epsilon)\E(U_eU_{e'})\) and \(\epsilon<1\).

Recall that
\(
\mathcal P_1
\)
is the set of unordered adjacent edge pairs. Since
\[
|\mathcal P_1|
=
\sum_{v\in V}\binom{\deg(v)}2
\le
(d_{\max}-1)|E|
\le
nd_{\max}^2,
\]
a union bound over \(\mathcal P_1\) gives
\begin{equation}
\label{eqn:errw-V_chernoff_union}
\Perrw^{v_0,A}\Big(
\exists (e,e')\in\mathcal P_1:\,
|\hat V_{e,e'}-\Eerrw^{v_0,A}(U_eU_{e'})|
\ge
3\epsilon\,\Eerrw^{v_0,A}(U_eU_{e'})
\Big)
\le
2|\mathcal P_1|
\exp\left(
-\frac{\epsilon^2c_0\mu_*^2K}{6}
\right).
\end{equation}
In particular, it suffices to take
\[
K\ge
\frac{6}{\epsilon^2c_0\mu_*^2}
\log\left(\frac{2|\mathcal P_1|}{\delta}\right)
\]
to make the probability in \cref{eqn:errw-V_chernoff_union} at most \(\delta\). Since
\(|\mathcal P_1|\le nd_{\max}^2\), a simpler sufficient condition is
\[
K\ge
\frac{6}{\epsilon^2c_0\mu_*^2}
\log\left(\frac{2nd_{\max}^2}{\delta}\right).
\]

\textbf{Step 3. Estimating 
$\Eerrw^{v_0,A}(U_e)\Eerrw^{v_0,A}(U_{e'})
-\Eerrw^{v_0,A}(U_eU_{e'})$ for adjacent $e,e'$.}

Recall that $\hat\Delta_{e,e'}$, defined in \cref{eqn:est-delta}, is the proposed estimator for
\[
\Delta_{e,e'}
:=
\Eerrw^{v_0,A}(U_e)\Eerrw^{v_0,A}(U_{e'})
-
\Eerrw^{v_0,A}(U_eU_{e'}).
\]
We study $\hat U_e\hat U_{e'}-\hat V_{e,e'}$, which equals $\hat\Delta_{e,e'}$ whenever it is nonnegative.

The key point is that no additional variance estimate is needed. Once $\hat U_e,\hat U_{e'}$, and $\hat V_{e,e'}$ are multiplicatively accurate, the corresponding estimate of $\Delta_{e,e'}$ is also accurate.

Fix an internal tolerance $\epsilon_0\in(0,1)$, to be chosen below. Suppose that for all $e\in E$ and all $(e,e')\in\mathcal P_1$,
\[
|\hat U_e-\E(U_e)|\le \epsilon_0\E(U_e),
\qquad
|\hat V_{e,e'}-\E(U_eU_{e'})|\le \epsilon_0\E(U_eU_{e'}).
\]
Then, for every $(e,e')\in\mathcal P_1$,
\begin{align*}
&\big|
(\hat U_e\hat U_{e'}-\hat V_{e,e'})
-
(\E(U_e)\E(U_{e'})-\E(U_eU_{e'}))
\big| \\
&\qquad\le
|\hat U_e\hat U_{e'}-\E(U_e)\E(U_{e'})|
+
|\hat V_{e,e'}-\E(U_eU_{e'})|.
\end{align*}
The second term is at most $\epsilon_0\E(U_eU_{e'})$. For the first term, since
$\hat U_e\le (1+\epsilon_0)\E(U_e)$ and
$\hat U_{e'}\le (1+\epsilon_0)\E(U_{e'})$, while also
$\hat U_e\ge (1-\epsilon_0)\E(U_e)$ and
$\hat U_{e'}\ge (1-\epsilon_0)\E(U_{e'})$, we have
\[
|\hat U_e\hat U_{e'}-\E(U_e)\E(U_{e'})|
\le
3\epsilon_0\,\E(U_e)\E(U_{e'})
\]
for $\epsilon_0\in(0,1)$. Therefore,
\begin{equation}
\label{eqn:errw-delta-det-bound}
\big|
(\hat U_e\hat U_{e'}-\hat V_{e,e'})
-\Delta_{e,e'}
\big|
\le
3\epsilon_0\,\E(U_e)\E(U_{e'})
+
\epsilon_0\,\E(U_eU_{e'}).
\end{equation}
Since $\E(U_eU_{e'})\le \E(U_e)\E(U_{e'})$, this gives
\[
\big|
(\hat U_e\hat U_{e'}-\hat V_{e,e'})
-\Delta_{e,e'}
\big|
\le
4\epsilon_0\,\E(U_e)\E(U_{e'}).
\]

It remains to compare $\E(U_e)\E(U_{e'})$ with $\Delta_{e,e'}$. From the moment formulas in Lemma \ref{lem:moment}, if $j$ is the shared vertex of $e$ and $e'$, then
\[
\Delta_{e,e'}
=
\frac{2\E(U_e)\E(U_{e'})}{a_j+3-\ind_{j=v_0}}.
\]
Equivalently,
\[
\E(U_e)\E(U_{e'})
=
\frac{a_j+3-\ind_{j=v_0}}{2}\Delta_{e,e'}
\le
\frac{\overline a d_{\max}+3}{2}\Delta_{e,e'}.
\]
Thus
\[
\big|
(\hat U_e\hat U_{e'}-\hat V_{e,e'})
-\Delta_{e,e'}
\big|
\le
2\epsilon_0(\overline a d_{\max}+3)\Delta_{e,e'}.
\]
Choose
\begin{equation}
\label{eqn:errw-eps0-delta-choice}
\epsilon_0:=
\frac{\epsilon}{2(\overline a d_{\max}+3)}.
\end{equation}
Then
\begin{equation}
\label{eqn:errw-delta-good-deterministic}
\big|
(\hat U_e\hat U_{e'}-\hat V_{e,e'})
-\Delta_{e,e'}
\big|
\le
\epsilon\,\Delta_{e,e'}
\qquad
\forall (e,e')\in\mathcal P_1.
\end{equation}

In particular, since \(\epsilon<1\), the quantity
\(\hat U_e\hat U_{e'}-\hat V_{e,e'}\) is positive on this event, so it agrees with the estimator
\(\hat\Delta_{e,e'}\) defined in \cref{eqn:est-delta}.

We now invoke the concentration bounds from Steps 1 and 2 with the final multiplicative tolerance there taken to be \(\epsilon_0\), equivalently with the auxiliary tolerance \(\epsilon_0/3\) in the last Chernoff step. Thus, provided the corresponding bias conditions hold, namely
\[
\delta'\le \frac{\epsilon_0}{27}\mu_*,
\qquad
\delta'\le \frac{\epsilon_0}{54}c_0\mu_*^2,
\]
Steps 1 and 2 give
\[
\Perrw^{v_0,A}\Big(
\exists e\in E:\,
|\hat U_e-\E(U_e)|>\epsilon_0\E(U_e)
\Big)
\le
2|E|\exp\left(-\frac{\epsilon_0^2\mu_*K}{54}\right),
\]
and
\[
\Perrw^{v_0,A}\Big(
\exists (e,e')\in\mathcal P_1:\,
|\hat V_{e,e'}-\E(U_eU_{e'})|>\epsilon_0\E(U_eU_{e'})
\Big)
\le
2|\mathcal P_1|\exp\left(-\frac{\epsilon_0^2c_0\mu_*^2K}{54}\right).
\]
Therefore, by the deterministic implication above,
\begin{align}
\label{eqn:errw-Delta_chernoff_union}
\Perrw^{v_0,A}\Big(
\exists (e,e')\in\mathcal P_1:\,
|(\hat U_e\hat U_{e'}-\hat V_{e,e'})-\Delta_{e,e'}|
>
\epsilon\Delta_{e,e'}
\Big)
&\le
2|E|\exp\left(-\frac{\epsilon_0^2\mu_*K}{54}\right) \nonumber\\
&\quad+
2|\mathcal P_1|\exp\left(-\frac{\epsilon_0^2c_0\mu_*^2K}{54}\right).
\end{align}
Since $\epsilon_0=\epsilon/[2(\overline a d_{\max}+3)]$, it suffices to take
\[
K
\ge
\max\left\{
\frac{216(\overline a d_{\max}+3)^2}{\epsilon^2\mu_*}
\log\left(\frac{4|E|}{\delta}\right),
\,
\frac{216(\overline a d_{\max}+3)^2}{\epsilon^2c_0\mu_*^2}
\log\left(\frac{4|\mathcal P_1|}{\delta}\right)
\right\}
\]
to make the failure probability in \eqref{eqn:errw-Delta_chernoff_union} at most $\delta$.

\textbf{Step 4. Piecing it all together.}

We now combine the preceding estimates. 
Let
\[
\epsilon_0:=\frac{\epsilon}{2(\overline a d_{\max}+3)}.
\]
To make the bias terms negligible at the internal tolerance level needed for the
\(\Delta\)-estimate, it is enough to choose \(\delta'\) so that
\[
\delta'
\le
\min\left\{
\frac{\epsilon_0}{27}\mu_*,
\frac{\epsilon_0}{54}c_0\mu_*^2
\right\}.
\]
Equivalently, it is enough to take, for example,
\[
\delta'
\le
\frac{\epsilon}{108(\overline a d_{\max}+3)}
\min\{\mu_*,c_0\mu_*^2\}.
\]

With this choice of \(\delta'\), and with \(m,T\) satisfying \cref{eqn:errw-mchoice,eqn:errw-Tchoice}
where \(\eta=\delta'\), Steps 1--3 imply the following. For all \(e\in E\),
\[
|\hat U_e-\Eerrw^{v_0,A}(U_e)|
\le
\epsilon_0\,\Eerrw^{v_0,A}(U_e)
\]
fails with probability at most
\[
2|E|\exp\left(-\frac{\epsilon_0^2\mu_*K}{54}\right).
\]
For all \((e,e')\in\mathcal P_1\),
\[
|\hat V_{e,e'}-\Eerrw^{v_0,A}(U_eU_{e'})|
\le
\epsilon_0\,\Eerrw^{v_0,A}(U_eU_{e'})
\]
fails with probability at most
\[
2|\mathcal P_1|\exp\left(-\frac{\epsilon_0^2c_0\mu_*^2K}{54}\right).
\]
On the intersection of these two good events, Step 3 gives
\[
|\hat\Delta_{e,e'}-\Delta_{e,e'}|
\le
\epsilon\,\Delta_{e,e'}
\qquad\forall (e,e')\in\mathcal P_1,
\]
where
\[
\Delta_{e,e'}
:=
\Eerrw^{v_0,A}(U_e)\Eerrw^{v_0,A}(U_{e'})
-
\Eerrw^{v_0,A}(U_eU_{e'}).
\]
Moreover, since \(\epsilon_0\le \epsilon\), the same good event also gives
\[
|\hat U_e-\Eerrw^{v_0,A}(U_e)|
\le
\epsilon\,\Eerrw^{v_0,A}(U_e)
\qquad\forall e\in E,
\]
and
\[
|\hat V_{e,e'}-\Eerrw^{v_0,A}(U_eU_{e'})|
\le
\epsilon\,\Eerrw^{v_0,A}(U_eU_{e'})
\qquad\forall (e,e')\in\mathcal P_1.
\]

Therefore, by the union bound, all three families of quantities are simultaneously
\(\epsilon\)-multiplicatively approximated with probability at least
\[
1
-
2|E|\exp\left(-\frac{\epsilon_0^2\mu_*K}{54}\right)
-
2|\mathcal P_1|\exp\left(-\frac{\epsilon_0^2c_0\mu_*^2K}{54}\right).
\]
Consequently, for any \(\delta\in(0,1)\), it suffices to take
\[
K
\ge
\max\left\{
\frac{54}{\epsilon_0^2\mu_*}
\log\left(\frac{4|E|}{\delta}\right),
\frac{54}{\epsilon_0^2c_0\mu_*^2}
\log\left(\frac{4|\mathcal P_1|}{\delta}\right)
\right\}.
\]
Using \(\epsilon_0=\epsilon/[2(\overline a d_{\max}+3)]\), this is equivalent to
\[
K
\ge
\max\left\{
\frac{216(\overline a d_{\max}+3)^2}{\epsilon^2\mu_*}
\log\left(\frac{4|E|}{\delta}\right),
\frac{216(\overline a d_{\max}+3)^2}{\epsilon^2c_0\mu_*^2}
\log\left(\frac{4|\mathcal P_1|}{\delta}\right)
\right\}.
\]
Since \(|E|\le nd_{\max}/2\) and \(|\mathcal P_1|\le nd_{\max}^2\), a simpler sufficient
condition is
\[
K
\ge
\frac{216(\overline a d_{\max}+3)^2}{\epsilon^2c_0\mu_*^2}
\log\left(\frac{4nd_{\max}^2}{\delta}\right).
\]
In particular,
\[
K
=
O_{\underline a,\overline a}
\left(
\frac{d_{\max}^6}{\epsilon^2}
\log\frac{nd_{\max}}{\delta}
\right).
\]

\vspace{.3em}
\begin{thm}[Non-asymptotic bounds]
\label{thm:nonasym}
Let $G=(V=[n],E)$ be a connected weighted graph with initial edge weights
$A=(a_e)_{e\in E}$ and diameter $\diam$. Assume that
$\underline a\le a_e\le\overline a$ for all $e\in E$, where
$\underline a,\overline a$ are positive constants. Fix \(\epsilon,\delta\in(0,1)\). Let
\[
\mathcal P_1:=\{\{e,e'\}:e,e'\in E,\ e\neq e',\ |e\cap e'|=1\}.
\]
Define
\[
\mu_*:=\frac{\underline a(\underline a+1)}{(d_{\max}\overline a+1)^2},
\qquad
c_0:=\frac{\underline a}{\underline a+1},
\qquad
\epsilon_0:=\frac{\epsilon}{2(\overline a d_{\max}+3)}.
\]
Let \(\delta'\in(0,1)\) satisfy
\[
\delta'
\le
\frac{\epsilon}{108(\overline a d_{\max}+3)}
\min\{\mu_*,c_0\mu_*^2\}.
\]
Suppose
\[
m\ge
\frac{8d_{\max}^2}{(\delta')^2}
\left(\frac{d_{\max}}{\delta'}\right)^{4g_1(\underline a,\overline a)\diam}
\log\left(\frac{2nd_{\max}}{\delta'}\right),
\]
and
\[
T\ge
e\,g_2\big(n^3\log n+(m+1)nd_{\max}\big)
\left(\frac{d_{\max}}{\delta'}\right)^{2g_1(\underline a,\overline a)\diam}
\log\left(\frac{e}{\delta'}\right).
\]
If
\[
K
\ge
\max\left\{
\frac{216(\overline a d_{\max}+3)^2}{\epsilon^2\mu_*}
\log\left(\frac{4|E|}{\delta}\right),
\frac{216(\overline a d_{\max}+3)^2}{\epsilon^2c_0\mu_*^2}
\log\left(\frac{4|\mathcal P_1|}{\delta}\right)
\right\},
\]
then, with probability at least \(1-\delta\), the estimators simultaneously give
\(\epsilon\)-multiplicative approximations of
\[
\Eerrw^{v_0,A}(U_e),\qquad
\Eerrw^{v_0,A}(U_eU_{e'}),\qquad
\Eerrw^{v_0,A}(U_e)\Eerrw^{v_0,A}(U_{e'})
-\Eerrw^{v_0,A}(U_eU_{e'})
\]
for all \(e\in E\) and all \((e,e')\in\mathcal P_1\).
Consequently, the induced estimates of the \(o_j\)'s, and hence of the edge weights
\(a_e\), are \(O(\epsilon)\)-accurate.
\end{thm}

\medskip
\begin{rem}[Constant-degree graphs]
For constant-degree graphs, i.e., when $d_{\max}=O(1)$, the preceding theorem gives a particularly simple form. For fixed $\epsilon,\delta\in(0,1)$ and fixed constants $\underline a,\overline a$, it suffices to take
\[
K=O(\log n),
\qquad
m=\exp\{O(\diam)\}\log n,
\]
and
\[
T
=
O\!\left(
n^3\log n\,\exp\{O(\diam)\}
\right).
\]
Thus, for bounded-degree graphs, the number of independent trajectories needed is only logarithmic in the graph size. The main cost is instead the required trajectory length, which is governed by the cover-time factor $\exp\{O(\diam)\}$. In particular, if the graph has both bounded degree and logarithmic diameter, then \(T\) is polynomial in \(n\).
\end{rem}

\medskip
\begin{rem}[Positive-integer-valued edge weights]
In the special case where all edge weights are integers, one might wonder whether the sample complexity can be significantly reduced. 
However, the requirement to cover the graph
imposes a fundamental bottleneck on the trajectory length $T$.

To see this, consider a path consisting of $(n+1)$ vertices and let $A = \mathbf{1}$ be the initial edge local time. Label the left-most vertex as $v_0$ and the right-most vertex as $v_n$. Define $\tilde{A}$ as an alternative hypothesis that matches $A$ except for the right-most edge weight, which is set to 2 in $\tilde{A}$. To obtain any information about the right-most edge, the trajectory must reach it first, necessitating a traversal of the entire path. This imposes a constraint on $T$, making it at least $2^{\Omega(n)}$. The time to reach $v_n$ from 
$v_0$ is asymptotically bounded below like this because
it stochastically dominates the time to reach $n$ from $0$ in a biased random walk with negative drift on the nonnegative integers. So, while the positive-integer valued case is likely to be easier than the general case as it may require a lower sample complexity for the number of trajectories $K$, the dependence on
the cover time $T$ 
is still unimprovable using our current techniques.
\end{rem}

\section{Open Problems and Future Directions}

In this work, we explored how to estimate the initial weights in an 
edge-reinforced random walk (ERRW) given a collection of sampled 
trajectories. We provided lower and upper bounds on the complexity of these estimation tasks, both in terms of the length of the trajectories and in terms of the number of trajectories. We expect that the upper bounds we derived can be significantly improved with more carefully designed algorithms.

We list some other questions that seem to warrant further investigation.

\textbf{Question 1. Sample-Complexity Upper Bounds for Testing}

Consider identity testing: the null hypothesis $H_0$ is $A = A_0$, while the alternative $H_1$ states $A$ is at least $\epsilon$-far from $A_0$ under some metric. An interesting question is whether the sample complexity of the classical likelihood-ratio tests,
is competitive with test that use ERRW-specific structure, perhaps along the lines of the ERRW-specific properties we have used in designing the estimators derived in thie paper.

\textbf{Question 2. Non-Linear Activation Functions}

Beyond the linear activation setting, one can study ERRW with non-linear activation. The law of ERRW with activation function $g:\mathbb{R}_{++}^n\rightarrow\mathbb{R}_{++}^n$ starting at $v_0$ with initial edge local time $L_0=A$ can be stated as
\[
\P_{\textnormal{errw}(g)}^{v_0,A}(X_{t+1}=i \mid \{X_0,X_1,\cdots, X_t\}) = \frac{g\big(L_t^{\{X_t, i\}}\big)}{\sum_{i':X_t \sim i'} g\big(L_t^{\{X_t, i'\}}\big)}, \quad \forall i\in V,
\quad \forall t \in \mathbb{Z}_+.
\]
Some examples for non-linear activation functions include $g(x) = x^\alpha$ ($\alpha \neq 1$) and $g(x) = \log(x+1)$. While such processes have been analyzed in the literature(e.g., see \cite{lt, ct}), their behavior cannot be reduced to a mixture of Markov chains, and no ``magic formula'' integration identity is known. Nevertheless, one might still attempt to recover the initial weights from trajectory data or, alternatively, prove that exact recovery is impossible.

\section*{Acknowledgements}

This research was supported by the grants
CCF-1901004 and CIF-2007965 from the U.S. National Science Foundation.

\end{document}